\documentclass[
  journal=bjps,
  manuscript=article,  
  year=2023,
  volume=2,
]{cup-journal}
\DeclareUnicodeCharacter{0301}{\'{e}}
\usepackage[toc,page,header]{appendix}

\usepackage{titletoc}
\usepackage{booktabs,microtype,siunitx}
\usepackage{amsfonts,amsmath,bbm,bbding}
\usepackage{amssymb}
\usepackage{multirow}
\usepackage{tikz}
\usepackage{nicefrac}
\usepackage{tcolorbox}
\usepackage{wrapfig}
\usepackage{amsthm}
\usepackage{xcolor}
\usepackage{csquotes}
\usepackage{algorithm,algorithmic}
\usepackage{sectsty}

\sectionfont{\sffamily\bfseries\large\color{green!40!black}}  
\subsectionfont{\sffamily\bfseries\normalsize\color{green!40!black}}
\newtheorem{theorem}{Theorem}
\newtheorem{remark}{Remark}

\newtheorem{lemma}{Lemma}
\newtheorem{assumption}{Assumption}
\newtheorem{definition}{Definition}

\usepackage{amsmath,amsfonts,bm}

\def\eqref#1{equation~\ref{#1}}

\def\1{\bm{1}}

\def\vb{{\bm{b}}}

\def\ve{{\bm{e}}}

\def\vg{{\bm{g}}}
\def\vh{{\bm{h}}}

\def\vp{{\bm{p}}}
\def\vq{{\bm{q}}}
\def\vr{{\bm{r}}}

\def\vt{{\bm{t}}}
\def\vu{{\bm{u}}}
\def\vv{{\bm{v}}}
\def\vw{{\bm{w}}}
\def\vx{{\bm{x}}}

\def\vz{{\bm{z}}}

\def\mA{{\bm{A}}}
\def\mB{{\bm{B}}}

\def\mE{{\bm{E}}}

\def\mG{{\bm{G}}}
\def\mH{{\bm{H}}}
\def\mI{{\bm{I}}}
\def\mJ{{\bm{J}}}

\def\mM{{\bm{M}}}
\def\mN{{\bm{N}}}
\def\mO{{\bm{O}}}
\def\mP{{\bm{P}}}
\def\mQ{{\bm{Q}}}
\def\mR{{\bm{R}}}

\def\mT{{\bm{T}}}

\def\mV{{\bm{V}}}
\def\mW{{\bm{W}}}
\def\mX{{\bm{X}}}

\DeclareMathAlphabet{\mathsfit}{\encodingdefault}{\sfdefault}{m}{sl}
\SetMathAlphabet{\mathsfit}{bold}{\encodingdefault}{\sfdefault}{bx}{n}

\def\gS{{\mathcal{S}}}

\newcommand{\R}{\mathbb{R}}

\usepackage{amssymb,enumitem,graphicx}
\usepackage[unicode=true,bookmarksnumbered=true,colorlinks]{hyperref}
\definecolor{amethyst}{rgb}{0.6, 0.4, 0.8}
\hypersetup{
    colorlinks=true,
    linkcolor=purple,
    filecolor=magenta,      
    urlcolor=cyan,
    citecolor=amethyst,
}
\usepackage{url}
\usepackage{tikz}
\usepackage{tcolorbox}
\usepackage{wrapfig}
\usepackage{subcaption}
\usepackage{placeins}
\usepackage[nameinlink,noabbrev,capitalise]{cleveref}
\Crefname{table}{Table}{Tables}
\crefname{equation}{}{}
\Crefname{assumption}{Assumption}{Assumptions}
\Crefname{example}{Example}{Examples}
\Crefname{hypothesis}{Hypothesis}{Hypotheses}
\Crefname{figure}{Figure}{Figures}

\usepackage{color-edits}
\addauthor{vs}{red}

\newtheorem{proposition}{Proposition}

\newtheorem{example}{Example}

\newcommand{\rank}[1]{\operatorname{rank}\left(#1\right)}

\renewcommand{\O}{\mathcal{O}}
	
\newenvironment{proof-sketch}{\noindent\textbf{\textit{Proof sketch}:}}{\hfill$\square$}

\renewcommand{\O}{\mathcal{O}}
\renewcommand{\H}{\mathcal{H}}
\newcommand{\G}{\mathcal{G}}

\renewcommand{\rank}[1]{\operatorname{rank}\left(#1\right)}

\newcommand{\I}{\mathcal{I}}
\newcommand\spanl[1]{\mathrm{span}\left\langle#1\right\rangle}

\newcommand{\var}{\operatorname{Var}}
\renewcommand{\geq}{\geqslant}
\renewcommand{\leq}{\leqslant}

\newcommand\normx[1]{\left\Vert#1\right\Vert}
\newcommand\absx[1]{\left\vert#1\right\vert}
\newcommand{\CRLsim}{\sim_{\mathrm{CRL}}}

\newcommand{\pa}{\mathrm{pa}}
\newcommand{\hpa}{\overline{\mathrm{pa}}}
\newcommand{\ch}{\mathrm{ch}}
\newcommand{\hch}{\overline{\mathrm{ch}}}
\newcommand{\ans}{\mathrm{ans}}

\newcommand\independent{\protect\mathpalette{\protect\independenT}{\perp}}
\def\independenT#1#2{\mathrel{\rlap{$#1#2$}\mkern2mu{#1#2}}}
\newcommand{\nd}{\mathrm{nd}}
\newcommand{\dom}{\mathrm{sur}}
\renewcommand{\algorithmiccomment}[1]{\hfill $\triangleright$ #1}
\newcommand{\algorithmicbreak}{\textbf{break}}
\newcommand{\BREAK}{\STATE \algorithmicbreak}

\usepackage{tikz}
\usetikzlibrary{positioning, arrows.meta, shapes.geometric, fit}

\title{Learning Causal Representations from General Environments: Identifiability and Intrinsic Ambiguity}

\author{Jikai Jin}
\affiliation{Institute for Computational \& Mathematical Engineering, Stanford University, CA, USA}
\email{jkjin@stanford.edu}

\author{Vasilis Syrgkanis}
\affiliation{Department of Management Science \& Engineering, Stanford University, CA, USA}
\email{vsyrgk@stanford.edu}


\begin{document}

\begin{abstract}
    We study causal representation learning, the task of recovering high-level latent variables and their causal relationships in the form of a causal graph from low-level observed data (such as text and images), assuming access to observations generated from multiple environments. 
    Prior results on the identifiability of causal representations typically assume access to single-node 
    interventions which is rather unrealistic in practice, since the latent variables are unknown in the first place. 
    In this work, we provide the first identifiability results based on data that stem from \textit{general environments}.
    We show that for linear causal models, while the causal graph can be fully recovered, the latent variables are only identified up to the \textit{surrounded-node ambiguity} (SNA) \citep{varici2023score}. We 
    provide a counterpart of our guarantee, showing that SNA is basically unavoidable in our setting. We also propose an algorithm, \texttt{LiNGCReL} which provably recovers the ground-truth model up to SNA, and we demonstrate its effectiveness via numerical experiments.
    Finally, we consider general non-parametric causal models and show that the same identification barrier holds when assuming access to groups of soft single-node interventions.
\end{abstract}

\section{Introduction}

Artificial intelligence (AI) has achieved tremendous success in various domains in the past decade \citep{bengio2013representation,silver2016mastering,bubeck2023sparks}. However, current approaches are largely based on learning the \textit{statistical} structures and relationships in the data that we observe. As a result, it is not surprising that these approaches often capture spurious statistical dependencies between different features, resulting in poor performance in the presence of test distribution shift \citep{ovadia2019can,koh2021wilds} or adversarial attacks \citep{akhtar2018threat,wang2023adversarial}.

In view of these pitfalls, a recent line of work has explored the problem of \textit{causal representation learning} \citep{scholkopf2021toward}, the task of learning the causal relationships between high-level latent variables underlying our low-level observations. Notably, it is widely believed in cognitive psychology that humans take a causal approach to distill information from the world and make decisions to achieve their goals \citep{shanks1988associative,dunbar2004causal,holyoak2011causal}. As a result, there is reason to believe that learning causal representations has the potential to significantly improve the power of AI, especially on tasks where performance lags far behind human level \citep{geirhos2020shortcut}.

Despite such promise, a crucial challenge in causal representation learning is the \textit{identifiability} of the data generating process; in other words, given the data that we observe, can we 
uniquely identify the underlying causal model. It has been shown that given observational data (\emph{i.e.}, i.i.d. data generated from a single environment), the model is already non-identifiable in strictly simpler settings where the latent variables are known to be independent \citep{locatello2019challenging,locatello2020sober}, or where there is no mixing function and one directly observes the latent variables \citep{silva2006learning}. A natural question that arises is what types of data do we need to acquire to make identification possible.

One line of works assumes access to counterfactual data \citep{locatello2020weakly,von2021self,brehmer2022weakly}, where some form of \textit{weak supervision} is typically required. A common assumption here is that one observes data in \emph{pairs}, where each pair of data is related via sharing part of the latent representation. However, such data is hard to acquire since it requires direct control on the latent representation.

Another line of works \citep{ahuja2023interventional,von2023nonparametric,buchholz2023learning,varici2023general} instead considers an interventional setting, where the learner observes data generated from multiple different environments. This is arguably a much more realistic setup and reflects common practices in robotics \citep{lippe2023biscuit} and genomics \citep{lopez2023learning,tejada2023causal} applications. However, a vast majority of identifiability guarantees assume that each environment corresponds to 
\emph{single-node}, \emph{hard} interventions, which is defined as interventions that isolate a single latent variable from its causal parents. Again, this is quite a restrictive assumption because of two reasons. \emph{First}, since the latent variables are unknown and need to be learned from data, it is unclear how to perform interventions that only affect one variable. \emph{Second}, even if one can perform single-node interventions, it may not be feasible to artificially remove causal effects in the data generating processes. This issue is ubiquitous in real-world applications as pointed out in \citet{campbell2007interventionist,eberhardt2014direct,eronen2020causal}. Motivated by these challenges, we instead consider the following two settings:
\begin{itemize}
    \item Learning from \textit{single-node, soft} interventions, which only change the dependency of each latent variable on its direct causes, but does not remove their causal relationships. This setting is considered in \citet{seigal2022linear,zhang2022towards,varici2023score,buchholz2023learning}, which, however, make parametric assumptions on either the causal model or the mixing function. The most related paper is \citet{varici2023score}, which proves identifiability under an ambiguity induced by the so-called \enquote{surrounded set}. In this paper, we show that this type of ambiguity is intrinsic in the soft intervention setting.
    \item Learning from \textit{fully general and diverse environments}. This is a significantly more general and challenging setting, and to the best of our knowledge, no identifiability guarantees are known. \citet{khemakhem2020variational,lu2021invariant} also consider a multi-environment setup without assuming single-node interventions, but they still assume that the distributions of latent variables all come from a certain parametric family with a fixed set of sufficient statistics. 
\end{itemize}
We make the following contributions:
\begin{itemize}
    \item For linear causal models, with a linear mixing function, we prove identification results assuming access to data from \textit{general and diverse environments} (\Cref{thm:linear-main-thm}). To the best of our knowledge, this is the first identification guarantee that makes no assumption on the relationship between the environments. Interestingly, while we show that the causal graph can be exactly recovered, the latent variables are only recovered up to a 
    \textit{surrounded-node ambiguity} (SNA) (\Cref{thm:informal-linear-ambiguity}).
    \item We propose an algorithm, \texttt{LiNGCReL}, in \Cref{sec:algorithm} that provably recovers the ground-truth model up to SNA (\Cref{thm:alg-guarantee}) in the setting of \Cref{thm:linear-main-thm}. To demonstrate the effectiveness of \texttt{LiNGCReL}, we present extensive experimental results in \Cref{sec:experiment-results} using it to learn causal representations from randomly generated causal models. 
    \item Going beyond the linear setting, we study the limit of identification for non-parametric causal models and general mixing functions, assuming access to single-node soft interventions. We show that the model is identifiable up to SNA (\Cref{thm:single-node-soft-nonparam}), and then prove that SNA is actually the best achievable guarantee in this setting (\Cref{thm:informal-non-param-ambiguity}), thereby highlighting a key difference between soft and hard interventions. 
\end{itemize}


\section{Preliminaries}
\label{sec:prelim}

We consider the standard setup of causal representation learning from multiple environments $E\in\mathfrak{E}$. Let $\G=(\mathcal{V},\mathcal{E})$ be the ground-truth causal graph which is directed and acylic (DAG), where $\mathcal{V}=[d]$ and $\mathcal{E}$ describes the causal relationship between different nodes. Each node corresponds to a latent variable $\vz_i\in\R$. 

For any node $i$, we let $\mathrm{pa}_{\G}(i)$, $\mathrm{ch}_{\G}(i)$, $\ans_{\G}(i)$ and $\mathrm{nd}_{\G}(i)$ to be the set of all parents, children, ancestors and non-descendants of $i$ in $\G$ respectively. We also define $\hpa_{\G}(i)=\pa_{\G}(i)\cup\{i\}$ and similarly for $\hch_{\G}(i),\overline{\ans}_{\G}(i)$ and $\overline{\mathrm{nd}}_{\G}(i)$. Assuming that all probability distributions have continuous densities, the joint density of the latent variables $\vz$ can then be written as
\begin{equation}
    \label{eq:generate-z}
    p_{E}(\vz) = \prod_{i=1}^d p_i^E\left(\vz_i\mid\vz_{\mathrm{pa}_{\G}(i)}\right).
\end{equation}
where $p_i^E$ is the (unknown) latent generating distribution from environment $E$ at node $i$. Here for a given vector $\vv$, we write $\vv_i=\ve_i^\top\vv$, and let $\vv_S = \left(\vv_i: i\in S\right)\in\R^{|S|}$.

The causal graph model with density given by \Cref{eq:generate-z} necessarily enjoys the following property:

\begin{definition}[Causal Markov Condition]
\label{asmp:causal-markov-condition}
    For any node $i$, conditioned on $\vz_{\mathrm{pa}_{\G}(i)}$, $\vz_i$ is independent of $\vz_{\mathrm{nd}_{\G}(i)}$. As a consequence, for any node $i,j\in[d]$ and $S\subseteq[d]$, if $S$ $d$-separates $i$ from $j$ (cf. \Cref{def:d-separation}), then $\vz_i\independent\vz_j\mid\vz_S$.
\end{definition}

The latent variables $\vz$ are unknown to the learner. Instead, the learner has access to observations $\vx\in\R^n$ ($n\geq d$) from all environments $E\in\mathfrak{E}$ that are related to the latent $\vz$ via an injective mixing function $\vg$:
\begin{equation}
    \label{generate-x}
    \vx = \vg(\vz).
\end{equation}
The main assumption here that the mixing function is the same across all environments:
\begin{assumption}
\label{asmp:data-generating-process}
    All environments $E\in\mathfrak{E}$ share the same diffeomorphic mixing function $\vg:\R^d\mapsto\R^n$.
\end{assumption}

In causal representation learning, the goal of the learner is to 1) recover the inverse of the mixing function $\vh = \vg^{-1}$ (often called the \textit{unmixing} function) which allows recovering the latent variables given any observations, and, 2) recover the underlying causal graph $\G$. In the remaining part of this paper, we refer to $(\vh,\G)$ as the causal model to be learned. Obviously, there would be some ambiguities in learning $(\vh,\G)$. For example, choosing a different permutation of the nodes in the causal graph would lead to a different model, and so does element-wise transformations on each component $\vh_i$ of $\vh$. 

A line of recent works show that the ground-truth model can be identified up to these ambiguities in various settings, assuming access to single-node hard interventions \citep{seigal2022linear,von2023nonparametric,varici2023general}. On the other hand, some weaker notions of identifiability have also been proposed and studied in the literature \citep{seigal2022linear,varici2023score,liang2023causal} for single-node soft interventions. Here, we provide a generic definition of single-node soft interventions that we will rely on in this paper.

\begin{definition}
\label{def:intervention}
    We say that a collection of environments $\hat{\mathfrak{E}}$ is a set of (soft) interventions on a subset of latent variables $\vz_i, i\in S$ if $\forall i\in[d]$, $\forall E_1, E_2 \in \hat{\mathfrak{E}}, E_1 \neq E_2$, we have $p_i^{E_1}=p_i^{E_2} \Leftrightarrow i\notin S$ (the notation $p_i^E$ comes from \Cref{eq:generate-z}). Equivalently, we write $\I_{\vz}^{\hat{\mathfrak{E}}} = S$.
\end{definition}


We note that soft interventions are very different from hard interventions, since they do not remove causal relationships between latent variables. The goal of this paper is to address the following question:
\begin{center}
    \emph{What is the best-achievable identification guarantee when hard interventions are not available, and what are the intrinsic ambiguities?}
\end{center}

\section{The effect domination set and a notion of identifiability}
\label{sec:eda}

One may expect that identifiability with soft interventions is not much different from hard interventions, since soft interventions can approximate hard interventions with arbitrary accuracy. However, we will show that this is not the case. At a high level, hard intervention is more powerful than soft intervention because it is capable of isolating a latent variable from its direct cause while soft interventions is not, so soft interventions can sometimes fail to identify the true causal relationship from a mixture of causal effects.


To quantify what kind of ambiguities may arise, we can define the surrounding set for each node in a causal graph $\G$ as follows:

\begin{definition}
\label{def:effect-domination}
    (\citealp[Definition 3]{varici2023score}) For two nodes $i,j\in[d]$ in $\G$, we say that $j$ is effect-dominated by $i$, or $i\in\dom_{\G}(j)$ if $i\in\mathrm{pa}_{\G}(j)$, and $\mathrm{ch}_{\G}(j)\subseteq\mathrm{ch}_{\G}(i)$. Moreover, we define $\overline{\dom}_{\G}(j)=\dom_{\G}(j)\cup\{j\}$.
\end{definition}

In other words, the effect of $j$ on its child set $\ch_{\G}(j)$ is dominated by the effect of $i$. Intuitively, if there exists some $i\in\dom_{\G}(j)$, then ambiguities may arise for the causal variable at node $j$, since any effect of $j$ on any of its child $k$ can also be interpreted as an effect of $i$. In \Cref{appsec:eda-example} we discuss a three-node example to further illustrate such ambiguities.

\begin{figure}[htpb]
\centering
\begin{tikzpicture}[node distance=1.1cm and 1.1cm, every node/.style={draw, circle, minimum size=.3cm}, >={Latex[width=0.5mm,length=0.5mm]}]
    \node (i) {$i$};
    \node (j) [right=of i] {$j$};
    \node (i1) [below left=of i] {$i_1$};
    \node (j1) [right=of i1] {$j_1$};
    \node (j2) [right=of j1] {$j_2$};
    \node (j3) [right=of j2] {$j_3$};

    \draw[->] (i) -- (j);
    \draw[->] (i) -- (i1);
    \draw[->] (j) -- (j1);
    \draw[->] (j) -- (j2);
    \draw[->] (j) -- (j3);
    \draw[->] (i) -- (j1);
    \draw[->] (i) -- (j2);
    \draw[->] (i) -- (j3);
    \node (c) [draw, dashed, ellipse, fit=(j1) (j2) (j3)] {};
    \node [draw=none,right=of c, xshift=-3cm, yshift=1.0cm] {$\mathrm{ch}_{\G}(j)\subseteq \mathrm{ch}_{\G}(i)$};
\end{tikzpicture}
\caption{An illustration of \Cref{def:effect-domination}; here $i\in\dom_{\G}(j)$.}
\label{fig:effect}
\end{figure}
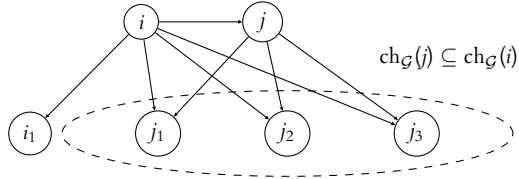

\Cref{def:effect-domination} naturally induces the following relationship between causal models:

\begin{definition}
\label{def:dom-ambiguity}
    Using the notations in \Cref{def:crl-identify}, we write
    $(\vh,\G)\sim_{\dom}(\hat{\vh},\hat{\G})$ if there exists a permutations $\pi$ on $[d]$, and a diffeomorphism $\psi:\R^d\mapsto\R^d$ where the $j$-th component of $\psi$, denoted by $\psi_j(\vz)$, is a function of $\vz_{\overline{\dom}_{\G}(j)}$ for $\forall j\in[d]$, such that the following holds: 
    \begin{itemize}
        \item For $\forall i,j\in[d]$, $i\in\pa_{\G}\left(j\right) \Leftrightarrow \pi(i)\in\pa_{\hat{\G}}\left(\pi(j)\right)$,
        \item $\mP_{\pi}\circ\hat{\vh}=\psi\circ\vh$, where $\mP_{\pi}$ is a permutation matrix satisfying $(\mP_{\pi})_{ij}=1 \Leftrightarrow j=\pi(i)$.
    \end{itemize}
\end{definition}

In other words, $\sim_{\dom}$ requires that the causal graph to be exactly the same up to some permutation of nodes, but allows each latent variable $\vv_i$ to be entangled with $\vz_{\dom_{\G}(i)}$. Although not obvious from definition, one can actually check that $\sim_{\dom}$ defines an \emph{equivalence relation} (see \Cref{lemma:dom-equivalence-relation}). Moreover, we will show later that $\sim_{\dom}$ is in general the best that we can hope for in our problem setting.

\section{Identifiability of linear causal models from general environments}
\label{sec:linear}
In this section, we consider learning causal  models from \emph{general} environments. Specifically, we assume that the environments $E_k, k\in[K]$ share the same causal graph, but the dependencies between connected nodes (latent variables) are completely unknown, and, in contrast with existing literature on single-node interventions, we impose no similarity constraints on the environments. We begin our investigation of identifiability in this setting in the context of linear causal models with a linear mixing function.

\subsection{Problem setup}

Formally, we assume the following generative model in $K$ distinct environments $\mathfrak{E}=\left\{E_k: k\in[K]\right\}$ with data generating process
\begin{equation}
    \label{latent-original}
    \vz = \mA_k \vz + \bm{\Omega}_k^{\frac{1}{2}}\epsilon,\quad \vx = \mG\vz \quad k\in[K],
\end{equation}
where the matrix $\mA_k$ satisfies $(\mA_k)_{ij}\neq 0$ if and only if $j\to i$ in $\mathcal{G}$. We refer to $(\mA_k,\bm{\Omega}_k)$ as the weight matrices of latent variables $\vz$ in the environment $E_k$. It is easy to see that \Cref{asmp:data-generating-process} in our general setup translates into the following assumption:

\begin{assumption}\label{asmp:full-rank}
    The mixing matrix $\mG\in\R^{n\times d}$ has full column rank. Equivalently, the unmixing matrix $\mH=\mG^{\dagger}$ has full row rank. 
\end{assumption}

Let $\mB_k = \bm{\Omega}_k^{-\frac{1}{2}}(\mI-\mA_k), k\in[K]$, then we have $\epsilon = \mB_k\vz = \mB_k\mH\vx$. Since in the linear case, there is an easy to see one-to-one correspondence between the matrix $\mH$ and the un-mixing function $\vx\mapsto\mH\vx$, we abuse the notation and write $(\mH,\G)$ to represent the model instead of $(\vh,\G)$. Using $\vh_i$ to denote the $i$-th row of $\mH$, the following lemma translates \Cref{def:dom-ambiguity} the the linear setting:

\begin{lemma}\label{lem:mixture}
    According to \Cref{def:dom-ambiguity}, $(\mH,\G)\sim_{\dom} (\hat{\mH},\hat{\G})$ if and only if there exists a permutation $\pi$ on $[d]$, such that $i\in\pa_{\G}(j)\Leftrightarrow \pi(i)\in\pa_{\hat{\G}}(\pi(j))$, and for $\forall i\in[d]$, $\hat{\vh}_i\in\spanl{\vh_j: \pi(j)\in\overline{\dom}_{\G}(i)}$.
\end{lemma}

We also need to make the following assumption on noise.

\begin{assumption}
    \label{asmp:noise-not-rotation-invariant}
    The noise vector $\epsilon\in\R^d$ has independent components, at most one component is Gaussian distributed, and any two components have different distribution.
\end{assumption}

The non-gaussianity of the noise vectors is a typical assumption in causal discovery within linear models \citep{comon1994independent,silva2006learning} and is always assumed in the LinGAM setting \citep{shimizu2006linear}. The assumption that all components have a different distribution is not so standard, but is quite natural in real-world scenarios.

\subsection{Identifiability guarantee}

For each node $i\in[d]$ of $\G$, we use $\vw_{k}(i)$ to be the \textit{weight vector} of environment $E_k$ at node $i$, \emph{i.e.}, $\vw_k(i) = \left((\mA_k)_{ij}:j\in\pa_{\G}(i)\right) \in \R^{\absx{\pa_{\G}(i)}}$. In other words, the structural equation for node $i$ in environment $k$ is of the form:
\begin{equation}
    \label{eq:structural-equatin-zi}
    z_i = w_k(i)^\top z_{\mathrm{pa}_{\G}(i)} + \sqrt{\omega_{k,i,i}} \epsilon_i
\end{equation}
To obtain our identifiability result, the main assumption we need to make is the non-degeneracy of the weights at each node:
\begin{assumption}
    \label{asmp:independent-row-orig}
    For each node $i\in[d]$ of $\G$, we have $\mathrm{aff}\left(\vw_{k}(i):k\in[K]\right) = \R^{\absx{\pa_{\G}(i)}}$ where $\mathrm{aff}(\cdot)$ denotes the affine hull. Equivalently, the weights $\vw_{k}(i), k=1,2,\cdots,K$ do not lie in a  $\left(\absx{\pa_{\G}(i)}-1\right)$-dimensional hyperplane of $\R^{\absx{\pa_{\G}(i)}}$.
\end{assumption}

This assumption is quite mild since it only requires the weight vectors to be in general positions, and it holds with probability $1$ if the weights at each node are sampled from continuous distributions.
Moreover, as shown in \Cref{lemma:assumptions-equiv}, it is equivalent to the following assumption.

\begin{assumption}[Node-level non-degeneracy]
    \label{asmp:independent-row}
    We say that the matrices $\left\{\mB_k\right\}_{k=1}^K$ are node-level non-degenerate if for all node $i \in [d]$, we have $\dim\spanl{(\mB_k)_i: k\in[K]} = \absx{\pa_{\G}(i)}+1$, where $(\mB_k)_i$ is the $i$-th row of $\mB_k$.
\end{assumption}

In the following, we state our main result in this section, which shows that $K=d$ non-degenerate environments suffices for the model to be identifiable up to $\sim_{\dom}$.

\begin{theorem}
\label{thm:linear-main-thm}
    Suppose that $K\geq d$ 
    and we have access to observations generated from the linear causal model $(\mH,\G)$ across multiple environments $\mathfrak{E}=\left\{E_k: k\in[K]\right\}$ with observation distributions $\{\mathbb{P}_{\vx}^E\}_{E\in\mathfrak{E}}$, and the data generating processes are given by \Cref{latent-original}.
    Let $(\hat{\mH},\hat{\mathcal{G}})$ be any candidate solution with the hypothetical data generating process
    \begin{equation}
        \notag
        \vv = \hat{\mA}_k\vv+\hat{\bm{\Omega}}_k^{\frac{1}{2}}\hat{\epsilon},\quad \vx = \hat{\mH}^{\dagger}\vv \quad \text{in the environment $E_k$}
    \end{equation}
    where $\hat{\mH}$ has full row rank, such that
    \begin{enumerate}[label={(\roman*)}]
        \item the observation distribution that this hypothetical model generates in $E_k$ is exactly $\mathbb{P}_{\vx}^{E_k}$;
        \item all environments share the same causal graph: $\forall k\in[K]$ and $i,j\in[d]$, $(\mA_k)_{ij}\neq 0 \Leftrightarrow j\in\pa_{\G}(i)$, $(\hat{\mA}_k)_{ij}\neq 0 \Leftrightarrow j\in\pa_{\hat{\G}}(i)$ and $\bm{\Omega}_k, \hat{\bm{\Omega}}_k$ are diagonal matrices with positive entries;
        \item $\{\mB_k\}_{k=1}^K$ and $\left\{\hat{\mB}_k = \hat{\bm{\Omega}}_k^{-\frac{1}{2}}(\mI-\hat{\mA}_k)\right\}_{k=1}^K$ are non-degenerate in the sense of \Cref{asmp:independent-row};
        \item the noise variables $\epsilon$ and $\hat{\epsilon}$ satisfy \Cref{asmp:noise-not-rotation-invariant}.
    \end{enumerate}
    Then we must have $(\mH,\G)\sim_{\dom}(\hat{\mH},\hat{\G})$.
\end{theorem}

The proof of \Cref{thm:linear-main-thm} is given in \Cref{proof:linear-main-thm}. In the next section, we will introduce an algorithm, \texttt{LiNGCReL}, that provably recovers the ground-truth up to $\sim_{\dom}$. 

To the best of our knowledge, this is the first identifiability guarantee in the literature for causal representation learning from general environments. Remarkably, while the fact that existing works \citep{seigal2022linear,zhang2023identifiability} focus on single-node interventions seem to suggest that learning from \textit{diverse} environments is hard, our result indicates that such diversity is actually helpful. Specifically, we show that in the worst case, $\Theta(d^2)$ interventions are required for identifying the ground-truth model under $\sim_{\dom}$:

\begin{theorem}[informal version of \Cref{thm:single-node-lower-bound}]
\label{thm:single-node-lower-bound-informal}
    There exists a causal graph $\G$ with $\Theta(d^2)$ edges, such that for any unmixing matrix $\mH\in\R^{d\times n}$ with full row rank, any independent noise variables $\epsilon$, and any $0 < s_i \leq \absx{\pa_{\G}(i)}, i\in[d]$, the ground-truth model $(\mH,\G)$ is non-identifiable up to $\sim_{\dom}$ with $s_i$ soft interventions for node $i$, unless the (ground-truth and intervened) weights of the causal model lie in a null set (w.r.t the Lebesgue measure).
\end{theorem}

A formal version and the proof of \Cref{thm:single-node-lower-bound-informal} can be found in \Cref{proof:single-node-lower-bound}. On the other hand, by having $d$ single-node interventions per node, \Cref{asmp:independent-row} can be satisfied as long as the weights are in general positions, so in this case we have $(\mH,\G)\sim_{\dom}(\hat{\mH},\hat{\G})$ by \Cref{thm:linear-main-thm}. Therefore, \Cref{thm:linear-main-thm,thm:single-node-lower-bound} together imply that $\Theta(d^2)$ single-node interventions are necessary and sufficient for identification up to $\sim_{\dom}$.

Given that \Cref{thm:linear-main-thm} only guarantees identification up to $\sim_{\dom}$ that is strictly weaker than full identification, one might naturally ask whether \Cref{thm:linear-main-thm} can be further improved. Our last theorem in this section indicates that $\sim_{\dom}$ is indeed the \emph{best} one can hope for in our setting, even assuming access to \emph{single node, soft intervention}.

\begin{theorem}[Counterpart to \Cref{thm:linear-main-thm}, informal version of \Cref{thm:linear-ambiguity}]
\label{thm:informal-linear-ambiguity}
    For any linear causal model $(\mH,\G)$ and any set of environments $\mathfrak{E}=\left\{E_k: k\in[K]\right\}$ such that all conditions in \Cref{thm:linear-main-thm} are satisfied, there must exists a candidate solution $(\hat{\mH},\G)$ and a hypothetical data generating process that satisfy the same set of conditions, but
    \begin{equation}
        \notag
        \frac{\partial \vv_i}{\partial \vz_j} \neq 0,\quad \forall j\in\overline{\dom}_{\G}(i).
    \end{equation}
    Moreover, if we additionally assume that the environments are groups of single-node soft interventions, then we can guarantee the existence of $(\hat{\mH},\G)$ and weight matrices which, besides the properties listed above, are also groups of single-node soft interventions.
\end{theorem}

\section{\texttt{LinGCReL}: Algorithm for linear non-Gaussian causal representation learning}
\label{sec:algorithm}

In this section, we describe Linear Non-Gaussian Causal Representation Learning (LiNGCReL), an algorithm that provably recovers the underlying causal graph and latent variables up to $\sim_{\dom}$ in the infinite-sample limit. At this point, it is instructive to recall the celebrated \texttt{LiNGAM} algorithm \citep{shimizu2006linear} for linear causal graph discovery. Different from their setting, we only observe some unknown linear mixture of the latent variables. Hence, running linear ICA as in \texttt{LiNGAM} only gives us $\mM_k=\mB_k\mH$ rather than the weight matrix $\mB_k$ itself.

The key idea in our approach is an effect cancellation scheme that allows us to determine the \enquote{remaining degree of freedom} (RDF) of any node (\emph{a.k.a.} latent variable) given any subset of its ancestors. This scheme allows us to not only find a topological order of the nodes, but also figure out direct causes by tracking the changes of the RDF. In the following, we present the main steps of \texttt{LiNGCReL} in more details.

Suppose that we are given samples of observations $\mX^{(k)}=\left\{\vx_i^{(k)}\right\}_{i=1}^N, k\in[K]$ where $\vx_i^{(k)}$ is the $i$-th sample from the $k$-th environment. 

\textbf{Step 1. Recover the matrices $\mM_k=\mB_k\mH$} Since $\epsilon = \mB_k\vz = \mB_k\mH\vx$ in the $k$-th environment, so we can use any identification algorithm for linear ICA to recover the matrix $\mM_k$. Then we properly rearrange the rows of $\mM_k$ so that all $\mM_k\vx, k=1,2,\cdots,K$ correspond to the same permutation of noise variables. This step is quite standard and details can be found in \Cref{appsubsec:align-matrices_Mk}.

\textbf{Step 2. Causal representation learning based on $\mM_k$} Now we have obtained $\mM_k = \mB_k\mH$, but the unmixing matrix $\mH$ is still unknown. We propose \Cref{alg:learn-graph} to learn $\mH$ and the causal graph $\G$. The main part of \Cref{alg:learn-graph} contains a loop that maintains a node set $S$ which, we will show later, is ancestral, \emph{i.e.,} $i\in S \Rightarrow \ans_{\G}(i)\subseteq S$. In each round the algorithm finds a new node $i\notin S$ such that $\ans_{\G}(i)\subseteq S$, and a subroutine \texttt{Identify-Parents} (\Cref{alg:identify-parents}) is used to find all parents of $i$. After that, we append $i$ into $S$ and continue until $S$ contains all nodes in $\G$. Finally, the rows of the mixing matrix $\mH$ is obtained by intersections of properly-chosen row spaces of $\mM_k$.

Both \Cref{alg:identify-parents} and \Cref{alg:learn-graph} include a crucial step, which we call it \textit{orthogonal projection}, as described in \Cref{alg:effect-cancellation}. At a high level, it helps determine the minimal RDF for $\vz_i$ after fixing the latent variables $\vz_S$, and this exactly corresponds to the number of parents of $\vz_i$ that are not in $\vz_S$. We provide a simple example in \Cref{appsec:lingcrel-example} to illustrate why this approach works.

The following result states that \Cref{alg:learn-graph} can recover the ground-truth causal model up to $\sim_{\dom}$:

\begin{theorem}
\label{thm:alg-guarantee}
    Suppose that $\mM_k, k\in[K]$ are perfectly identified in \textbf{Step 1}. Let $(\hat{\mH},\hat{\G})$ be the solution returned by \Cref{alg:learn-graph}, then we must have $(\mH,\G)\sim_{\dom}(\hat{\mH},\hat{\G})$.
\end{theorem}

The full proof of \Cref{thm:alg-guarantee} is given in \Cref{proof:alg-guarantee}. It crucially relies on the following two propositions that reveal how \Cref{alg:learn-graph} and the subroutine \Cref{alg:identify-parents} work. 

\begin{algorithm}
\caption{\texttt{Orthogonal-projections}}\label{alg:effect-cancellation}
\begin{algorithmic}[1]
    \STATE {\bfseries Input:} Ordered set $S = \left\{ s_1, s_2,\cdots, s_m\right\}\subseteq[d]$, index $i\notin S$,  matrices $\mM_k\in\R^{d\times n}$, $k\in [K]$
    \STATE {\bfseries Output:} Set of vectors $\{\vp_k\}_{k=1}^K$
    \FOR{$k \gets 1$ {to} $K$}
    \STATE $\mW \gets \spanl{(\mM_k)_s: s\in S}$ \\
    \COMMENT{$(\mM_k)_s$ is the $s$-th row of $\mM_k$}\label{line:parents-span}
    \STATE $\vp_k \gets \mathrm{proj}_{\mW^\perp}\left((\mM_k)_i\right)$
    \ENDFOR
\end{algorithmic}
\end{algorithm}

\begin{proposition}
\label{prop:alg2}
    The following two propositions hold for \Cref{alg:learn-graph}:
    \begin{itemize}
        \item $\ans_{\G}(i)\subseteq S \Leftrightarrow$ the \texttt{if} condition in line 8 of \Cref{line:if-rank1} is fulfilled;
        \item the set $S$ maintained in \Cref{alg:learn-graph} is always an ancestral set, in the sense that $j\in S \Rightarrow \ans_{\G}(j)\subseteq S$.
    \end{itemize}
\end{proposition}

\begin{proposition}
\label{prop:alg1}
    Given any ordered ancestral set $S$ that contains $\pa_{\G}(i)$ for some $i\notin S$,
    \Cref{alg:identify-parents} returns a set $P_i \subseteq S$ that is exactly $\pa_{\G}(i)$.
\end{proposition}

\begin{algorithm}
\caption{\texttt{Identify-Parents}}\label{alg:identify-parents}
\begin{algorithmic}[1]
    \STATE {\bfseries Input:} An ordered set $S = \left\{ s_1, s_2,\cdots, s_m\right\}\subseteq[d]$, a node $i\notin S$ and matrices $\mM_k, k\in[K]$
    \STATE {\bfseries Output:} The parent set $P_i$ of node $i$
    \STATE $P_i \gets \emptyset$
    \FOR{$m' \gets 0$ {to} $m$}
    \STATE $\{\vp_k\}_{k=1}^K \gets \texttt{Orthogonal-projections}\left(\right.$\newline
    \hspace*{2em} $\left.\{s_j:j\leq m'\},i,\{\mM_k\}_{k\in [K]}\right)$
    \STATE $r_{m'} \gets \dim\mathrm{span}\left\langle \vp_k: k\in[K]\right\rangle$\label{line:parents-obtain-rank}
    \IF{$m'\geq 1$ and $r_{m'} = r_{m'-1}-1$\label{line:parents-rank}}
    \STATE $P_i \gets P_i \cup \{m'\}$
    \ENDIF
    \ENDFOR
\end{algorithmic}
\end{algorithm}

\begin{algorithm}
\caption{\texttt{Learn-Causal-Model}}\label{alg:learn-graph}
\begin{algorithmic}[1]
    \STATE {\bfseries Input:} Matrices $\mM_k, k\in[K]$
    \STATE {\bfseries Output:} The edge set $\mathcal{E}$ on the vertex set $[d]$ and the mixing matrix $\hat{\mH}$
    \STATE $S \gets \emptyset$; \COMMENT{$S$ is an ordered set}
    \STATE $\mathcal{E} \gets \emptyset$
    \WHILE{$\absx{S}<d$}
    \FOR{$i\notin S$}
    \STATE $\{\vp_k\}_{k=1}^K \gets \texttt{Orthogonal-projections}\left(\right.$\newline \hspace*{2em}$\left.S,i,\{\mM_k\}_{k\in [K]}\right)$
    \IF{$\dim\mathrm{span}\left\langle \vq_k: k\in[K]\right\rangle=1$}\label{line:if-rank1}
    \BREAK \COMMENT{we will show that such an $i$ must exist}
    \ENDIF
    \ENDFOR
    \STATE $P_i \gets \texttt{Identify-Parents}(S,i)$
    \STATE $S \gets S \cup \{i\}$
    \STATE $\mathcal{E} \gets \mathcal{E}\cup\left\{ (j,i): j\in P_i\right\}$
    \ENDWHILE
    \FOR{$i=1$ {to} $d$}
    \STATE $E_i \gets \mathrm{span}\left\langle (\mM_k)_i: k\in[K]\right\rangle$
    \ENDFOR
    \FOR{$i=1$ {to} $d$}
    \STATE $\hat{\vh}_i \gets \text{any non-zero vector in} \left(\cap_{j:(i,j)\in\mathcal{E}} E_j\right)\cap E_i$\label{line:intersect-subspaces}
    \ENDFOR
    \STATE $\hat{\mH} \gets \left[ \hat{\vh}_1^\top, \hat{\vh}_2^\top,\cdots,\hat{\vh}_d^\top\right]^\top$\;
\end{algorithmic}
\end{algorithm}

\section{Experiments}
\label{sec:experiment-results}

In this section, we present our experimental setup and results for \texttt{LiNGCReL}. Note that \texttt{LiNGCReL} as described in the previous section only works in the population regime. When the number of samples is limited, two main challenges in implementing \texttt{LiNGCReL} are to accurately compute the dimension of a subspace (line 6 of \Cref{alg:identify-parents} and line 8 of \Cref{alg:learn-graph}), and to find a vector in the intersection of multiple subspaces (line 20, \Cref{alg:learn-graph}). Due to space limit, the implementation details are described in \Cref{appsubsec:finite-sample-implementation}.

\textbf{Experimental setup.} We generate the independent noise variables from generalized Gaussian distributions $p_{\beta}(x)\propto \exp\left(-\absx{x}^{\beta}\right)$ with parameters $\beta_k = 0.2 k^2, k=1,2,\cdots,d$, multiplied by normalization constants to make their variances equal to $1$. The ground-truth causal graph is generated by first fixing a total order of the vertices, say $1,2,\cdots,d$, then add directed edges $i\to j (i<j)$ according to i.i.d. Bernoulli($p$) distributions, where $p\in(0,1)$. The non-zero entries of matrices $\mB_k$ and $\mH$ are all generated independently from Gaussian distributions. For simplicity, we focus on the case $n=d$ since recovery of the latent graphs only requires information from $d$ components of $\vx$.

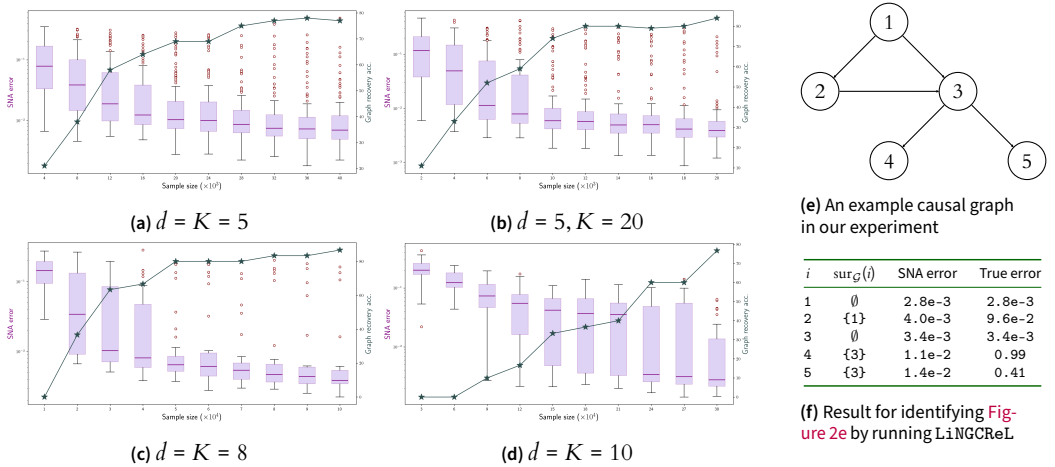
\begin{figure*}[t]
    \centering
    \begin{subfigure}[c]{0.35\textwidth}
    \setcounter{subfigure}{0}
        \centering
        \includegraphics[width=\textwidth]{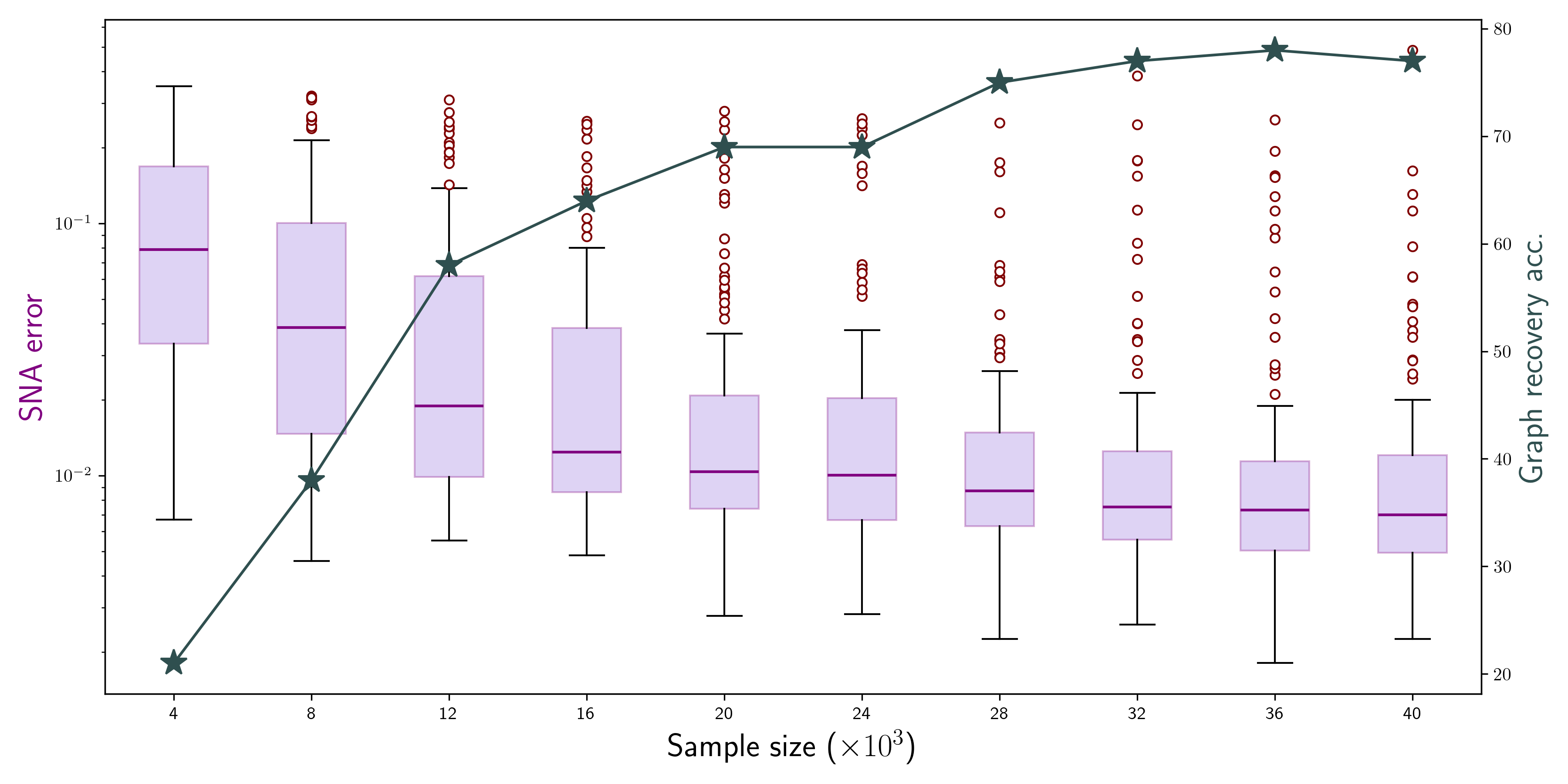}
        \caption[]%
        {{\small $d=K=5$}}    
        \label{d=K=5}
    \end{subfigure}
    \begin{subfigure}[c]{0.35\textwidth} 
    \setcounter{subfigure}{1}
        \centering 
        \includegraphics[width=\textwidth]{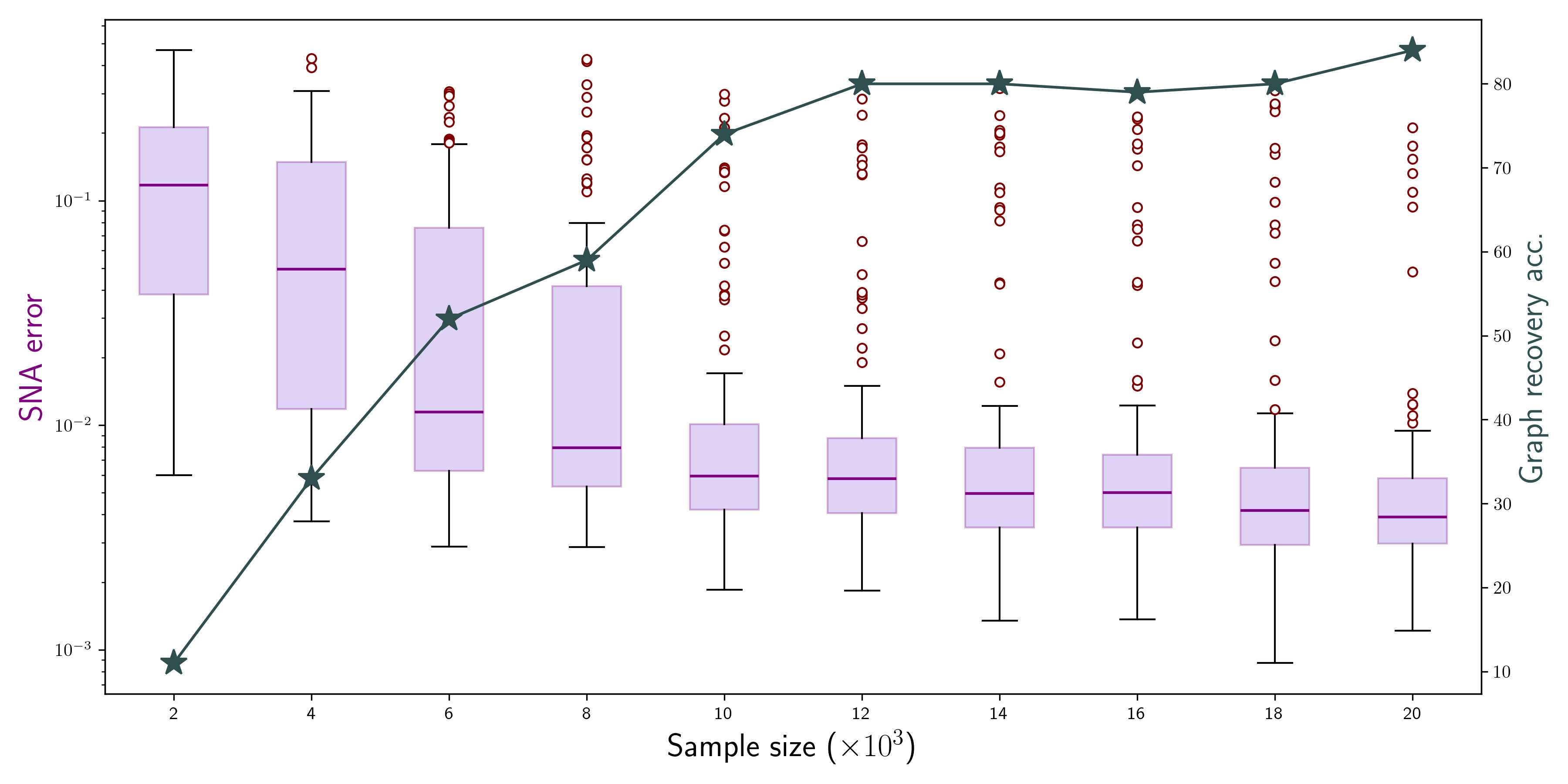}
        \caption[]%
        {{\small $d=5, K=20$}}    
        \label{d=5K=20}
    \end{subfigure}
    \qquad
    \begin{subfigure}[c]{0.2\textwidth}
    \setcounter{subfigure}{4}
        \begin{tikzpicture}[node distance=0.55cm and 0.55cm, every node/.style={draw, circle, minimum size=.15cm}, >={Latex[width=0.25mm,length=0.25mm]}]
            \node (1) {$1$};
            \node (2) [below left=of 1] {$2$};
            \node (3) [below right=of 1] {$3$};
            \node (4) [below left=of 3] {$4$};
            \node (5) [below right=of 3] {$5$};
        
            \draw[->] (1) -- (2);
            \draw[->] (1) -- (3);
            \draw[->] (2) -- (3);
            \draw[->] (3) -- (4);
            \draw[->] (3) -- (5);
        \end{tikzpicture}
        \caption{An example causal graph in our experiment}
        \label{fig:example-graph-in-experiment}
    \end{subfigure}
    \vskip\baselineskip\vspace{-0.15in}
    \begin{subfigure}[c]{0.35\textwidth} 
    \setcounter{subfigure}{2}
        \centering 
        \includegraphics[width=\textwidth]{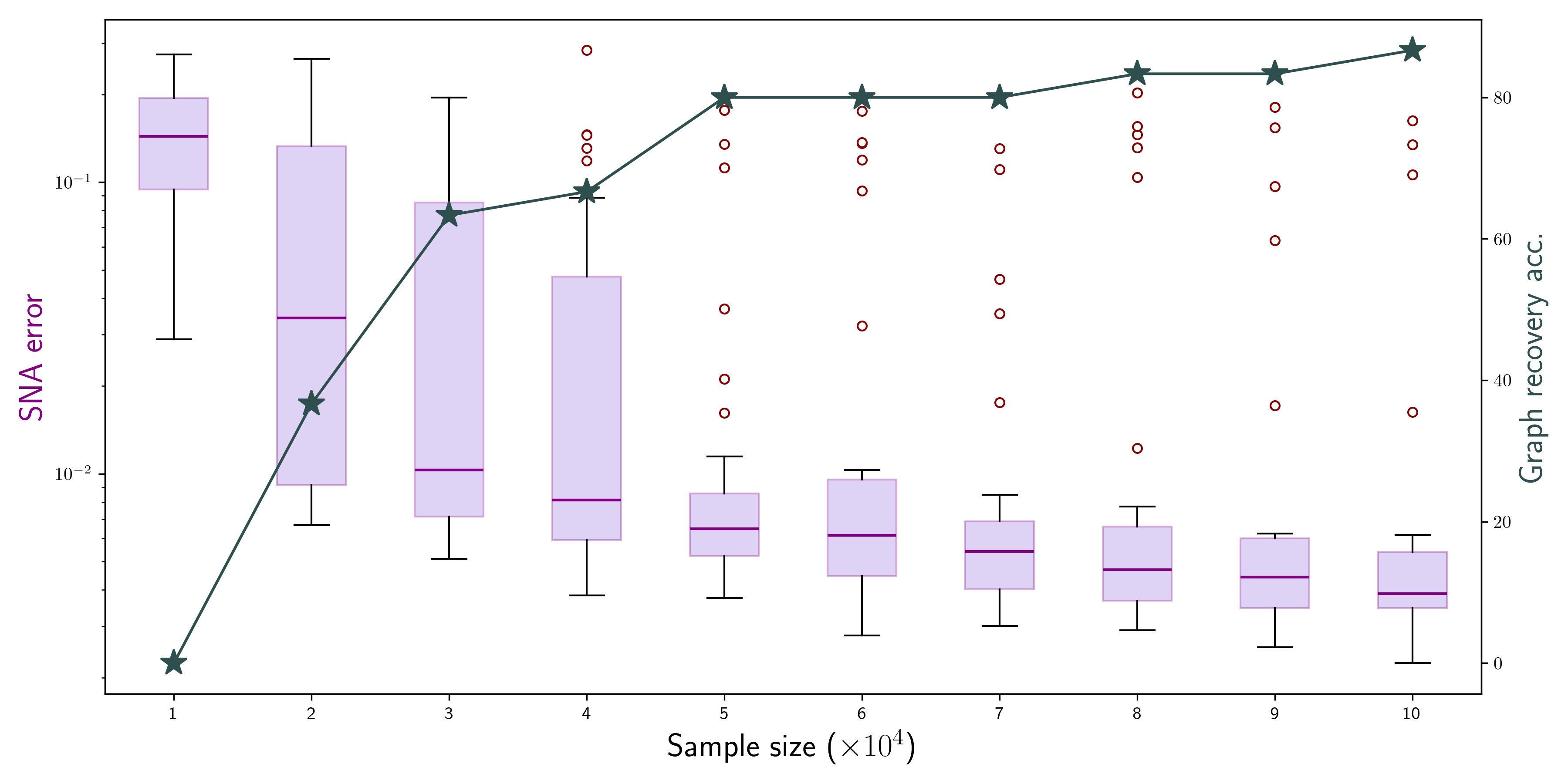}
        \caption[]%
        {{\small $d=K=8$}}    
        \label{d=K=8}
    \end{subfigure}
    \begin{subfigure}[c]{0.35\textwidth}   
    \setcounter{subfigure}{3}
        \centering 
        \includegraphics[width=\textwidth]{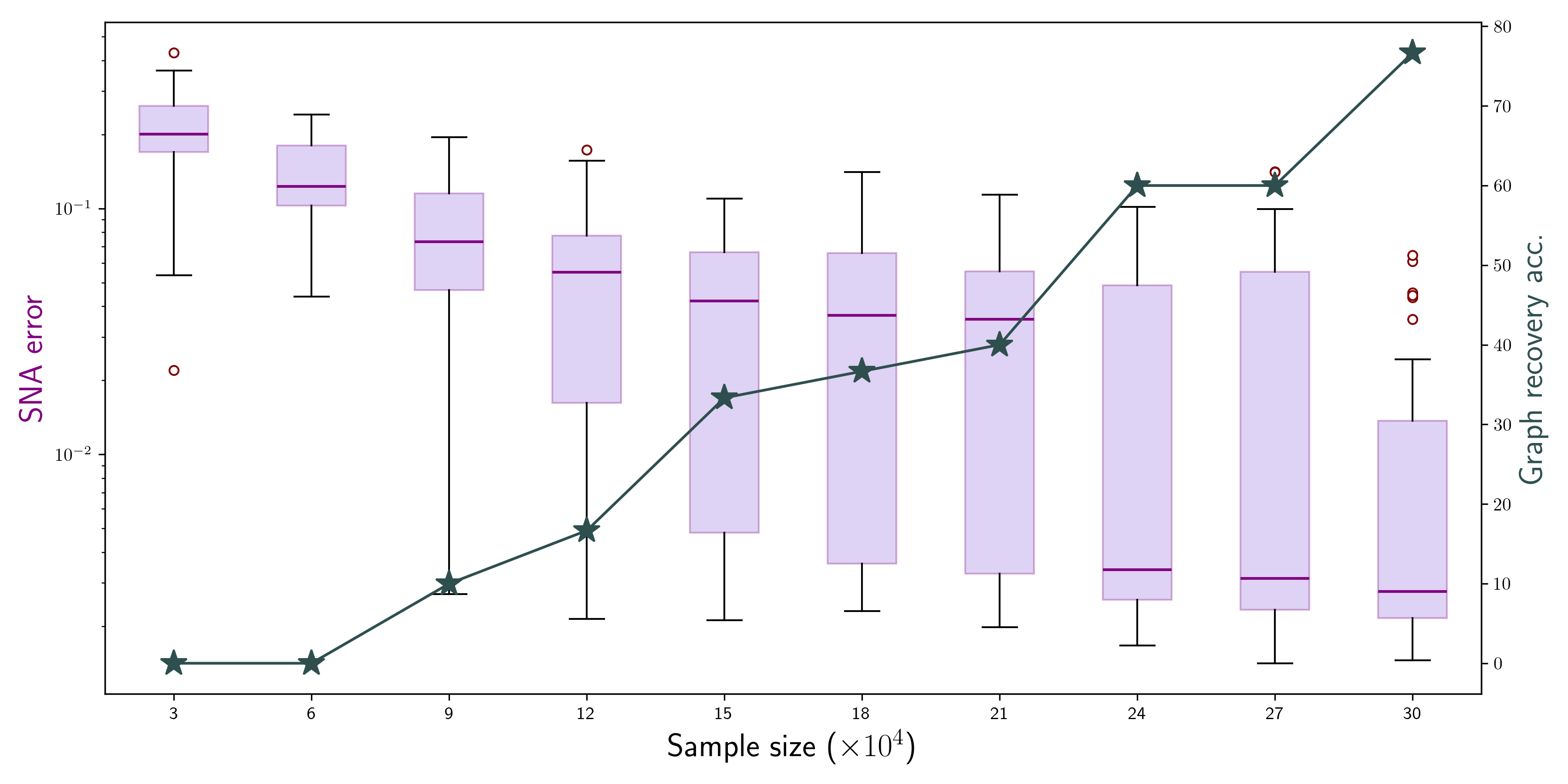}
        \caption[]%
        {{\small $d=K=10$}}    
        \label{d=K=10}
    \end{subfigure}
    \qquad
    \begin{subfigure}[c]{0.2\textwidth}
    \setcounter{subfigure}{5}
        \centering
        \scalebox{0.7}{
        \begin{tabular}{@{}cccc@{}}
        \toprule
        $i$        & $\dom_{\G}(i)$       & SNA error       & True error      \\ \midrule
        \texttt{1} & $\mathtt{\emptyset}$ & \texttt{2.8e-3} & \texttt{2.8e-3} \\
        \texttt{2} & \texttt{\{1\}}       & \texttt{4.0e-3} & \texttt{9.6e-2} \\
        \texttt{3} & $\mathtt{\emptyset}$  & \texttt{3.4e-3} & \texttt{3.4e-3} \\
        \texttt{4} & \texttt{\{3\}}       & \texttt{1.1e-2} & \texttt{0.99}   \\
        \texttt{5} & \texttt{\{3\}}       & \texttt{1.4e-2} & \texttt{0.41}   \\ \bottomrule
        \end{tabular}}
        \caption{Result for identifying \Cref{fig:example-graph-in-experiment} by running \texttt{LiNGCReL}}
        \label{tab:example-result}
    \end{subfigure}
    \caption[]%
    {Left: plots of SNA Error and graph recovery accuracy achieved by \texttt{LiNGCReL} as functions of sample size (per environment) for different choices of graph size $d$ and number of environments $K$. Right: an example of causal graph generated in our experiments, and the estimation error of \texttt{LiNGCReL} for each node.}
    \label{boxplots}
\end{figure*}

\textbf{Metrics of estimation error.} Since causal representation learning seeks to learn both the causal graphs and the latent variables, for each output of our algorithm we first check if it exactly recovers the ground-truth causal graph. Then, recall that the latent variables and the observations are related by $\vz = \mH\vx$, given any output unmixing matrix $\hat{\mH}$ from \Cref{alg:learn-graph}, we define the relative estimation error $\Delta_i$ for $\vz_i$ as the solution of the following optimization problem:
\begin{equation}
    \label{eq:def-estimation-error}
    \begin{aligned}
        & \min \normx{\bm{\Delta}}_{\infty}\quad s.t. \Delta_i=\frac{\normx{\mathrm{proj}_{\mathrm{span}\left\langle\vh_j:j\in\overline{\dom}_{\G}(i)\right\rangle}(\hat{\hat{\vh}}_i)}_2}{\normx{\hat{\hat{\vh}}_i}_2},\\
        &\quad \hat{\hat{\mH}} = \mP\hat{\mH} \text{ for some signed permutation matrix } \mP.
    \end{aligned}
\end{equation}
where signed permutation is allowed here since the noise distribution in our experiments is symmetric and the order of latent variables $\vz_i, i=1,2,\cdots,d$ does not matter. We refer to the errors $\Delta_i$ defined in \Cref{eq:def-estimation-error} as the \textit{SNA error}. The SNA error measures how much of the row $\hat{\hat{\vh}}_i$ that we learn is contained in the span of the ground-truth rows $\vh_j, j\in\overline{\dom_{\G}}(i)$. Indeed, recall that given any observation $\vx$, the ground-truth latent variable is $\vz = \mH\vx$ while our algorithm outputs $\hat{\vv}_i = \hat{\hat{\vh}}_i^\top\vx$, so the SNA error essentially captures whether the recovered latent variable is close to some linear mixture of latent variables in the effect-dominating set of $i$. When the SNA error is zero for some node $i$, we know that the recovered latent variable at node $i$ is exactly a linear mixture of the ground-truth latent variables in $\overline{\dom_{\G}}(i)$, according to \Cref{lem:mixture}.

We also define the \textit{true error} for estimating each latent variable. Formally, let $\hat{\hat{\mH}}$ be the unmixing matrix that corresponds to the solution of \Cref{eq:def-estimation-error}, then we define the true estimation error $\tilde{\Delta}_i$ of $\vz_i$ as
\begin{equation}
    \label{eq:true-error}
    \tilde{\Delta}_i = \normx{\left(\mI-\vh_i\vh_i^\top\right)\hat{\hat{\vh}}_i}_2.
\end{equation}

\textbf{Results.} We randomly sample $100$ causal models with size $d=5$, $30$ causal models with size $d=8$ ad $30$ causal models of size $d=10$. In light of \Cref{thm:linear-main-thm}, for each $d\in\{5,8,10\}$, we sample data from $K=d$ randomly chosen environments; for $d=5$ we also consider $K=20$ to study how different choices of $K$ can affect the result. We run \texttt{LiNGCReL} for each model with different sample sizes, compute the SNA error and true error of the obtained solution from \Cref{eq:def-estimation-error,eq:true-error} respectively for each latent variable, and check whether the ground-truth causal graph is exactly recovered.

\Cref{boxplots} shows how the average SNA error (over all latent variables) and the accuracy of graph recovery changes when sample size grows. We can see \texttt{LiNGCReL} successfully recovers about $80\%$ of all models within each category, and the median of the average SNA error is smaller than $1\%$. Moreover, by comparing \Cref{d=K=5} with \Cref{d=5K=20}, one can observe that if we fix the total number of samples but choose a larger $K$ (\emph{i.e.}, fewer samples per environment), \texttt{LiNGCReL} can still achieve the same level of performance compared with the choice $K=d$. Intuitively, this is because $K \gg d$ vectors sampled from an $r (r\leq d)$ dimensional subspace are unlikely to approximately lie in an ($r-1$)-dimensional subspace, so that the calculation of line 6 of \Cref{alg:identify-parents} and line 8 of \Cref{alg:learn-graph} can be more accurate. We leave a better and quantitative understanding of the trade-off between $d$ and $K$ to future work.

\textbf{SNA error v.s. true error.} To understand the implication of our theory, we dive deeper by looking into the learning outcome of \texttt{LiNGCReL} on a specific model, of which the causal graph is shown in \Cref{fig:example-graph-in-experiment}.

In \Cref{tab:example-result}, we list the surrounding set of each node and the corresponding SNA error and true error. We can see that if $\dom_{\G}(i)=\emptyset$, the two errors equal and both are small, but if $\dom_{\G}(i)\neq\emptyset$, the true error is much larger than the SNA error. This indicates that \texttt{LiNGCReL} indeed learns the ground-truth model up to $\sim_{\dom}$, as \Cref{thm:linear-main-thm} predicts.

\section{Identification limit of general causal models with soft interventions}

While \Cref{thm:linear-main-thm} guarantees identifiability with general environments, it only applies to linear causal models. In this section, we show that if we have access to single-node soft interventions, then we can identify general non-parametric causal models up to $\sim_{\dom}$. 
To obtain our identifiability result, we also require that the environments are non-degenerate in the following sense:

\begin{definition}[Non-degeneracy set of interventions]
\label{non-degeneracy-distribution}
    Let $\hat{p}_k\left(\vz_i\mid\vz_{\mathrm{pa}_{\G}(i)}\right), k\in[K_i]$ be conditional probability densities at node $i$, then $\left\{\hat{p}_k\right\}_{k=1}^{K_i}$ is said to be non-degenerate on node $i$ at point $\hat{\vz}\in\R^d$ if all these conditional densities are well-defined and positive at $\hat{\vz}$, and the matrix
    \begin{equation}
        \notag
    \left[\frac{\partial\left(\hat{p}_1/\hat{p}_k\right)}{\partial \vz_j}\right]_{2\leq k\leq K_i, j\in\bar{\mathrm{pa}}_{\G}(i)}\Biggm\vert_{\vz=\hat{\vz}} \in \R^{(K_i-1)\times \left(\absx{\pa_{\G}(i)}+1\right)}
    \end{equation}
    has full row rank. Moreover, we say that  $\left\{\hat{p}_k\right\}_{k=1}^{K_i}$ is non-degenerate in a point set $\mO$ if for all $\hat{\vz}\in\mO$, it is non-degenrate at $\hat{\vz}$. 
\end{definition}

The following lemma shows how \Cref{non-degeneracy-distribution} is related to \Cref{asmp:independent-row} in the linear setting:

\begin{lemma}
\label{lemma:assumptions-relation}
    Suppose that $\hat{p}_k(\vz)=\prod_{i=1}^d \hat{p}_k\left(\vz_i\mid\vz_{\pa_{\G}(i)}\right), k\in[K]$ be probability distributions of latent variables $\vz$ generated from the linear causal models \Cref{latent-original}, such that for $\forall i\in[d]$, $\hat{p}_k\left(\vz_i\mid\vz_{\pa_{\G}(i)}\right), k\in[K]$ are non-degenerate on node $i$ in the sense of \Cref{non-degeneracy-distribution}. Then the corresponding matrices $\mB_k, k\in[K]$ satisfy \Cref{asmp:independent-row}.
\end{lemma}

Now we are ready to state our main result in this section:

\begin{theorem}
\label{thm:single-node-soft-nonparam}
    Suppose that we have access to observations generated from multiple environments $\{P_{\mX}^E\}_{E\in\mathfrak{E}}$. Let $\left(\hat{\vh},\hat{\G}\right)$ be any candidate solution with data generated according to \Cref{asmp:data-generating-process} with latent variables $\vv = \hat{\vh}(\vx)$ 
    and joint distribution $q_E$ with factors $q_{i}^E$. Assuming that
    \begin{enumerate}[label={(\roman*)}]
        \item the joint densities $\left\{p_E(\vz)\right\}_{E\in\mathfrak{E}}$ are continuous differentiable on $\R^d$ with common support $\mO_{\vz}$, and $\left\{q_E(\vv)\right\}_{E\in\mathfrak{E}}$ 
        are continuous differentiable on $\R^d$ with common support $\mO_{\vv}$; 
        \item we have access to multiple single-node soft interventions on each node with unknown targets: there exists a partition $\mathfrak{E} = \cup_{i=1}^d\mathfrak{E}_i$ such that $\mathcal{I}_{\vz}^{\mathfrak{E}_i} = \{\pi(i)\}, \mathcal{I}_{\vv}^{\mathfrak{E}_i} = \{\pi'(i) \},\forall i\in[d]$ for some unknown permutations $\pi$ and $\pi'$ on $[d]$;
        \item the intervention distributions on each node are non-degenerate in the sense of \Cref{non-degeneracy-distribution}: 
        there exists $\mN_{\vz}\subseteq\mO_{\vz}$ and $\mN_{\vv}\subseteq\mO_{\vv}$ satisfying $\mN_{\vz}^{\mathrm{o}} = \mN_{\vv}^{\mathrm{o}} = \emptyset$
        where $S^{\mathrm{o}}$ denotes the interior of a set $S$, such that for all $i\in[d]$, $\left\{p_{i}^E(\cdot): E\in\mathfrak{E}_{\pi^{-1}(i)}\right\}$ (resp. $\left\{q_{i}^E(\cdot): E\in\mathfrak{E}_{\pi'^{-1}(i)}\right\}$) is non-degenerate on node $i$ in $\mO_{\vz}\setminus\mN_{\vz}$ (resp. $\mO_{\vv}\setminus\mN_{\vv}$).
    \end{enumerate}
    Then we must have $(\vh,\G)\sim_{\dom}(\hat{\vh},\hat{\G})$.
\end{theorem}

Previous works on the identifiability of non-parametric causal models typically require that all the joint distributions are supported on the whole space $\R^d$ \citep{von2023nonparametric,liang2023causal,varici2023general}. In contrast, we only assume that the densities have common and unknown support across all interventions.

\Cref{thm:single-node-soft-nonparam} can be regarded as a soft-intervention version of ~\citealp[Theorem 4.3]{von2023nonparametric}, which assumes access to hard interventions and only need two paired interventions per node. While they are able to show full identifiability, we show in the following that identifiability up to $\sim_{\dom}$ is the best we can hope for with soft interventions. 

\begin{theorem}[Counterpart to \Cref{thm:single-node-soft-nonparam}, informal version of \Cref{thm:non-param-ambiguity}]
\label{thm:informal-non-param-ambiguity}
    For any causal model $(\vh,\G)$ and any set of environments $\mathfrak{E}=\left\{E_k: k\in[K]\right\}$ such that all conditions in \Cref{thm:single-node-soft-nonparam} are satisfied, there must exists a candidate solution $(\hat{\vh},\G)$ and a hypothetical data generating process that satisfy the same set of conditions, but
    \begin{equation}
        \notag
        \frac{\partial \vv_i}{\partial \vz_j} \neq 0,\quad \forall j\in\overline{\dom}_{\G}(i).
    \end{equation}
    Finally, the ambiguity still exists if we additionally assume standard axioms such as causal minimality (\Cref{asmp:minimality}) and faithfulness (\Cref{asmp:faithfulness}) on the causal model. 
\end{theorem}
\label{sec:non-param}


\section{Conclusions}
\label{sec:conclusion}

This paper studies the limit of learning identifiable causal representations using data from multiple environments. When hard interventions are not available, we provide theory and algorithm for identification up to SNA, and also show that SNA is an intrinsic ambiguity in our setting.

It is interesting to further investigate the setting where we do not assume that the causal model is linear. Moreover, it is important to understand the concrete form of available interventions in real-world applications. For instance, it is suggested that for single-cell genomics, the intervention is sometimes a "mixture" of hard and soft interventions, and sometimes can even reverse the direction of an edge \citep{tejada2023causal}. Modelling such more complicated interventions appears to be crucial to reveal the underlying causal mechanisms in real-world problems.

\section{Broader Impact}

This paper presents work whose goal is to advance the field of Machine Learning and in particular the sub-field Causal Representation Learning. There are many potential societal consequences of our work, especially as it pertains to building more reliable machine learning models, none which we feel must be specifically highlighted here.

\printbibliography

@inproceedings{locatello2020weakly,
  title={Weakly-supervised disentanglement without compromises},
  author={Locatello, Francesco and Poole, Ben and R{\"a}tsch, Gunnar and Sch{\"o}lkopf, Bernhard and Bachem, Olivier and Tschannen, Michael},
  booktitle={International Conference on Machine Learning},
  pages={6348--6359},
  year={2020},
  organization={PMLR}
}

@article{brehmer2022weakly,
  title={Weakly supervised causal representation learning},
  author={Brehmer, Johann and De Haan, Pim and Lippe, Phillip and Cohen, Taco S},
  journal={Advances in Neural Information Processing Systems},
  volume={35},
  pages={38319--38331},
  year={2022}
}

@inproceedings{khemakhem2020variational,
  title={Variational autoencoders and nonlinear ica: A unifying framework},
  author={Khemakhem, Ilyes and Kingma, Diederik and Monti, Ricardo and Hyvarinen, Aapo},
  booktitle={International Conference on Artificial Intelligence and Statistics},
  pages={2207--2217},
  year={2020},
  organization={PMLR}
}

@inproceedings{locatello2019challenging,
  title={Challenging common assumptions in the unsupervised learning of disentangled representations},
  author={Locatello, Francesco and Bauer, Stefan and Lucic, Mario and Raetsch, Gunnar and Gelly, Sylvain and Sch{\"o}lkopf, Bernhard and Bachem, Olivier},
  booktitle={international conference on machine learning},
  pages={4114--4124},
  year={2019},
  organization={PMLR}
}

@inproceedings{ahuja2023interventional,
  title={Interventional causal representation learning},
  author={Ahuja, Kartik and Mahajan, Divyat and Wang, Yixin and Bengio, Yoshua},
  booktitle={International Conference on Machine Learning},
  pages={372--407},
  year={2023},
  organization={PMLR}
}

@article{von2023nonparametric,
  title={Nonparametric Identifiability of Causal Representations from Unknown Interventions},
  author={von K{\"u}gelgen, Julius and Besserve, Michel and Liang, Wendong and Gresele, Luigi and Keki{\'c}, Armin and Bareinboim, Elias and Blei, David M and Sch{\"o}lkopf, Bernhard},
  journal={arXiv preprint arXiv:2306.00542},
  year={2023}
}

@inproceedings{lu2021invariant,
  title={Invariant causal representation learning for out-of-distribution generalization},
  author={Lu, Chaochao and Wu, Yuhuai and Hern{\'a}ndez-Lobato, Jos{\'e} Miguel and Sch{\"o}lkopf, Bernhard},
  booktitle={International Conference on Learning Representations},
  year={2021}
}

@article{von2021self,
  title={Self-supervised learning with data augmentations provably isolates content from style},
  author={Von K{\"u}gelgen, Julius and Sharma, Yash and Gresele, Luigi and Brendel, Wieland and Sch{\"o}lkopf, Bernhard and Besserve, Michel and Locatello, Francesco},
  journal={Advances in neural information processing systems},
  volume={34},
  pages={16451--16467},
  year={2021}
}

@article{silva2006learning,
  title={Learning the Structure of Linear Latent Variable Models.},
  author={Silva, Ricardo and Scheines, Richard and Glymour, Clark and Spirtes, Peter and Chickering, David Maxwell},
  journal={Journal of Machine Learning Research},
  volume={7},
  number={2},
  year={2006}
}

@article{seigal2022linear,
  title={Linear causal disentanglement via interventions},
  author={Seigal, Anna and Squires, Chandler and Uhler, Caroline},
  journal={arXiv preprint arXiv:2211.16467},
  year={2022}
}

@article{liang2023causal,
  title={Causal Component Analysis},
  author={Liang, Wendong and Keki{\'c}, Armin and von K{\"u}gelgen, Julius and Buchholz, Simon and Besserve, Michel and Gresele, Luigi and Sch{\"o}lkopf, Bernhard},
  journal={arXiv preprint arXiv:2305.17225},
  year={2023}
}

@book{spirtes2000causation,
  title={Causation, prediction, and search},
  author={Spirtes, Peter and Glymour, Clark N and Scheines, Richard},
  year={2000}
}

@article{zhang2023identifiability,
  title={Identifiability Guarantees for Causal Disentanglement from Soft Interventions},
  author={Zhang, Jiaqi and Squires, Chandler and Greenewald, Kristjan and Srivastava, Akash and Shanmugam, Karthikeyan and Uhler, Caroline},
  journal={arXiv preprint arXiv:2307.06250},
  year={2023}
}

@inproceedings{zhang2022towards,
  title={Towards principled disentanglement for domain generalization},
  author={Zhang, Hanlin and Zhang, Yi-Fan and Liu, Weiyang and Weller, Adrian and Sch{\"o}lkopf, Bernhard and Xing, Eric P},
  booktitle={Proceedings of the IEEE/CVF Conference on Computer Vision and Pattern Recognition},
  pages={8024--8034},
  year={2022}
}

@article{holyoak2011causal,
  title={Causal learning and inference as a rational process: The new synthesis},
  author={Holyoak, Keith J and Cheng, Patricia W},
  journal={Annual review of psychology},
  volume={62},
  pages={135--163},
  year={2011},
  publisher={Annual Reviews}
}

@incollection{dunbar2004causal,
  title={Causal thinking in science: How scientists and students interpret the unexpected},
  author={Dunbar, Kevin N and Fugelsang, Jonathan A},
  booktitle={Scientific and technological thinking},
  pages={57--79},
  year={2004},
  publisher={Psychology Press}
}

@incollection{shanks1988associative,
  title={Associative accounts of causality judgment},
  author={Shanks, David R and Dickinson, Anthony},
  booktitle={Psychology of learning and motivation},
  volume={21},
  pages={229--261},
  year={1988},
  publisher={Elsevier}
}

@article{locatello2020sober,
  title={A sober look at the unsupervised learning of disentangled representations and their evaluation},
  author={Locatello, Francesco and Bauer, Stefan and Lucic, Mario and R{\"a}tsch, Gunnar and Gelly, Sylvain and Sch{\"o}lkopf, Bernhard and Bachem, Olivier},
  journal={The Journal of Machine Learning Research},
  volume={21},
  number={1},
  pages={8629--8690},
  year={2020},
  publisher={JMLRORG}
}

@article{varici2023score,
  title={Score-based causal representation learning with interventions},
  author={Varici, Burak and Acarturk, Emre and Shanmugam, Karthikeyan and Kumar, Abhishek and Tajer, Ali},
  journal={arXiv preprint arXiv:2301.08230},
  year={2023}
}

@article{varici2023general,
  title={General Identifiability and Achievability for Causal Representation Learning},
  author={Var{\i}c{\i}, Burak and Acart{\"u}rk, Emre and Shanmugam, Karthikeyan and Tajer, Ali},
  journal={arXiv preprint arXiv:2310.15450},
  year={2023}
}

@article{tejada2023causal,
  title={Causal machine learning for single-cell genomics},
  author={Tejada-Lapuerta, Alejandro and Bertin, Paul and Bauer, Stefan and Aliee, Hananeh and Bengio, Yoshua and Theis, Fabian J},
  journal={arXiv preprint arXiv:2310.14935},
  year={2023}
}

@inproceedings{lopez2023learning,
  title={Learning causal representations of single cells via sparse mechanism shift modeling},
  author={Lopez, Romain and Tagasovska, Natasa and Ra, Stephen and Cho, Kyunghyun and Pritchard, Jonathan and Regev, Aviv},
  booktitle={Conference on Causal Learning and Reasoning},
  pages={662--691},
  year={2023},
  organization={PMLR}
}

@article{lippe2023biscuit,
  title={BISCUIT: Causal Representation Learning from Binary Interactions},
  author={Lippe, Phillip and Magliacane, Sara and L{\"o}we, Sindy and Asano, Yuki M and Cohen, Taco and Gavves, Efstratios},
  journal={arXiv preprint arXiv:2306.09643},
  year={2023}
}

@article{bengio2013representation,
  title={Representation learning: A review and new perspectives},
  author={Bengio, Yoshua and Courville, Aaron and Vincent, Pascal},
  journal={IEEE transactions on pattern analysis and machine intelligence},
  volume={35},
  number={8},
  pages={1798--1828},
  year={2013},
  publisher={IEEE}
}

@article{silver2016mastering,
  title={Mastering the game of Go with deep neural networks and tree search},
  author={Silver, David and Huang, Aja and Maddison, Chris J and Guez, Arthur and Sifre, Laurent and Van Den Driessche, George and Schrittwieser, Julian and Antonoglou, Ioannis and Panneershelvam, Veda and Lanctot, Marc and others},
  journal={nature},
  volume={529},
  number={7587},
  pages={484--489},
  year={2016},
  publisher={Nature Publishing Group}
}

@article{bubeck2023sparks,
  title={Sparks of artificial general intelligence: Early experiments with gpt-4},
  author={Bubeck, S{\'e}bastien and Chandrasekaran, Varun and Eldan, Ronen and Gehrke, Johannes and Horvitz, Eric and Kamar, Ece and Lee, Peter and Lee, Yin Tat and Li, Yuanzhi and Lundberg, Scott and others},
  journal={arXiv preprint arXiv:2303.12712},
  year={2023}
}

@article{ovadia2019can,
  title={Can you trust your model's uncertainty? evaluating predictive uncertainty under dataset shift},
  author={Ovadia, Yaniv and Fertig, Emily and Ren, Jie and Nado, Zachary and Sculley, David and Nowozin, Sebastian and Dillon, Joshua and Lakshminarayanan, Balaji and Snoek, Jasper},
  journal={Advances in neural information processing systems},
  volume={32},
  year={2019}
}

@article{akhtar2018threat,
  title={Threat of adversarial attacks on deep learning in computer vision: A survey},
  author={Akhtar, Naveed and Mian, Ajmal},
  journal={Ieee Access},
  volume={6},
  pages={14410--14430},
  year={2018},
  publisher={IEEE}
}

@article{wang2023adversarial,
  title={Adversarial Policies Beat Superhuman Go AIs},
  author={Wang, Tony Tong and Gleave, Adam and Tseng, Tom and Pelrine, Kellin and Belrose, Nora and Miller, Joseph and Dennis, Michael D and Duan, Yawen and Pogrebniak, Viktor and Levine, Sergey and others},
  year={2023}
}

@inproceedings{koh2021wilds,
  title={Wilds: A benchmark of in-the-wild distribution shifts},
  author={Koh, Pang Wei and Sagawa, Shiori and Marklund, Henrik and Xie, Sang Michael and Zhang, Marvin and Balsubramani, Akshay and Hu, Weihua and Yasunaga, Michihiro and Phillips, Richard Lanas and Gao, Irena and others},
  booktitle={International Conference on Machine Learning},
  pages={5637--5664},
  year={2021},
  organization={PMLR}
}

@article{scholkopf2021toward,
  title={Toward causal representation learning},
  author={Sch{\"o}lkopf, Bernhard and Locatello, Francesco and Bauer, Stefan and Ke, Nan Rosemary and Kalchbrenner, Nal and Goyal, Anirudh and Bengio, Yoshua},
  journal={Proceedings of the IEEE},
  volume={109},
  number={5},
  pages={612--634},
  year={2021},
  publisher={IEEE}
}

@article{geirhos2020shortcut,
  title={Shortcut learning in deep neural networks},
  author={Geirhos, Robert and Jacobsen, J{\"o}rn-Henrik and Michaelis, Claudio and Zemel, Richard and Brendel, Wieland and Bethge, Matthias and Wichmann, Felix A},
  journal={Nature Machine Intelligence},
  volume={2},
  number={11},
  pages={665--673},
  year={2020},
  publisher={Nature Publishing Group UK London}
}

@article{comon1994independent,
  title={Independent component analysis, a new concept?},
  author={Comon, Pierre},
  journal={Signal processing},
  volume={36},
  number={3},
  pages={287--314},
  year={1994},
  publisher={Elsevier}
}

@article{shimizu2006linear,
  title={A linear non-Gaussian acyclic model for causal discovery.},
  author={Shimizu, Shohei and Hoyer, Patrik O and Hyv{\"a}rinen, Aapo and Kerminen, Antti and Jordan, Michael},
  journal={Journal of Machine Learning Research},
  volume={7},
  number={10},
  year={2006}
}

@article{rudelson2009smallest,
  title={Smallest singular value of a random rectangular matrix},
  author={Rudelson, Mark and Vershynin, Roman},
  journal={Communications on Pure and Applied Mathematics: A Journal Issued by the Courant Institute of Mathematical Sciences},
  volume={62},
  number={12},
  pages={1707--1739},
  year={2009},
  publisher={Wiley Online Library}
}

@article{schwartz1954formula,
  title={The formula for change in variables in a multiple integral},
  author={Schwartz, J},
  journal={The American Mathematical Monthly},
  volume={61},
  number={2},
  pages={81--85},
  year={1954},
  publisher={Taylor \& Francis}
}

@article{eronen2020causal,
  title={Causal discovery and the problem of psychological interventions},
  author={Eronen, Markus I},
  journal={New Ideas in Psychology},
  volume={59},
  pages={100785},
  year={2020},
  publisher={Elsevier}
}

@article{eberhardt2014direct,
  title={Direct causes and the trouble with soft interventions},
  author={Eberhardt, Frederick},
  journal={Erkenntnis},
  volume={79},
  pages={755--777},
  year={2014},
  publisher={Springer}
}

@article{campbell2007interventionist,
  title={An interventionist approach to causation in psychology},
  author={Campbell, John},
  journal={Causal learning: Psychology, philosophy, and computation},
  pages={58--66},
  year={2007},
  publisher={Oxford University Press Oxford}
}

@article{buchholz2023learning,
  title={Learning Linear Causal Representations from Interventions under General Nonlinear Mixing},
  author={Buchholz, Simon and Rajendran, Goutham and Rosenfeld, Elan and Aragam, Bryon and Sch{\"o}lkopf, Bernhard and Ravikumar, Pradeep},
  journal={arXiv preprint arXiv:2306.02235},
  year={2023}
}

@article{tillman2014learning,
  title={Learning causal structure from multiple datasets with similar variable sets},
  author={Tillman, Robert E and Eberhardt, Frederick},
  journal={Behaviormetrika},
  volume={41},
  number={1},
  pages={41--64},
  year={2014},
  publisher={The Behaviormetric Society}
}

@article{hyttinen2013experiment,
  title={Experiment selection for causal discovery},
  author={Hyttinen, Antti and Eberhardt, Frederick and Hoyer, Patrik O},
  journal={Journal of Machine Learning Research},
  volume={14},
  pages={3041--3071},
  year={2013},
  publisher={Microtome Publishing}
}

@book{woodward2005making,
  title={Making things happen: A theory of causal explanation},
  author={Woodward, James},
  year={2005},
  publisher={Oxford university press}
}

@article{fisher1960design,
  title={The design of experiments.},
  author={Fisher, Ronald Aylmer and others},
  journal={The design of experiments.},
  number={7th Ed},
  year={1960},
  publisher={Oliver and Boyd. London and Edinburgh}
}

@misc{strevens2007review,
  title={Review of Woodward," Making Things Happen"},
  author={Strevens, Michael},
  year={2007},
  publisher={JSTOR}
}

@article{ahuja2023multi,
  title={Multi-Domain Causal Representation Learning via Weak Distributional Invariances},
  author={Ahuja, Kartik and Mansouri, Amin and Wang, Yixin},
  journal={arXiv preprint arXiv:2310.02854},
  year={2023}
}

@inproceedings{roeder2021linear,
  title={On linear identifiability of learned representations},
  author={Roeder, Geoffrey and Metz, Luke and Kingma, Durk},
  booktitle={International Conference on Machine Learning},
  pages={9030--9039},
  year={2021},
  organization={PMLR}
}

@article{prystawski2023think,
  title={Why think step-by-step? Reasoning emerges from the locality of experience},
  author={Prystawski, Ben and Goodman, Noah D},
  journal={arXiv preprint arXiv:2304.03843},
  year={2023}
}

@article{kiciman2023causal,
  title={Causal reasoning and large language models: Opening a new frontier for causality},
  author={K{\i}c{\i}man, Emre and Ness, Robert and Sharma, Amit and Tan, Chenhao},
  journal={arXiv preprint arXiv:2305.00050},
  year={2023}
}

@book{pearl2009causality,
  title={Causality},
  author={Pearl, Judea},
  year={2009},
  publisher={Cambridge university press}
}

@inproceedings{cooper1999causal,
  title={Causal discovery from a mixture of experimental and observational data},
  author={Cooper, Gregory F and Yoo, Changwon},
  booktitle={Proceedings of the Fifteenth conference on Uncertainty in artificial intelligence},
  pages={116--125},
  year={1999}
}

@inproceedings{tong2001active,
  title={Active learning for structure in Bayesian networks},
  author={Tong, Simon and Koller, Daphne},
  booktitle={International joint conference on artificial intelligence},
  volume={17},
  number={1},
  pages={863--869},
  year={2001},
  organization={Citeseer}
}

@article{hauser2014two,
  title={Two optimal strategies for active learning of causal models from interventional data},
  author={Hauser, Alain and B{\"u}hlmann, Peter},
  journal={International Journal of Approximate Reasoning},
  volume={55},
  number={4},
  pages={926--939},
  year={2014},
  publisher={Elsevier}
}

@inproceedings{eberhardt2008almost,
  title={Almost optimal intervention sets for causal discovery},
  author={Eberhardt, Frederick},
  booktitle={Proceedings of the Twenty-Fourth Conference on Uncertainty in Artificial Intelligence},
  pages={161--168},
  year={2008}
}

\begin{appendices}

\startcontents[sections]
\printcontents[sections]{l}{1}{\setcounter{tocdepth}{2}}

\section{Related works}
\label{appsec:related-work}
\textbf{The interventionist approach to causation} For the problem of causal graph discovery, it is well-known that the underlying causal structure is non-identifiable given only \enquote{passively observed} (equivalently, \emph{i.i.d.}) data alone. As a result, randomized controlled experiments \citep{fisher1960design} is often used to infer causality. These experiments typically take the form of interventions \citep{spirtes2000causation,pearl2009causality}, \emph{i.e.}, manipulations on the \enquote{natural state} of the system of interest. Early works \citep{woodward2005making,strevens2007review} define the \enquote{hard}  (also called \enquote{surgical} or \enquote{arrow-breaking}) interventions in which the value of the intervened variable is entirely determined by the experimenter, thereby removing the dependence of this variable on its direct causes. This type of intervention is arguably the most natural one to consider, and following this definition, a line of works explore sufficient conditions for designing experiments that guarantee identifiability of the causal model in various settings \citep{cooper1999causal,tong2001active,eberhardt2008almost,hyttinen2013experiment,hauser2014two}.

\textbf{Intervention \emph{v.s.} passive observation} While extensive works demonstrate the success of the interventionist approach, it faces several key challenges that significantly limit its applicability. First, \citet{eberhardt2014direct} finds that in the presence of unobserved variables, certain causal structures are indistinguishable if we only perform hard interventions. This issue can be resolved by performing soft interventions \emph{i.e.}, interventions that do not remove the dependency on direct causes but only changes the conditional distribution. Second, as pointed out in \citep{tillman2014learning}, interventions --- whether hard or soft --- are often expensive or even infeasible to perform in practice. For example, a psychological intervention is likely to affect multiple psychological variables simultaneously \citet{eronen2020causal}. As a result, \citep{tillman2014learning} returns to the \enquote{passive observation} setting but with multiple datasets with overlapping latent variables.

\textbf{Interventional causal representation learning} Motivated by the interventionist literature in causal graph discovery, a recent line of works \citep{ahuja2023interventional,seigal2022linear,varici2023score,von2023nonparametric,buchholz2023learning,zhang2023identifiability,varici2023general} consider performing interventions to resolve the non-identifiability issue in causal representation learning \citep{locatello2019challenging}. Roughly speaking, these result indicate that identification (possibly with some ambiguities) is possible if one can perform intervention on every latent variable. However, it is unclear how to perform such interventions in practice, given that the underlying latent variables are unknown. \citet{khemakhem2020variational,lu2021invariant,roeder2021linear} do not require single-node interventions to achieve identifiability, but assumes that the joint distribution of latent variables in each environment lie in a certain exponential family. This assumption can be understood as a prior on the latent variables, but it is unclear when or why it is reasonable to make in reality. Recently, \citet{ahuja2023multi} considers learning causal representations from multiple domains that relate to each other via an invariance constraint on the subset $\gS$ of \emph{stable} latent variables, and they prove identification up to affine mixtures within $\gS$.

\textbf{Causal reasoning capacity of LLMs} In view of the tremendous success of large language models (LLMs), several works aim to understand the causal reasoning ability of LLMs. \citet{kiciman2023causal} conducts an an extensive experimental study and finds that LLMs outperforms all existing causal discovery methods on multiple datasets, but also have simple and mysterious failure modes. \citet{prystawski2023think} provide theoretical evidence that the chain-of-thought (CoT) prompt allows LLMs to reduce their uncertainty when answering questions related to causal variables that are far apart.

\section{Experiment details for \Cref{sec:experiment-results}}
\label{appsec:experiment-details}

\subsection{Details for step 1 in \Cref{sec:algorithm}}
\label{appsubsec:align-matrices_Mk}

Since $\epsilon = \mB_k\vz = \mB_k\mH\vx$ in the $k$-th environment, so we can use any identification algorithm for linear ICA to recover the matrix $\mM_k$. Note that while standard linear ICA algorithms only apply to the case where $n=d$, for $n>d$ we can arbitrarily choose $d$ principal components of $\vx$ to reduce it to the $n=d$ case. This is without loss of generality, since when $n>d$ there is redundant information in $\vx$.

After recovering $\mM_k$ for each $k$ by running linear ICA, we still do not know whether each $\mM_k\vx$ corresponds to the same permutation of the ground-truth noise variables $\epsilon$. To resolve this issue, we choose test function $\Psi$ mapping any distribution on $\R$ to a deterministic real value, which we expect to take different values for different $\epsilon_i$'s. We choose $\Psi(\mathbb{P})=\mathbb{P}\left[\absx{X}\leq 1\right]$ in our experiments. For all $k\geq 2$, we calculate the $\Psi$ value of each component of the $d$-dimensional empirical distribution $\hat{\mathbb{P}}_k=\frac{1}{N}\sum_{i=1}^N\mathbbm{1}_{\mM_k\vx_{i}^{(k)}}$, and choose a permutation $\pi_k$ to rearrange them in increasing order. Then, we rearrange the columns of $\mM_k$ using the same permutation $\pi_k$. This procedure would asymptotically produce correct alignments as long as $\Psi(\epsilon_i),i\in[d]$ are different, and we find that it empirically works well. 

Alternatively, this alignment step can be done as follows: for each pair of environments $(E_1, E_t)$, and for each pair of nodes $(i, j)$, we calculate the distribution distance between $\epsilon_i$ in environment $E_1$ and $\epsilon_j$ in environment $E_t$, based on some notion of distribution distance (\emph{e.g.} kernel maximum mean discrepancy). Then we find the min-cost perfect matching, where the cost of an edge is the distribution distance.

\subsection{Details for the implementation of \texttt{LiNGCReL} in the finite-sample regime}
\label{appsubsec:finite-sample-implementation}

Although \texttt{LiNGCReL} provably works in the population regime, it faces several challenges when there is only a finite number of samples:
\begin{itemize}
    \item First, since rank is not a continuous function, it is sensitive to finite-sample estimation errors. In our implementation of \Cref{alg:learn-graph}, in each iteration we instead choose $i\notin S$ that has the largest ratio between the first and second singular values of $\left[\vq_1,\vq_2,\cdots,\vq_K\right]$. And in line 6 of \Cref{alg:identify-parents}, we introduce a hyper-parameter $\mathtt{tl}$ such that the matrix $\left[\vq_1,\vq_2,\cdots,\vq_K\right]$ is considered to have rank $r_{m'-1}$ if its $r_{m'}$-th singular value is smaller than $\mathtt{tl}$. Since the smallest singular value of a random matrix $\mA\in\R^{K\times m} (K\geq m)$ is at the order of $\sqrt{K}-\sqrt{m-1}$ with high probability \citep{rudelson2009smallest}, when $K=d$ one shall choose $\mathrm{tl}\sim \sqrt{d}-\sqrt{d-1} = \O\left(\frac{1}{\sqrt{d}}\right)$. On the other hand, for larger $K$ we can correspondingly choose a larger \texttt{tl}. Note that a small \texttt{tl} potentially has the risk of being dominated the noise in the estimation, which means that we need more samples per environment to reduce the noise. In contrast, for larger \texttt{tl} the estimation is more robust to noise and we can use fewer samples.
    \item Second, finite-sample estimation errors of $\mM_k$ make it harder to obtain $\vh_i$ in \Cref{line:intersect-subspaces} of \Cref{alg:learn-graph}. We implement this step in the following way: first let $\mQ_j$ be the orthogonal projection matrix onto $E_j^{\perp}$ \emph{i.e.}, $\mQ_j^\top\vx=\mathrm{proj}_{E_j^\perp}(\vx)$, then choose $\vh_i$ to be the singular vector of $\sum_{j:(j,i)\in\mathcal{E} \text{ or } j=i}\mQ_j^\top\mQ_j$ that corresponds to the smallest singular value (including zero). Indeed, in the noiseless case we would have $\left(\sum_{j:(j,i)\in\mathcal{E} \text{ or } j=i}\mQ_j^\top\mQ_j\right) \vh_i = 0$ if and only if $\vh_i\in\left(\cap_{j:(i,j)\in\mathcal{E}} E_j\right)\cap E_i$. 
\end{itemize}

\section{Further experiment results}
\label{appsec:experiment-results}
\textbf{SNA error v.s. true error} We plot the SNA error v.s. true error achieved by \texttt{LiNGCReL} in \Cref{fig:SNA_vs_true_error}. We observe that
\begin{itemize}
    \item For most nodes, SNA error is exactly equal to the true error and both errors are small, indicating that the corresponding latent variables have been successfully learned by \texttt{LiNGCReL}.
    \item The remaining nodes typically have true error much larger than SNA error. This indicates that there exists some ambiguities at these nodes in the sense that $\dom_{\G}(i)\neq\emptyset$. Note that the true error for many nodes are close to $1$; one possible reason is that one selects the wrong singular vector in the second part of \Cref{appsubsec:finite-sample-implementation}, so that it is orthogonal to the ground-truth vector.
\end{itemize}

\begin{figure}
    \centering
    \includegraphics[width=0.6\textwidth]{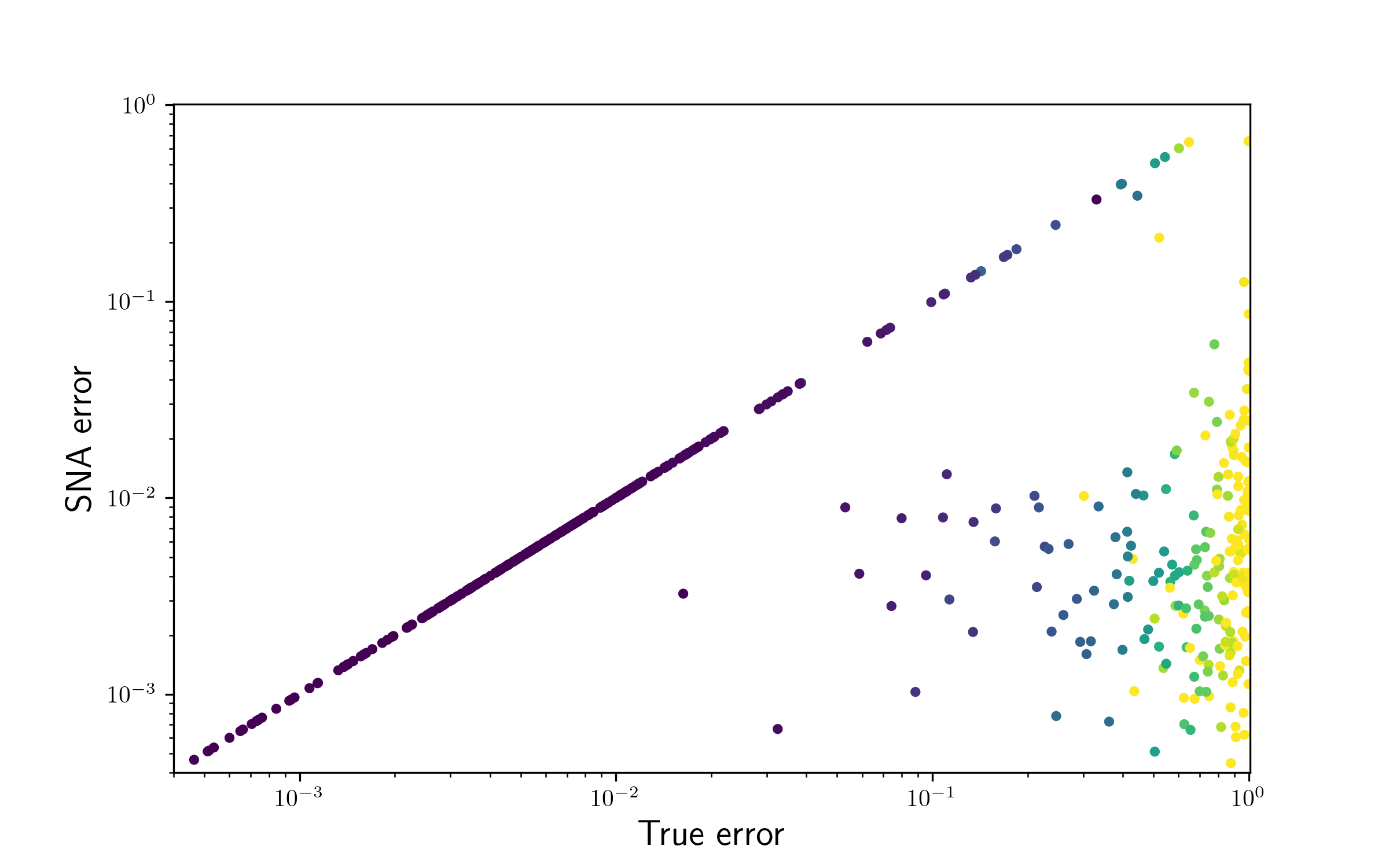}
    \caption{Comparing SNA error with true error for the $500$ latent variables in the $100$ graphs of size $d=5$ that we sample in \Cref{sec:experiment-results}.}
    \label{fig:SNA_vs_true_error}
\end{figure}

\textbf{Sensitivity of \texttt{LiNGCReL} to the hyperparameter \texttt{tl}} We examine how different choices of \texttt{tl} would affect the performance of \texttt{LiNGCReL}. Specifically, we run \texttt{LiNGCReL} on the $100$ models with size $d=5$ and number of environments $K=5$ sampled in \Cref{sec:experiment-results} with $\texttt{tl}\in\{0.1,0.15,0.2,0.25,0.3\}$ and the results are reported in \Cref{fig:diff-tl}. We can see that the permance is actually quite sensitive to $\texttt{tl}$.

\begin{figure}
    \centering
    \includegraphics[width=0.6\textwidth]{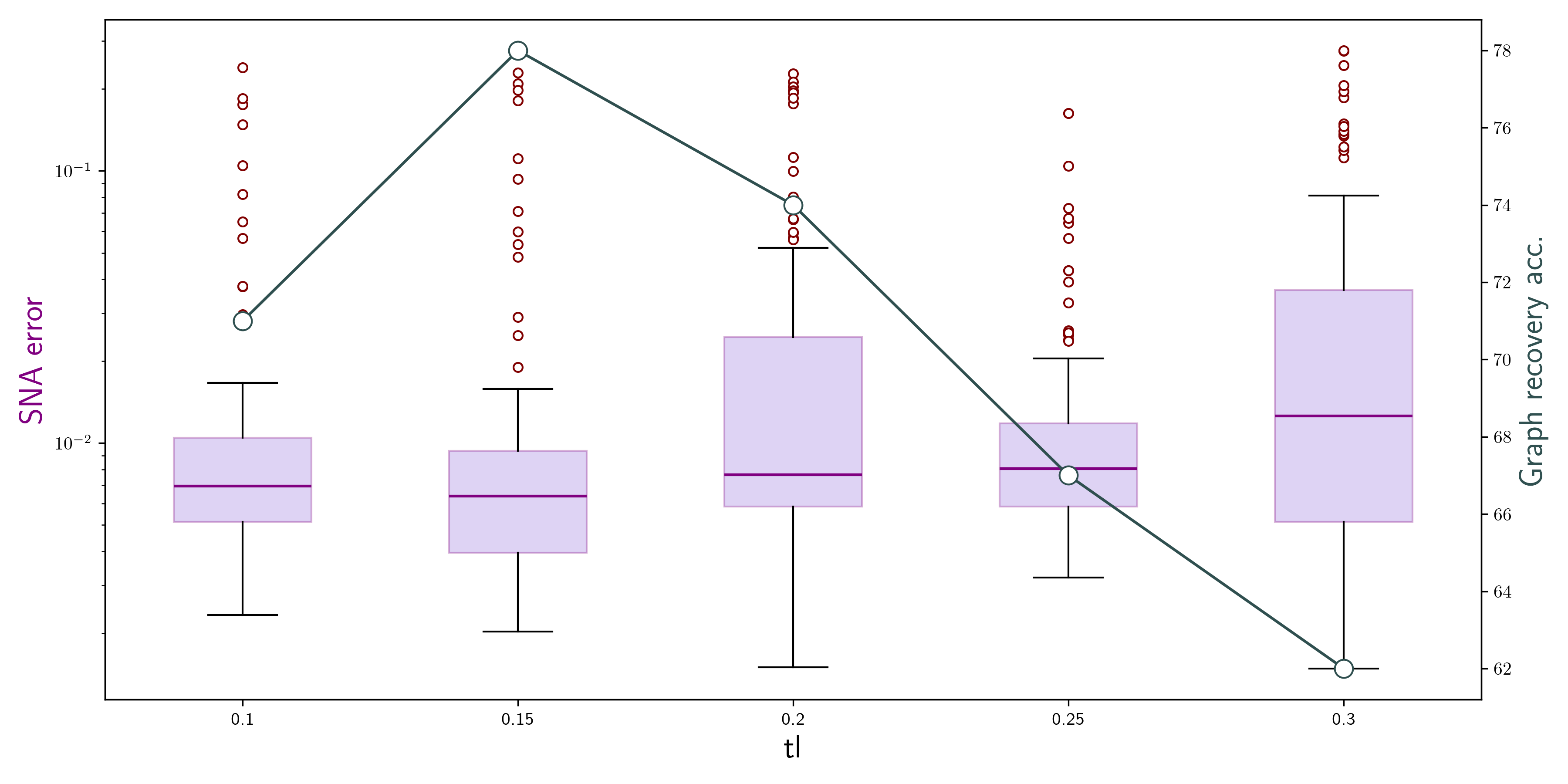}
    \caption{Performance of \texttt{LiNGCReL} as a function of \texttt{tl}. \texttt{tl}$=0.15$ achieves the best performance in terms of both SNA error and graph recovery accuracy.}
    \label{fig:diff-tl}
\end{figure}

\FloatBarrier

\section{Background on causal representation learning}
\label{appsec:background}

It is common to assume some axioms on what kind of (conditional) dependency information is encoded in a causal graph (see~\citealp[Section 3.4]{spirtes2000causation} for a detailed discussion). The most natural one is the Causal Markov Condition introduced in \Cref{asmp:causal-markov-condition} that gives sufficient conditions for conditional independence via $d$-separation. We introduce the formal definition of $d$-separation below:

\begin{definition}[paths and colliders]
    Let $i,j$ be two nodes of a DAG $\G$, a path is a sequence of nodes $i_0=i,i_1,\cdots,i_k=j$ such that there is an edge (in either direction) between $i_j$ and $i_{j+1}, j=0,1,\cdots,k-1$. A node $i_j$ is called a collider on this path if $i_j\in\ch_{\G}(i_{j-1})\cap\ch_{\G}(i_{j+1})$.
\end{definition}

\begin{definition}[blocked path]
    A path in a DAG $\G$ between node $i$ and node $j$ is said to be blocked by a node set $S$ if either of the following holds:
    \begin{itemize}
        \item there exists a node $v$ on the path that is in $S$ but not a collider, or
        \item there exists a node $v$ on the path that is a collider, but none of its descendants (including itself) are in $S$.
    \end{itemize}
\end{definition}

\begin{definition}[d-separation]
    \label{def:d-separation}
    For a DAG $\G$ with node set $[d]$, any two nodes $i\neq j$ are said to be $d$-separated by a set $S\subset [d]\setminus\{i,j\}$ if all paths from $i$ to $j$ are blocked by $S$.
\end{definition}

The \textit{minimality} condition states that there is no redundant edges in the causal graph, and is a natural consequence of the Occam's Razor Principle.

\begin{assumption}[Causal minimality, {~\citealp[Section 3.4.2]{spirtes2000causation}}]
    \label{asmp:minimality} For latent variables $\vz$, removing any edge from $\G$ would render violation of the causal Markov condition \Cref{asmp:causal-markov-condition}. In other words, let $\G_1$ be the graph obtained by removing any single edge from $\G$, then there must exist $i\in[d]$ such that $\vz_i \not\independent \vz_{\nd_{\G_1}(i)}\mid \vz_{\pa_{\G_1}(i)}$.
\end{assumption}

The \textit{faithfulness} condition states that the Causal Markov Condition actually entails all (conditional) independence in the latent variables.

\begin{assumption}[Faithfulness, {~\citealp[Section 3.4.3]{spirtes2000causation}}]
\label{asmp:faithfulness}
    Every (conditional) independence in the latent variables $\vz$ is entailed by the Causal Markov Condition applied to $\G$. In other words, $\vz_i\independent\vz_j\mid \vz_S \Leftrightarrow i,j$ are $d$-separated by $S$.
\end{assumption}

Existing works have explored different notions of identifiability. For observational data, it is well known that Markov equivalence of graphs is an intrinsic ambiguity that one cannot resolve:

\begin{definition}[Markov equivalence/Faithful Indistinguishability, {\citealp[Section 4.2]{spirtes2000causation}}]
    If two DAGs encodes the same set of dependency relations, we say that they are \textit{Markov equivalent}.
\end{definition}

Any DAG $\mathcal{G}$ induces a partial order on its nodes which we denote by $\prec_{\mathcal{G}}$. In the special case when for all $i\neq j$, either $i\prec_{\G}j$ or $j\prec_{\G}i$ holds, we say that $\prec_{\G}$ is a total order. This partial order is equivalent to the transitional closure of the graph, as defined below:

\begin{definition}[Transitional closure]
    Given any DAG $\G$, its transitional closure $\bar{\G}$ is defined to be the graph obtained by connecting all edges $i\to j$ where $i$ is an ancestor of $j$ in $\G$.
\end{definition}

When $\prec_{\G}$ is a total order, each pair of nodes are connected by a directed edge in its transitive closure $\bar{\G}$. Such $\bar{\G}$ is often called a \textit{tournament} in graph theory.

In the following, we list different forms of identifiability that appear in the literature:

\begin{definition}[different notions of identifiability]
\label{def:crl-identify}
    Let $\H:\R^n \supseteq \mathcal{X}\mapsto\R^d$ be the space of diffeomorphic mappings from observation to latent, and $\mathfrak{G}$ be the space of all DAGs with $d$ nodes, then for $h,\hat{h}\in\H$ and $\G,\hat{\G}\in\mathfrak{G}$, we write
    \begin{enumerate}[label={(\roman*)}]
        \item \citep{seigal2022linear,liang2023causal} $(h,\G)\overset{T}{\sim}_{G} (\hat{h},\hat{\G})$ if there exists a permutation $\pi$ on $[d]$ such that $\pi(\G)$ and $\hat{\G}$ have the same transitional closure;
        \item \citep{von2023nonparametric,varici2023general} $(h,\G)\CRLsim (\hat{h},\hat{\G})$ if we actually have $\G=\hat{\G}$ for the $\phi$ defined above.
    \end{enumerate}
\end{definition}

Given an equivalence relation $\sim$ on $\H\times\mathfrak{G}$, we say that a causal model $(\vh,\G)$ is identifiable under $\sim$ if any  candidate solution $(\hat{\vh},\hat{\G})$ satisfies $(\hat{\vh},\hat{\G})\sim(\vh,\G)$. The notion of identification up to $\overset{T}{\sim}_{G}$, as shown in \citet{seigal2022linear} with single-node soft interventions on linear causal models, is highly related to this paper. Compared with their result, our $\sim_{\dom}$ guarantee is must stronger, since not only the causal graph can be fully recovered, but the latent variables can be identified up to mixtures of the effect-dominating sets as well.

\section{Illustrating examples for our theory and algorithm}
\label{appsec:eda-example}

\subsection{An example for understanding the SNA ambiguity}
We provide a simple example below to illustrate the SNA ambiguity discussed in \Cref{sec:eda}.

\begin{example}
    \label{example:discount-revisit}
    Let $G$ be a causal graph with $d=3$ nodes and edges $1\to 2$ and $2\to 3$. We have access to observations from a set of environments $\mathfrak{E}$. It turns out that there is no way to distinguish between the following two structural equation models:
    \begin{align*}
        \label{eq:example-scm}
        \vz_1 &= \epsilon_1^E &\qquad \vv_1 &= \epsilon_1^E \\
        \vz_2 &= f_2^E(\vz_1, \epsilon_2^E) &\qquad \vv_2 &= f_2^E(\vv_1, \epsilon_2^E)\\
        \vz_3 &= f_3^E(\vz_2, \epsilon_3^E) &\qquad \vv_3 &= \vv_2 + f_3^E(\vv_2, \epsilon_3^E)\\
        \vx &= \vz = (\vz_1,\vz_2,\vz_3)^\top &\qquad \vx &= (\vv_1,\vv_2,\vv_3-\vv_2)^\top
    \end{align*}
    where $\epsilon_i^E, i=1,2,3$ are independent noise variables, if we do not change the causal graph $\G$, no matter what environment $E$ that we have. 
    
    This issue does not exist when we assume access to hard interventions on node $3$, which effectively removes the edge $2\to 3$. Specifically, with hard intervention on $\vz_3$, the variables $\vz_2$ and $\vz_3$ become independent. But by definition, $\vv_2 = \vz_2$ and $\vv_3 = \vz_2+\vz_3$ must be dependent, so this intervention cannot be realized by any hard intervention on $\vv_3$, thereby providing a way to distinguish between the above models.

    Without node $3$, the same ambiguity would arise on node $2$. However, node $3$ can help us to overcome this ambiguity, thanks to the fact that node $2$ is the only causal parent of node $3$. Suppose for example that $\vv_2 = m(\vz_1,\vz_2)$ is some mixture of $\vz_1$ and $\vz_2$, then $\vv_3 = \hat{f}_3^E\left(\vv_2,\epsilon_3^E\right) = \hat{f}_3^E\left(m(\vz_1,\vz_2),\epsilon_3^E\right)$. Since all environments share the same mixing function, $\vv_3$ must be some deterministic function $\psi_3(\vz)$ of $\vz$, where $\psi_3$ is the same across all environment $E$. Hence, we have
    \begin{equation}
        \label{eq:mixture-compare}
        \hat{f}_3^E\left(m(\vz_1,\vz_2),\epsilon_3^E\right) = \psi_3\left(\vz_1,\vz_2,f_3^E(\vz_2,\epsilon_3^E)\right)
    \end{equation}
    Now we note that the dependencies of LHS on $\vz_1$ and $\vz_2$ are through a single scalar-valued function $m$, but since we would have different $f_3^E$'s in different environments, this in general does not hold for the RHS. Therefore, any causal model with latent variable $\vv_2$ as a mixture of $\vz_1$ and $\vz_2$ cannot be equivalent to the ground-truth model.
\end{example}

According to \Cref{def:effect-domination}, in \Cref{example:discount-revisit} we have $\dom_{\G}(1)=\dom_{\G}(2)=\emptyset$ but $\dom_{\G}(3)=\{2\}$.

\subsection{An example for the main idea behind LiNGCReL}
\label{appsec:lingcrel-example}

To illustrate our main algorithm on how we can recover the graph $\G$ and the matrix $\mH$, we first provide some intuition using a simple three-node example:

\begin{example}
    Let $\G$ be the graph with $d=3$ nodes and edges $1\to 2, 1\to 3$ and $2\to 3$, so that each $\mB_k$ is of form
    \begin{equation}
        \label{eq:Bk-form}
        \mB_k = \left(
        \begin{aligned}
            \times &\quad 0 &\quad 0 \\
            \times &\quad \times &\quad 0 \\
            \times &\quad \times &\quad \times
        \end{aligned}
        \right) 
        \begin{aligned}
            &\rightsquigarrow\textcolor{purple}{\vb_{k1}} \\
            &\rightsquigarrow\textcolor{purple}{\vb_{k2}} \\
            &\rightsquigarrow\textcolor{purple}{\vb_{k3}}
        \end{aligned}
    \end{equation}
    We can identify the graph as follows: first, for $i\in\{1,2,3\}$, look at the space $\mW_i$ spanned by the rows $(\mM_k)_i, k\in[K]$. If $\dim\mW_i = 1$, we know that $i$ is a source node (\emph{i.e.}, $\pa_{\G}(i)=\emptyset$) in $\G$. Otherwise it is not, due to \Cref{asmp:independent-row}. Hence we can know that node $1$ is a source node.

    In our example, there is no other node that satisfies this requirement. We then proceed to search for some $i\neq 1$ such that the projection of $\mW_i$ onto $\mW_1^\perp$ has dimension $1$. If this holds, then one can show that $\pa_{\G}(i)=\{1\}$. Otherwise, $i$ must have parents other than $1$.

    It turns this requirement is satisfied for node $2$ since $\dim\left(\mathrm{proj}_{\vh_1}\spanl{\vh_1,\vh_2}\right) = 1$, but is not satisfied for node $3$ since $\dim\left(\mathrm{proj}_{\vh_1}\spanl{\vh_1,\vh_2, \vh_3}\right) \geq 2$ (by \Cref{lemma:linalg-proj-dim}). Hence we know that $\pa_{\G}(2)=\{1\}$.

    Finally, it remains to determine $\pa_{\G}(3)$. To do this, we first note that $\dim\mW_3=3$. Then we project $\mW_3$ onto $\mW_1^\perp$ and $\mW_2^\perp$ respectively, and the resulting dimensions are $2$ and $1$. As we rigorously show in \Cref{prop:alg1}, a decrease of the dimension exactly indicates finding a new parent. Thus we have $\pa_{\G}(3)=\{1,2\}$, completing the recovery of the graph.

    Finally, we recover the unmixing matrix $\mH$ (and thus the latent variables) by noticing that $\vh_1\in\mW_1$, $\vh_2\in\mW_2\cap\mW_3$ and $\vh_3\in\mW_3$. Ambiguities would arise at nodes $2$ and $3$, which are exactly the nodes that have non-empty effect-dominating sets.
\end{example}

\section{Auxiliary lemmas}
\label{appsec:auxiliary-lemma}

\begin{lemma}
    \label{lemma:linalg-invertible-transformation}
    For any family of $m$-dimensional vectors $\{\vv_k\}_{k=1}^K$ and $\{\vz_k\}_{k=1}^K$ if $\vv_k = \vz_k \mT$ and $\mT\in\R^{m\times m}$ is invertible, then 
    \begin{equation}
        \notag
        \dim\spanl{\vv_k: k\in[K]} = \dim\spanl{\vz_k: k\in[K]}
    \end{equation}
\end{lemma}

\begin{theorem}[Darmois-Skitovic Theorem]
    Let $\epsilon_i,i\in[d]$ be independent random variables and $X=\sum_{i=1}^d \alpha_i\epsilon_i, Y=\sum_{i=1}^d \beta_i\epsilon_i$. If $X \independent Y$, then for $\forall i\in[d]$, $\alpha_i\beta_i\neq 0 \Rightarrow \epsilon_i$ is Gaussian distributed.
\end{theorem}

\begin{lemma}
\label{lemma:asmp-no-gaussian}
    Suppose that $\epsilon=(\epsilon_1,\cdots,\epsilon_d)$ is a $d$-dimensional random vector with independent components such that $\var(\epsilon_i)=1, \forall i\in[d]$, and there exists an invertible and non-diagonal matrix $\mM$ such that $\mM\epsilon\overset{d}{=}\epsilon$, then at least one of the following statements must hold:
    \begin{enumerate}[label={(\arabic*)}]
        \item there exists at least two Gaussian variables in $\epsilon_1,\cdots,\epsilon_d$;
        \item $\mM$ is a permutation matrix and there exists $1\leq i<j\leq d$ such that $\epsilon_i\overset{d}{=}\epsilon_j$.
    \end{enumerate}
\end{lemma}

\begin{proof}
    Suppose that (1) does not hold, then there is at most one Gaussian variable in $\epsilon_1,\cdots,\epsilon_d$. We assume WLOG that $\epsilon_1,\cdots,\epsilon_{d-1}$ are all non-Gaussian. Then by the Darmois-Skitovic Theorem, we know that for $\forall 1\leq j<k\leq[d]$ and $i\in[d-1]$, $\mM_{ji}\cdot\mM_{ki}= 0$ $\Rightarrow$ there is at most one non-zero entry in each of the first $d-1$ columns of $\mM$. 

    Assume that $\mM_{k_i,i}\neq 0$, $i\in[d-1]$. Since $\mM$ is invertible, we know that $k_i, i\in[d-1]$ must be different. Let $k_d$ be the remaining element in $[d]$ that does not appear in $k_i,i<d$, then $(\mM\epsilon)_{k_d} = \mM_{k_d,d}\epsilon_d$, while $(\mM\epsilon)_{k_i}=\mM_{k_i,i}\epsilon_i+\mM_{k_i,d}\epsilon_d$. Since the components of $\mM\epsilon$ are independent, it is easy to see that $\mM_{id}\neq 0, \forall i\neq k_d$. In other words, $\mM$ only has non-zero entries at $(k_i,i),i\in[d]$.
    
    Since $\var(\epsilon_i)=1$, we know that $\mM$ must be a signed permutation matrix. Finally, let $\pi$ be the permutation on $[d]$ such that $\mM_{i,\pi(i)}\neq 0$. Since $\mM$ is not diagonal, $\pi$ must have a cycle $(i_1,i_2,\cdots,i_k)$ with length $k\geq 2$, so that $\epsilon_{i_1},\cdots,\epsilon_{i_k}$ all have the same distribution, which implies that (2) holds, as desired.
\end{proof}

\begin{lemma}
\label{lemma:linalg-proj-dim}
    Let $\mV_1,\mV_2$ be two subspaces of $\R^d$ such that $\mV_1\cap\mV_2=\{\bm{0}\}$, and $\mP_{\mV_1^{\perp}}$ be the orthogonal projection onto $\mV_1^{\perp}$, then we have that $\dim(\mV_2)=\dim\left(\mP_{\mV_1^{\perp}}\mV_2\right)$.
\end{lemma}

\begin{proof}
    Obviously we have $\dim(\mV_2)\geq\dim\left(\mP_{\mV_1^{\perp}}\mV_2\right)$. On the other hand, let $\vu_1,\vu_2,\cdots,\vu_m$ be a basis of $\mV_2$, then $\vw_i=\mP_{\mV_1^{\perp}}\vu_i, i=1,2,\cdots,m$ are also independent. Indeed, suppose that $\lambda_i,i=1,2,\cdots,m$ satisfy $\sum_{i=1}^m\lambda_i\vw_i=0$, then $\mP_{\mV_1^{\perp}}\left(\sum_{i=1}^m\lambda_i\vu_i\right)=0$, implying that $\sum_{i=1}^m\lambda_i\vu_i \in\mV_1$. However, we know that $\mV_1\cap\mV_2=\{\bm{0}\}$, so $\lambda_1=\cdots=\lambda_m=0$. This concludes the proof.
\end{proof}

\begin{lemma}
    \label{lemma:assumptions-equiv}
    \Cref{asmp:independent-row-orig} is equivalent to \Cref{asmp:independent-row}.
\end{lemma}

\begin{proof}
    The main observation is that for each $k\in[K]$, $(\mB_k)_i$ only has non-zero entries at the $j$-th coordinate where $j\in\hpa_{\G}(i)$. Moreover, let $\tilde{\vw}_k(i)$ be the vector consisting of these entries, then $\tilde{\vw}_k(i) = (\bm{\Omega}_k)_{ii}^{-\frac{1}{2}}\left(-\vw_k(i),1\right)$. Hence,
    \begin{equation}
        \notag
        \dim\spanl{(\mB_k)_i: k\in[K]} = \dim\spanl{\left(-\vw_k(i),1\right): k\in[K]}.
    \end{equation}
    Suppose that \Cref{asmp:independent-row-orig} holds, then for $\forall \vx \in \R^{\absx{\pa_{\G}(i)}}$, there exists $\lambda_k \in \R, 1\leq k\leq \absx{\pa_{\G}(i)}$ such that $\sum_k\lambda_k=1$ and $\sum_k\lambda_k\vw_k(i)= \vx$. Hence,
    \begin{equation}
        \notag
        \left(\vx, 1\right) = \sum_{k} \lambda_k \tilde{\vw}_k(i) \in \spanl{(\mB_k)_i: k\in[K]}.
    \end{equation}
    This immediately implies that $\spanl{(\mB_k)_i: k\in[K]} = \R^{\absx{\pa_{\G}(i)}+1}$, so that \Cref{asmp:independent-row} holds.

    Conversely, suppose that \Cref{asmp:independent-row} holds, then for $\forall \vx \in \R^{\absx{\pa_{\G}(i)}}$, there exists $\lambda_k \in \R, 1\leq k\leq \absx{\pa_{\G}(i)}$ such that $\sum_{k} \lambda_k \tilde{\vw}_k(i) = \left(\vx, 1\right)$. Hence we have $\sum_{k} \lambda_k \vw_k(i) = \vx$ and $\sum_k\lambda_k=1$, implying \Cref{asmp:independent-row-orig}.
\end{proof}

\section{Properties of effect-domination sets}
\label{appsec:dom-property}

\begin{lemma}
    \label{lemma:dom-equivalent}
    \begin{itemize}
        \item $j\in\dom_{\G}(i)$ if and only if $\hch_{\G}(i) \subseteq \ch_{\G}(j)$;
        \item when $i\neq j$, $j\in\dom_{\G}(i)$ if and only if $\hch_{\G}(i) \subseteq \hch_{\G}(j)$.
    \end{itemize}
\end{lemma}

\begin{proof}
    If $j\in\dom_{\G}(i)$, by definition $i\in\ch_{\G}(j)$ and $\ch_{\G}(i) \subseteq \ch_{\G}(j)$, so that $\hch_{\G}(i) \subseteq \ch_{\G}(j)$. Conversely, $\hch_{\G}(i) \subseteq \ch_{\G}(j)$ implies that $i\in\ch_{\G}(j)$ and $\ch_{\G}(i) \subseteq \ch_{\G}(j)$, so $j\in\dom_{\G}(i)$. This proves the first claim.

    To prove the second claim, assume that $\hch_{\G}(i) \subseteq \hch_{\G}(j)$ holds but $\hch_{\G}(i) \subseteq \ch_{\G}(j)$ does not hold, then we must have $j\in\hch_{\G}(i)$. since $j\neq i$, we have $j\in\ch_{\G}(i)$, but then $i\notin\hch_{\G}(j)$, which is a contradiction. Hence $\hch_{\G}(i) \subseteq \ch_{\G}(j)$ and the conclusion follows from the first claim.
\end{proof}

\begin{lemma}
\label{prop:sp-chain}
    Let $\G$ be a DAG and $i$ be its node, then for $\forall j\in\pa_{\G}(i)$, we have $\dom_{\G}(j) \subseteq \pa_{\G}(i)$.
\end{lemma}

\begin{proof}
    Let $k \in \dom_{\G}(j)$, then by definition we have $\ch_{\G}(j)\subseteq\ch_{\G}(k)$. In particular, we have $i\in\ch_{\G}(k) \Rightarrow k\in\pa_{\G}(i)$.
\end{proof}

\begin{lemma}
    \label{lemma:sp-sp-chain}
    Let $\G$ be a DAG and $i$ be its node, then for $\forall j\in\dom_{\G}(i)$, we have $\dom_{\G}(j) \subseteq \dom_{\G}(i)$.
\end{lemma}

\begin{proof}
    Let $k\in\dom_{\G}(j)$, then by definition we have $\hch_{\G}(j)\subset\hch_{\G}(k)$. We also know that $\hch_{\G}(i)\subset\hch_{\G}(j)$, so $\hch_{\G}(i)\subset\hch_{\G}(k)$, implying that $k\in\dom_{\G}(i)$. 
\end{proof}

\begin{lemma}
\label{prop:sp-respecting-inverse}
    If $\mM\in\mathcal{M}_{\dom}^0(\G)$, then $\mM^{-1}\in\mathcal{M}_{\dom}^0(\G)$. 
\end{lemma}

\begin{proof}
    Assume WLOG that the nodes of $\G$ satisfy $i\in\pa_{\G}(j)\Rightarrow i<j$ (otherwise we can choose a different index of the nodes and correspondingly swap some rows and columns of $\mM$). Since $i\in\dom_{\G}(j)\Rightarrow i\in\pa_{\G}(j)$, it follows that $\mM$ must be lower triangular and the diagonal entries are nonzero.
    
    Let $\mN=\mM^{-1}$, then for $\forall i\in[d]$, we have
    \begin{equation}
        \label{eq:N-M-inverse}
        \sum_{j=1}^d \mN_{ij}\mM_{j\ell} = 0, \quad \forall\ell\notin\overline{\dom}_{\G}(i).
    \end{equation}
    Since $\mM\in\mathcal{M}_{\dom}^0(\G)$, we have $\mM_{j\ell}=0$ for $\forall j$ such that $\ell\notin\overline{\dom}_{\G}(j)$. By \Cref{lemma:sp-sp-chain}, if $j\in\overline{\dom}_{\G}(i)$, then $\ell\notin\overline{\dom}_{\G}(i)$ necessarily implies that $\ell\notin\overline{\dom}_{\G}(j)$. Hence the left hand side of \Cref{eq:N-M-inverse} is essentially a sum over $j\notin\overline{\dom}_{\G}(i)$, \emph{i.e.},
    \begin{equation}
        \notag
        \sum_{j\notin\overline{\dom}_{\G}(i)} \mN_{ij}\mM_{j\ell} = 0, \quad \forall\ell\notin\overline{\dom}_{\G}(i).
    \end{equation}
    Viewing the above as a system of linear equations in $\mN_{ij}, j\notin\overline{\dom}_{\G}(i)$, the coefficient matrix $\left(\mM_{j\ell}\right)_{j,\ell\in \notin\overline{\dom}_{\G}(i)}$ must be invertible since it is a sub-matrix of the invertible lower-triangular matrix $\mM$. As a result, we necessary have $\mN_{ij}=0, \forall j\notin\overline{\dom}_{\G}(i)$. Finally, $\mN=\mM^{-1}$ must be invertible, so $\mN\in\mathcal{M}_{\dom}^0(\G)$ as desired.
\end{proof}

\begin{lemma}
\label{lemma:dom-preserving-map-inverse}
    Suppose that $\psi:\R^d\mapsto\R^d$ is a diffeomorphism and $\G$ be a DAG, such that for $\forall i\in[d]$, $\psi_i(\vz)$ is a function of $\vz_{\overline{\dom}_{\G}(i)}$. Then for $\forall j\in[d]$, $(\psi^{-1})_j(\vv)$ is a function of $\vv_{\overline{\dom}_{\G}(j)}$. 
\end{lemma}

\begin{proof}
    Let $\mJ_{\vz} = \mJ_{\psi}(\vz)$ be the Jacobian matrix of $\psi$. Since $\psi$ is a diffeomorphism, $\mJ_{\vz}$ is invertible for any $\vz\in\R^d$. Moreover, our assumption implies that $(\mJ_{\vz})_{ij}=0, \forall j\notin\overline{\dom}_{\G}(i)$, so $\mJ_{\vz}\in\mathcal{M}_{\dom}^0(\G)$. By \Cref{prop:sp-respecting-inverse}, $\mJ_{\vz}^{-1}\in\mathcal{M}_{\dom}^0(\G)$. But $\mJ_{\vz}^{-1}$ is exactly the Jacobian matrix of $\psi^{-1}$ at $\vv=\psi(\vz)$, hence it follows that $(\psi^{-1})_j(\vv)$ is only a function of $\vv_{\overline{\dom}_{\G}(j)}$, as desired.
\end{proof}

\begin{lemma}
    \label{lemma:dom-equivalence-relation}
    The binary relation $\sim_{\dom}$ defined in \Cref{def:dom-ambiguity} is an equivalence relation.
\end{lemma}

\begin{proof}
    It is obvious that $(\vh,\G) \sim_{\dom} (\vh,\G)$ holds for any model $(\vh,\G)$.\\

    Suppose that $(\vh_1,\G_1) \sim_{\dom} (\vh_2,\G_2)$, then there exists a permutation $\pi$ on $[d]$ and a diffeomorphism $\psi:\R^d\mapsto\R^d$ where $\psi_i(\vz)$ is a function of $\vz_{\overline{\dom}_{\G_1}(i)}$, such that $i\in\pa_{\G_1}(j)\Leftrightarrow \pi(i)\in\pa_{\G_2}(\pi(j))$ and $\mP_{\pi}\circ\vh_2=\psi\circ\vh_1$. Then we can write $\mP_{\pi}^{-1}\circ\vh_1 = \hat{\psi}\circ\vh_2$ 
    where $\hat{\psi}=\mP_{\pi}^{-1}\circ\psi^{-1}\circ\mP_{\pi}$. By \Cref{lemma:dom-preserving-map-inverse}, we know that $\left(\psi^{-1}\right)_j(\vv)$ is a function of $\vv_{\overline{\dom}_{\G_1}(j)}$, so $(\hat{\psi})_j$ is a function of $\vv_{\pi(\overline{\dom}_{\G_1}(j))}=\vv_{\overline{\dom}_{\G_2}(j)}$, implying that $(\vh_2,\G_2) \sim_{\dom} (\vh_1,\G_1)$.\\

    Finally, let $(\vh_1,\G_1) \sim_{\dom} (\vh_2,\G_2)$ and $(\vh_2,\G_2) \sim_{\dom} (\vh_3,\G_3)$, then we can write
    \begin{equation}
        \notag
        \mP_{\pi}\circ\vh_2 = \psi\circ\vh_1 \quad \text{and}\quad \mP_{\hat{\pi}}\circ\vh_3 = \hat{\psi}\circ\vh_2
    \end{equation}
    where: for $\forall i\in[d]$, $\psi_i(\vz)$ is a function of $\vz_{\overline{\dom}_{\G_1}(i)}$, $\hat{\psi}_i(\vz)$ is a function of $\vz_{\overline{\dom}_{\G_2}(i)}$, $i\in\pa_{\G_1}(j)\Leftrightarrow \pi(i)\in\pa_{\G_2}(\pi(j))$ and $i\in\pa_{\G_2}(j)\Leftrightarrow \hat{\pi}(i)\in\pa_{\G_2}(\hat{\pi}(j))$. Then, we can write
    \begin{equation}
        \notag
        \mP_{\pi}\circ\mP_{\hat{\pi}}\circ\vh_3 = \mP_{\pi}\circ\hat{\psi}\circ\mP_{\pi}^{-1}\circ\psi\circ\vh_1.
    \end{equation}
    Since $\hat{\psi}_i(\vz)$ is a function of $\vz_{\overline{\dom}_{\G_2}(i)}$, we deduce that $\left(\mP_{\pi}\circ\hat{\psi}\circ\mP_{\pi}^{-1}\right)_i(\vz)$ is a function of $\vz_{\overline{\dom}_{\G_1}(i)}$. Hence, $\left(\mP_{\pi}\circ\hat{\psi}\circ\mP_{\pi}^{-1}\circ\psi\right)_i(\vz) = \left(\mP_{\pi}\circ\hat{\psi}\circ\mP_{\pi}^{-1}\right)_i\left(\psi(\vz)\right)$ is a function of $\psi_{\overline{\dom}_{\G_1}(i)}(\vz)$. The definition of $\psi$ implies that for each $j\in\overline{\dom}_{\G_1}(i)$, $\psi_j(\vz)$ is a function of $\vz_{\overline{\dom}_{\G_1}(j)}$. By \Cref{lemma:sp-sp-chain}, we have $\cup_{j\in\overline{\dom}_{\G_1}(i)}\overline{\dom}_{\G_1}(j) \subseteq\overline{\dom}_{\G_1}(i)$. Hence $\left(\mP_{\pi}\circ\hat{\psi}\circ\mP_{\pi}^{-1}\circ\psi\right)_i(\vz)$ is still a function of $\vz_{\overline{\dom}_{\G_1}(i)}$.
    Moreover, we also have $i\in\pa_{\G_1}(j)\Leftrightarrow \pi(i)\in\pa_{\G_2}(\pi(j))\Leftrightarrow \hat{\pi}\circ\pi(i)\in\pa_{\G_2}(\hat{\pi}\circ\pi(j))$, so by definition, $(\vh_1,\G_1) \sim_{\dom} (\vh_3,\G_3)$, as desired.
\end{proof}

\section{Omitted proofs from \Cref{sec:linear} and \Cref{sec:algorithm}}
\label{appsec:omitted-proof}

\subsection{Proof of \Cref{thm:linear-main-thm}}
\label{proof:linear-main-thm}

According to the assumption, we have that $\epsilon = \mB_k\mH\vx$ and $\hat{\epsilon}=\hat{\mB}_k\hat{\mH}\vx$, so that $\epsilon = \mB_k\mH(\hat{\mB}_k\hat{\mH})^{\dagger}\hat{\epsilon}, \forall k\in[K]$. By \Cref{lemma:asmp-no-gaussian}, we know that for each $k$, $\mP_k := \mB_k\mH(\hat{\mB}_k\hat{\mH})^{\dagger}$ is a signed permutation matrix, so that $\epsilon=\mP_k\hat{\epsilon}$. Since for any $i\neq j$, $\hat{\epsilon}_i \overset{d}{\neq} \hat{\epsilon}_j$, we must have $\absx{\mP}_1=\absx{\mP}_2=\cdots=\absx{\mP}_K=:\mP$ and $\epsilon = \mP\hat{\epsilon}$, where $\absx{M}$ denotes the resulting matrix by taking the absolute value of all entries in $\mM$. Thus, we can WLOG assume that $\epsilon=\hat{\epsilon}$, since otherwise we can permute the noise variables $\hat{\epsilon}$, and also permute the rows of $\mB_k$ correspondingly. In other words, suppose that the permutation matrix $\absx{\mP}$ has $\absx{\mP}_{k_i,i}=1, i\in[d]$, then we can assign to each node $i$ in $\hat{\G}$ a new index $k_i$ and work with the new indices.

In this case, by \Cref{lemma:asmp-no-gaussian} we have $\mB_k\mH=\bm{\Sigma}_k\hat{\mB}_k\hat{\mH},\forall k\in[K]$ or equivalently $\bm{\Sigma}_k\hat{\mB}_k=\mB_k\mT$, where $\mT = \mH\hat{\mH}^{\dagger} \in\R^{d\times d}$, and $\bm{\Sigma}_k$ is a diagonal matrix with diagonal entries in $\{+1,-1\}$. Let $\hat{\hat{\mB}}_k = \bm{\Sigma}_k\hat{\mB}_k$, then the rows of $\hat{\hat{\mB}}_k$ equals (up to sign) to the rows of $\hat{\mB}_k$. 

To summarize, we now know that i) $\hat{\hat{\mB}}_k = \mB_k\mT, k\in[K]$, ii) $(\mB_k)_{ij} \neq 0 \Leftrightarrow j\in\hpa_{\G}(i)$, and similarly, $(\hat{\hat{\mB}}_k)_{ij} \neq 0 \Leftrightarrow j\in\hpa_{\hat{\G}}(i)$, and iii) Both $\left\{\mB_k\right\}$ and $\left\{\hat{\hat{\mB}}_k\right\}$ satisfy the node-level non-degeneracy assumption \Cref{asmp:independent-row}. For any two such matrices that satisfy such a set of conditions, it must necessarily be true that $\G=\hat{\G}$.
\begin{lemma}[Graph Identifiability]
Consider any two sets matrices $\{\hat{\hat{\mB}}_k\}_{k\in [K]}$ and $\{\mB_k\}_{k\in [K]}$ and associated graphs $\G, \hat{\G}$. If these sets and graphs satisfy that:
    \begin{enumerate}
    \item \label{cond:first} $\hat{\hat{\mB}}_k = \mB_k\mT, k\in[K]$;
    \item \label{cond:second} $(\mB_k)_{ij} \neq 0 \Leftrightarrow j\in\hpa_{\G}(i)$, and similarly, $(\hat{\hat{\mB}}_k)_{ij} \neq 0 \Leftrightarrow j\in\hpa_{\hat{\G}}(i)$.
    \item \label{cond:third} Both $\left\{\mB_k\right\}$ and $\left\{\hat{\hat{\mB}}_k\right\}$ satisfy the node-level non-degeneracy assumption \Cref{asmp:independent-row}.
    \end{enumerate}
then it must hold that $\G=\hat{\G}$.
\end{lemma}
\begin{proof}
We prove this via induction on the size of the graph $d$. Note that here $\G=\hat{\G}$ is not up to permutation and our statement is equivalent to $\pa_{\G}(i)=\pa_{\hat{\G}}(i), \forall i\in[d]$.

If $d=1$, \emph{i.e.}, $\G=\hat{\G}$ obviously holds since both are graphs with only $1$ node.

Suppose that for all graphs $\G$ of size $d-1$, the graph $\hat{\G}$ satisfying all given assumptions must necessarily be equal to $\G$. Now, we consider the case that $\G$ has $d$ nodes. WLOG we can assume that the nodes of $\G$ are properly indexed such that $i\in\pa_{\G}(j)\Rightarrow i<j$, so $\mB_k, k\in[K]$ are lower-triangular matrices. (However, it is currently unknown whether $\hat{\hat{\mB}}_k$ are also lower-triangular.)

By our assumption that $i\in\pa_{\G}(j)\Rightarrow i<j$, the node $d$ in $\G$ has no child. Thus we can write
\begin{equation}
    \notag
    \mB_k = \left(
    \begin{array}{cc}
        \mB_k^- & \bm{0} \\
        \vb_k & c_k
    \end{array}
    \right),
    \mT = \left(
    \begin{array}{cc}
        \mT^- & \times \\
        \times & \times
    \end{array}
    \right)
    \text{ and }
    \hat{\hat{\mB}}_k = \mB_k\mT = \left(
    \begin{array}{cc}
        \hat{\hat{\mB}}_k^- & \times \\
        \times & \times
    \end{array}
    \right)
\end{equation}
where $\mB_k^-,\mT,\hat{\hat{\mB}}_k^- = \mB_k^-\mT^-\in\R^{(d-1)\times(d-1)}, \vb_k\in\R^{d-1},c_k\in\R$ and $\times$ denotes irrelevant entries. 

Let $\mA_k^-,\hat{\mA}_k^-,\bm{\Omega}_k^-$ and $\hat{\bm{\Omega}}_k^-$ be the top-left $(d-1)\times(d-1)$ sub-matrices of $\mA_k,\hat{\mA},\bm{\Omega}_k$ and $\hat{\bm{\Omega}}_k$ respectively, and $\G^-$ and $\hat{\G}^-$ are graphs obtained by deleting node $d$ and all related edges from $\G$ and $\hat{\G}$. Then it is easy to see that 
\begin{equation}
    \label{eq:A-reduced-nonzeros}
    \left(\mA_k^-\right)_{ij}\neq 0 \Leftrightarrow j\in\pa_{\G^-}(i)\quad\text{and}\quad\left(\hat{\mA}_k^-\right)_{ij}\neq 0 \Leftrightarrow j\in\pa_{\hat{\G}^-}(i).
\end{equation}
Moreover,
\begin{equation}
    \notag
    \left(
    \begin{array}{cc}
        \mB_k^- & \bm{0} \\
        \vb_k & c_k
    \end{array}
    \right) = \mB_k = \bm{\Omega}_k^{-\frac{1}{2}}\left(\mI-\mA_k\right) = \left(\begin{array}{cc}
        \left(\bm{\Omega}_k^-\right)^{-\frac{1}{2}} & \bm{0} \\
        \bm{0} & \times
    \end{array}\right)
    \left(\begin{array}{cc}
        \mI-\mA_k^- & \times \\
        \times & \times
    \end{array}\right) =
    \left(\begin{array}{cc}
        \left(\bm{\Omega}_k^-\right)^{-\frac{1}{2}}(\mI-\mA_k^-) & \times \\
        \times & \times
    \end{array}\right)
\end{equation}
so that $\mB_k^-=\left(\bm{\Omega}_k^-\right)^{-\frac{1}{2}}(\mI-\mA_k^-)$. Similarly, we have $\hat{\hat{\mB}}_k^-=\left(\hat{\bm{\Omega}}_k^-\right)^{-\frac{1}{2}}(\mI-\hat{\mA}_k^-)$.

We can also verify that $\left\{\mB_k^-\right\}_{k=1}^K$ and $\left\{\hat{\hat{\mB}}_k^-\right\}_{k=1}^K$ are node-level independent in the sense of \Cref{asmp:independent-row}. 
We only prove this for $\left\{\hat{\hat{\mB}}_k\right\}_{k=1}^K$; the arguments used for $\left\{\mB_k\right\}_{k=1}^K$ are exactly the same as the first case considered below. Now for each $i\in[d-1]$, let $\mR_i\in\R^{K\times d}$ be the matrix whose $k$-th row is the $i$-th row of $\hat{\hat{\mB}}_k$, and $\mR_i^-\in\R^{K\times(d-1)}$ be the matrix whose $k$-th row is the $i$-th row of $\hat{\hat{\mB}}_k^-$, then obviously $\mR_i$ is of form $\left[\mR_i^-,\vr_i\right]$. We consider two cases:
\begin{itemize}
    \item \textbf{Case 1.$d\notin\pa_{\hat{\G}}(i)$} This means that the last entry of the $i$-th row of $\hat{\hat{\mB}}_k$ is zero. Thus $\vr_i=\bm{0}$, and $\rank{\mR_i^-}=\rank{\mR_i}=\absx{\hpa_{\hat{\G}}(i)}=\absx{\hpa_{\hat{\G}^-}(i)}$, where the second equality follows from \Cref{asmp:independent-row}. 
    \item \textbf{Case 2.$d\in\pa_{\hat{\G}}(i)$} In this case we have $\rank{\mR_i^-}\geq\rank{\mR_i} -1=\absx{\hpa_{\hat{\G}}(i)}-1=\absx{\hpa_{\hat{\G}^-}(i)}$. Due to our assumption on $\hat{\mA}_k$ and the relationship $\hat{\hat{\mB}}_k^-=\left(\hat{\bm{\Omega}}_k^-\right)^{-\frac{1}{2}}(\mI-\hat{\mA}_k^-)$, we know that each row of $\mR_i^-$, namely the $i$-th row of some $\hat{\hat{\mB}}_k$ only has $\absx{\hpa_{\hat{\G}}(i)}-1=\absx{\hpa_{\hat{\G}^-}(i)}$ non-zero entries, so that $\rank{\mR_i^-}=\absx{\hpa_{\hat{\G}^-}(i)}$ holds.
\end{itemize}
Since we have shown that the matrices $\mB_k^-$ and $\hat{\hat{\mB}}_k^-$ satisfy the three properties that we assume for induction with $\mT$ replaced by $\mT^-$ and $\G,\hat{\G}$ replaced by $\G^-,\hat{\G}^-$ respectively, 
by induction hypothesis, we can thus deduce that $\G^-=\hat{\G}^-$. To prove $\G=\hat{\G}$ it remains to show that the dependency of node $d$ on the remaining nodes are the same in $\G$ and $\hat{\G}$.

\textbf{\textit{First}}, we show that $\ch_{\hat{\G}}(d)=\emptyset$. Suppose in contrary that there is some $i\in\ch_{\hat{\G}}(d)$, then $\absx{\mathrm{pa}_{\G}(i)} = \absx{\mathrm{pa}_{\G^-}(i)} = \absx{\mathrm{pa}_{\hat{\G}^-}(i)} = \absx{\mathrm{pa}_{\hat{\G}}(i)}-1$. Recalling that $\left(\mB\right)_i$ denotes the $i$-th row of matrix $\mB$, we have
\begin{equation}
    \label{eq:linear-proof-wrong-dim}
    \begin{aligned}
        \mathrm{dim}\left(\mathrm{span}\left\langle \left(\hat{\hat{\mB}}_{k}\right)_i: 1\leq k\leq K\right\rangle\right) &= \mathrm{dim}\left(\mathrm{span}\left\langle \left(\mB_{k}\right)_i: 1\leq k\leq K\right\rangle\right)\\
        &\leq \absx{\mathrm{pa}_{\G}(i)} + 1 < \absx{\mathrm{pa}_{\hat{\G}}(i)} +1,
    \end{aligned}
\end{equation}
where the first inequality follows from $\left(\hat{\hat{\mB}}_{k}\right)_i=\left(\mB_{k}\right)_i\mT$ and \Cref{lemma:linalg-invertible-transformation}, the second holds since each $\left(\mB_{k}\right)_i$ has nonzero elements only at coordinates in $j\in\hpa_{\G}(i)$, and the last one holds since $\absx{\mathrm{pa}_{\G}(i)} = \absx{\mathrm{pa}_{\hat{\G}}(i)}-1$. 
However, \Cref{eq:linear-proof-wrong-dim} contradicts the non-degeneracy condition \Cref{asmp:independent-row} that we assume for matrices $\hat{\mB}_k, k\in[K]$ in the statement of the theorem. Therefore we have $\mathrm{ch}_{\hat{\G}}(d)=\emptyset = \mathrm{ch}_{\G}(d)$. 

\textbf{\textit{Second}}, by a similar argument comparing the number of nonzero elements in the last row of $\mB_k$ and $\hat{\hat{\mB}}_k$, we can also deduce that
\begin{equation}
    \notag
    \label{eq:last-row-same-nonzero}
    \absx{\pa_{\G}(d)} = \absx{\pa_{\hat{\G}}(d)}.
\end{equation} 
Indeed, since $\left(\hat{\hat{\mB}}_{k}\right)_d=\left(\mB_{k}\right)_d\mT$, by \Cref{lemma:linalg-invertible-transformation} we have 
\begin{equation}
    \notag
    \mathrm{dim}\left(\mathrm{span}\left\langle \left(\hat{\hat{\mB}}_{k}\right)_d: 1\leq k\leq K\right\rangle\right) = \mathrm{dim}\left(\mathrm{span}\left\langle \left(\mB_{k}\right)_d: 1\leq k\leq K\right\rangle\right)
\end{equation}
However, since we assume that \Cref{asmp:independent-row} is satisfied for $\left\{\mB_k\right\}_{k=1}^K$ and $\left\{\hat{\mB}_k\right\}_{k=1}^K$, we know that the LHS and RHS of the above equation are equal to $\absx{\pa_{\G}(d)}+1$ and $\absx{\pa_{\hat{\G}}(d)}+1$ respectively, implying \Cref{eq:last-row-same-nonzero}.

\textbf{\textit{Third}}, we show that $\mathrm{pa}_{\G}(d) = \mathrm{pa}_{\hat{\G}}(d)$. Suppose the contrary, let $\ell$ be the smallest element in $\mathrm{pa}_{\G}(d)\Delta\mathrm{pa}_{\hat{\G}}(d)$, where $A\Delta B := (A\setminus B)\cup(B\setminus A)$. Recall that while $\G$ and $\hat{\G}$ are originally not symmetric as nodes are topologically sorted according to $\G$, now we have shown that $\G^-\equiv\hat{\G}^-$ and that $\mathrm{ch}_{\G}(d)=\mathrm{ch}_{\hat{\G}}(d)=\emptyset$, so we can assume WLOG 
that $\ell\in\mathrm{pa}_{\G}(d)$ and $\ell\notin\mathrm{pa}_{\hat{\G}}(d)$, and the other case can be handled symmetrically. Since $\mB_k$ is lower triangular and $\left(\mB_k\right)_{jj} = \left(\Omega_k\right)_{jj}^{-\frac{1}{2}}\neq 0, \forall j\in[d]$, the top-left $\ell\times\ell$ sub-matrix of $\mB_k$, which we denote by $\left[\mB_k\right]_{\ell,\ell}$, must be invertible. This implies that $\left\{\left[\mB_k\right]_{\ell,\ell}^\top\bm{\lambda}:\bm{\lambda}\in\R^{\ell}\right\} = \R^\ell$, so we can always find coefficients $\lambda_{kj}, j\in[\ell]$ such that the first $\ell$ entries of the vector $(\mB_k)_d-\sum_{i=1}^{\ell}\lambda_{ki}(\mB_k)_i\in\R^d$ are all zero. Since $\hat{\hat{\mB}}_k=\mB_k\mT$ and $\mT$ is invertible, we have $\left(\hat{\hat{\mB}}_k\right)_d-\sum_{j=1}^{\ell}\lambda_{kj}\left(\hat{\hat{\mB}}_k\right)_j = \left(\left(\mB_k\right)_d-\sum_{j=1}^{\ell}\lambda_{kj}\left(\mB_k\right)_j\right)\mT, \forall k\in[K]$ and
\begin{equation}
    \notag
    \begin{aligned}
        \mathrm{dim}\left(\mathrm{span}\left\langle \left(\hat{\hat{\mB}}_k\right)_d-\sum_{j=1}^{\ell}\lambda_{kj}\left(\hat{\hat{\mB}}_k\right)_j: k\in[K] \right\rangle\right) &= \mathrm{dim}\left(\mathrm{span}\left\langle \left(\mB_k\right)_d-\sum_{j=1}^{\ell}\lambda_{kj}\left(\mB_k\right)_j: k\in[K] \right\rangle\right) \\
        &\leq \absx{\mathrm{pa}_{\G}(d)\setminus[\ell]}+1.
    \end{aligned}
\end{equation}
Here, the inequality holds because for any coordinate $t\in[d]$,
\begin{equation}
    \label{eq:substract-linear-combination}
    \left(\left(\mB_k\right)_d-\sum_{j=1}^{\ell}\lambda_{kj}\left(\mB_k\right)_j\right)_t = \left\{
    \begin{aligned}
        0 &\quad \text{if } t\leq\ell \\
        \left(\mB_k\right)_{d,t} &\quad \text{otherwise}
    \end{aligned}
    \right.
\end{equation}
where we note that $\mB_k$ is lower-triangular and thus $\left(\mB_k\right)_{j,t}=0, \forall j\leq\ell, t>\ell$. This implies that $\left(\left(\mB_k\right)_d-\sum_{j=1}^{\ell}\lambda_{kj}\left(\mB_k\right)_j\right)_t$ is nonzero only if $t>\ell$ and $t\in\pa_{\G}(d)$.

On the other hand, let $S=\left(\mathrm{pa}_{\hat{\G}}(d)\cap[\ell]^c\right)\cup\{d\}$, then  
\begin{equation}\notag
    \begin{aligned}
        \mathrm{dim}\left(\mathrm{span}\left\langle \left(\hat{\hat{\mB}}_k\right)_d-\sum_{j=1}^{\ell}\lambda_{kj} \left(\hat{\hat{\mB}}_k\right)_{j}: k\in[K] \right\rangle\right) 
        &\geq \mathrm{dim}\left(\mathrm{span}\left\langle \left(\left(\hat{\hat{\mB}}_k\right)_d-\sum_{j=1}^{\ell}\lambda_{kj}\left(\hat{\hat{\mB}}_k\right)_j\right)_S: k\in[K] \right\rangle\right) \\
        &= \mathrm{dim}\left(\mathrm{span}\left\langle \left(\left(\hat{\hat{\mB}}_k\right)_d\right)_S: k\in[K] \right\rangle\right) = |S|.
    \end{aligned}
\end{equation}
where we recall that $\vu_S$ denotes the vector $\left(u_i:i\in S\right)\in\R^{\absx{S}}$. Here the first equality holds due to the same reason as \Cref{eq:substract-linear-combination}, and the second follows from \Cref{asmp:independent-row}. To see why this is the case, note that \Cref{asmp:independent-row} implies that the $K\times\left(\absx{\pa_{\G}(d)}+1\right)$ having $\left((\mB_k)_d\right)_{\hpa_{\G}(d)}$ as the $k$-th row has full column rank, so that the sub-matrix obtained by extracting columns corresponding to the node set $S$ also has full column rank.

We have shown that $\absx{\hpa_{\hat{\G}}(d)\cap[\ell]^c} = \absx{S} \leq \absx{\pa_{\G}(d)\cap[\ell]^c}+1 = \absx{\hpa_{\G}(d)\cap[\ell]^c}$. On the other hand, recall that by our choice of $\ell$, we have $\absx{\hpa_{\G}(d)\cap[\ell-1]} = \absx{\hpa_{\hat{\G}}(d)\cap[\ell-1]}$ and $\ell\in\hpa_{\G}(d)\setminus\hpa_{\hat{\G}}(d)$. Putting these together, we have $\absx{\hpa_{\G}(d)}>\absx{\hpa_{\hat{\G}}(d)}$. However, we know from \Cref{eq:last-row-same-nonzero} that $\absx{\mathrm{pa}_{\G}(d)} = \absx{\mathrm{pa}_{\hat{\G}}(d)}$, leading to a contradiction. Hence, such $\ell$ shouldn't exist and we must have $\mathrm{pa}_{\G}(d) = \mathrm{pa}_{\hat{\G}}(d)$, completing the induction step for graphs of size $d$.

By the principle of induction, we have shown that $\G=\hat{\G}$ holds for any graphs under given assumptions.
\end{proof}

Now that we have established that $\G=\hat{\G}$, we prove the remaining part of the theorem. Note that for any $i,j\in [d]$ such that $i\notin\hpa_{\G}(j)$, we have $(\mB_k)_{ji}=(\hat{\hat{\mB}}_k)_{ji}=0, \forall k\in[K]$. Since $\hat{\hat{\mB}}_k=\mB_k\mT$, we have
\begin{equation}
    \notag
    \sum_{\ell\in\hpa_{\G}(j)} (\mB_k)_{j\ell}\mT_{\ell i} = 0.
\end{equation}
By \Cref{asmp:independent-row}, the above implies that $\mT_{\ell i} = 0$ for $\forall \ell\in\hpa_{\G}(j)$. In short, we have argued that if there exists $j$ such that $i\notin\hpa_{\G}(j)$ and $\ell\in\hpa_{\G}(j)$, then $\mT_{\ell i}=0$. 

This implies that $T_{\ell i}$ is non-zero only if $\bar{\mathrm{ch}}_{\G}(\ell)\subseteq\bar{\mathrm{ch}}_{\G}(i)$. Since $\vv=\mT\vz$, we have $\vv_{\ell}=\sum_{i=1}^d \mT_{\ell i}\vz_i = \sum_{i\in [d]: \bar{\mathrm{ch}}_{\G}(\ell)\subseteq\bar{\mathrm{ch}}_{\G}(i)} \mT_{\ell i}\vz_i$. Note that when $i\neq\ell$, $\bar{\mathrm{ch}}_{\G}(\ell)\subseteq\bar{\mathrm{ch}}_{\G}(i)$ is equivalent to $i\in\dom_{\G}(\ell)$, so $\vv_{\ell}$ only depends on $\vz_{\overline{\dom}_{\G}(\ell)}$ by \Cref{lemma:dom-equivalent}, as desired.

\subsection{Formal version and proof of \Cref{thm:single-node-lower-bound-informal}}
\label{proof:single-node-lower-bound}

In previous works \citep{seigal2022linear,zhang2023identifiability}, it is common to consider single-node soft interventions in the following sense:

\begin{assumption}
    \label{asmp:soft-interventions}
    For $\forall 2\leq k\leq K$, there exists $i_k\in[d]$, such that the structural equation in environment $k$ satisfies \Cref{eq:structural-equatin-zi} satisfies $\vw_k(i) = \vw_1(i)$ and $\omega_{k,i,i}=\omega_{1,i,i}$ for $\forall i\neq i_k$.
\end{assumption}

Let $S_i = \left\{k: 2\leq k\leq K, i_k=i\right\}, i\in[d]$ and $s_i = \absx{S_i}$. Suppose that $\G$ has $e = \sum_{i=1}^d \absx{\pa_{\G}(i)}$ edges, then we can view the weight vectors $\left\{(\vw_k(i), \omega_{k,i,i}): k=1 \text{ or } i=i_k\right\}$ as elements of the Euclidean space $\R^{e+\sum_{k=2}^K \absx{\pa_{\G}(i_k)}}\times \R_{+}^{d+K-1}$. Under \Cref{asmp:soft-interventions}, the models can be fully determined by these weight vectors. The following result states that if we restrict ourselves to single-node interventions, then in the worst case, $\Theta(d^2)$ interventions are required.

\begin{theorem}
\label{thm:single-node-lower-bound}
    There exists a causal graph $\G$ with $\Theta(d^2)$ edges, such that for any unmixing matrix $\mH\in\R^{d\times n}$ with full row rank, any independent noise variables $\epsilon$, and any $s_i >0, i\in[d]$ such that $s_i \leq \absx{\pa_{\G}(i)}$ for some $i$, the following holds: except from a null set of the weight space $\R^{e+\sum_{k=2}^K \absx{\pa_{\G}(i_k)}}\times \R_{+}^{d+K-1}$ (w.r.t the Lebesgue measure), there must exist a candidate solution $(\hat{\mH},\hat{\G})$ and a hypothetical data generating process
    \begin{equation}
        \notag
        \forall k\in[K],\quad \vv = \hat{\mA}_k \vv + \hat{\bm{\Omega}}_k^{\frac{1}{2}}\epsilon,\quad \vx = \hat{\mH}^{\dagger}\vv
    \end{equation}
    such that
    \begin{enumerate}[label={(\roman*$'$)}]
        \item the unmixing matrix $\hat{\mH}\in\R^{d\times n}$ has full row rank;
        \item $\forall k\in[K]$ and $i,j\in[d]$, $(\hat{\mA}_k)_{ij}\neq 0 \Leftrightarrow j\in\pa_{\hat{\G}}(i)$ and $\hat{\bm{\Omega}}_k$ is a diagonal matrix with positive entries;
        \item for $\forall 2\leq k\leq K$, the weight matrices $\hat{\mA}_k, \hat{\bm{\Omega}}_k$ of environment $E_k$ are from a single-node soft intervention on $E_1$ on node $i_k$, in the sense of \Cref{asmp:soft-interventions},
    \end{enumerate}
    but $\G$ is non-isomorphic to $\hat{\G}$.
\end{theorem}

In this subsection we give the full proof of \Cref{thm:single-node-lower-bound}. We say that $S\subseteq\R^m$ is a null set if it has zero Lebesgue measure. Obviously, any hyperplanes in $\R^m$ are null sets. We will also need the following simple lemma:
\begin{lemma}
\label{lemma:solution-with-nonzero-first-entry}
    Suppose that $m\in\mathbb{Z}_+$ and $\mV$ is a subspace of $\R^m$. Then for any set of vectors $\vu_i\in\R^m, i=1,2,\cdots,n$ that does not lie in $\mV$, there must exists $\vv\in\R^m$ such that $\vu_i^\top\vv\neq 0, \forall i\in[n]$ but $\vv\in\mV^\perp$, where $\mV^\perp$ is the orthogonal space of $\mV$.
\end{lemma}
\begin{proof}
    Let $\vw_i$ be the orthogonal projection of $\vu_i$ onto $\mV^\perp$. Since $\vu_i\notin\mV$, we know that $\vw_i\neq\bm{0}$. The solution space of each equation $\vw_i^\top\vv=0$ in $\mV^\perp$ must then be a proper subspace of $\mV^\perp$. Equipped with the Lebesgue measure, all these spaces are null sets in $\mV^\perp$, so one can always choose a $\vv\in\mV^\perp$ that does not lie in any of these solution spaces. Such $\vv$ satisfies all the requirements.
\end{proof}

We choose $\G$ to be the graph with $i\to j$ for $\forall 1\leq i<j\leq d$, so that $\G$ has $\frac{d(d-1)}{2}$ edges. Suppose that $i_0\in[d]$ satisfies $s_i \leq \absx{\pa_{\G}(i)}-1$, then we must have $i_0\geq 2$, so there is an edge $1\to i_0$ in $\G$, Let $\hat{\G}$ be the resulting graph obtained via removing the edge $1\to i_0$ in $\G$, then $\G$ and $\hat{\G}$ are clearly non-isomorphic.

Note that the $i$-th row of $\mB_k$ can be written as $\omega_{k,i,i}^{-\frac{1}{2}}\left(\ve_i-(\mA_k)_i\right)$. Let's choose an lower-triangular matrix $\mT = (t_{ij})_{i,j=1}^d\in\R^{d\times d}$ with columns $\vt_i, i\in[d]$ such that the following holds:
\begin{equation}
    \label{eq:single-node-constraints}
        \left(\ve_i-(\mA_k)_i\right)^\top\vt_j = \left\{
        \begin{aligned}
            = 0, &\quad \forall k \in \{1\}\cup S_{i_0}, j=1 \text{ and } i=i_0 \\
            > 0, &\quad \forall i=j \text{ and } k\in\{1\}\cup S_i \\
            \neq 0, &\quad \forall \text{ remaining } (i,j,k)\in \{k=1,j<i\}\cup\{k\geq 2, i=i_k,j<i\}
        \end{aligned}
        \right.
\end{equation}
and
\begin{equation}
    \label{eq:single-node-diagonal-constraints}
    t_{ii} \neq 0,\quad \forall i\in[d].
\end{equation}
We now show that: except from a null set in the weight space, such $\mT$ can always be chosen. To see why this is the case, we first consider all the constraints on $\vt_1$:
\begin{equation}
    \label{eq:single-node-constraints-column1}
        \left(\ve_i-(\mA_k)_i\right)^\top\vt_1 = \left\{
        \begin{aligned}
            = 0, &\quad \forall k \in \{1\}\cup S_{i_0}\text{ and } i=i_0 \\
            > 0, &\quad \forall i=1 \text{ and } k\in\{1\}\cup S_i \\
            \neq 0, &\quad \forall \text{ remaining } (i,k)\in \{k=1,i>1\}\cup\{k\geq 2, i=i_k>1\}
        \end{aligned}
        \right.
\end{equation}
Now let $\mV = \spanl{\ve_i-(\mA_k)_i: k \in \{1\}\cup S_{i_0}\text{ and } i=i_0}$ and $R$ be the set of pairs $(i,k)$ specified in the second and third row of \Cref{eq:single-node-constraints-column1}. For $\forall (i,k)$, let $\vw_k(i)$ be the weight vector of node $i$ in the environment $k$, \emph{i.e.}, the vector of nonzero entries in $(\mA_k)_i$. Then for $\forall (i,k)\in R$, the following set (as a subset of the weight space)
\begin{equation}
    \label{eq:null-set}
    \bigcup_{k^* \in \{1\}\cup S_{i_0}} \left\{ \ve_{i_0}-\vw_{k^*}(i_0) \in \spanl{\ve_i-\vw_{k}(i), \ve_{i_0}-\vw_{k'}(i_0): k'\in \{1\}\cup S_{i_0}\setminus\{k^*\}} \right\}
\end{equation}
must be a null set. Thus
\begin{equation}
    \label{eq:combine-null-set}
    \mE = \bigcup_{(i,k)\in R} \bigcup_{k^* \in \{1\}\cup S_{i_0}} \left\{ \ve_{i_0}-\vw_{k^*}(i_0) \in \spanl{\ve_i-\vw_{k^*}(i), \ve_{i_0}-\vw_{k'}(i_0): k'\in \{1\}\cup S_{i_0}\setminus\{k^*\}} \right\}
\end{equation}
is also a null set. For any weights that are not in $\mE$, we necessarily have
\begin{equation}
    \notag
    \ve_i-\vw_{k}(i)\notin \spanl{\ve_i-(\mA_k)_i: k \in \{1\}\cup S_{i_0}\text{ and } i=i_0} = \mV,\quad (i,k)\in R.
\end{equation}
Let $U = \left\{ \ve_i-\vw_{k}(i): (i,k)\in R\right\}$, then we can apply \Cref{lemma:solution-with-nonzero-first-entry} to deduce that there exists $\vt_1$ such that
\begin{equation}
    \label{eq:single-node-constraints-column1-first-step}
    \left(\ve_i-(\mA_k)_i\right)^\top\vt_1 = \left\{
    \begin{aligned}
        = 0, &\quad \forall k \in \{1\}\cup S_{i_0}\text{ and } i=i_0 \\
        \neq 0, &\quad \forall \text{ remaining } (i,k)\in \{k=1\}\cup\{k\geq 2, i=i_k\}
    \end{aligned}
    \right.
\end{equation}

Note that the only difference between \Cref{eq:single-node-constraints-column1-first-step} and \Cref{eq:single-node-constraints-column1} is that the latter one further requires that 
\begin{equation}
    \notag
    \left(\ve_1-(\mA_k)_1\right)^\top\vt_1 > 0, \quad\forall k\in\{1\}\cup S_i.
\end{equation}
while the former only guarantees that these terms are nonzero. However, recall that $(\mA_k)_{ij}\neq 0 \Rightarrow j\in\pa_{\G}(i)\Rightarrow j<i$, so the above essentially says that $t_{11}>0$. This can be easily guaranteed by replacing the solution $\vt_1$ we obtained satisfying \Cref{eq:single-node-constraints-column1-first-step} with $-\vt_1$ if needed.

Assuming that the weights do not lie in the null set $\mE$ we have shown that $\vt_1$ can always be chosen to satisfy all constraints imposed on it. We now proceed to choose the remaining entries of $\mT$. The remaining entries in $\vt_1$ can be chosen arbitrarily. For $\vt_j, j>1$, we note that the remaining constraints in \Cref{eq:single-node-constraints} that need to be satisfied consist of the "nonzero" part and the "positivity" part. The positivity constrains can always be satisfied by choosing a sufficiently large $t_{jj}$ for $j>1$.

After choosing the $\vt_j$'s satisfying the positivity constraints, the nonzero constraints along with \Cref{eq:single-node-diagonal-constraints} are easy to fulfill by slightly perturbing $\vt_j$ if they are violated; since each of these constraints are only violated in a zero-measure set of the weight space.
Hence, we have shown that except a null set $\mE$ in the weight space, there always exists some $\mT$ satisfying \Cref{eq:single-node-constraints}. Such $\mT$ must be invertible since it is lower-triangular and its diagonal entries are nonzero. Now let $\hat{\mH}=\mT^{-1}\mH$ 
and $\hat{\bm{\Omega}}_k$ be the diagonal matrix with entries $\hat{\omega}_{k,i,i} = t_{ii}^{-2}\cdot\omega_{k,i,i}, i\in[d]$ and 
\begin{equation}
    \label{eq:def-hat-Ak}
    \hat{\mA}_k=\mI-\hat{\bm{\Omega}}_k^{\frac{1}{2}}\bm{\Omega}_k^{-\frac{1}{2}}(\mI-\mA_k)\mT.
\end{equation}
First since $\mT$ is invertible and $\mH$ has full rank, $\hat{\mH}$ must also have full row rank. 
Second, 
\begin{equation}
    \notag
    (\hat{\mA}_k)_{ij} = \left\{
    \begin{aligned}
        1 - \hat{\omega}_{k,i,i}^{\frac{1}{2}}\omega_{k,i,i}^{-\frac{1}{2}}t_{ii} = 0 &\quad \text{if } j=i \\
        - \hat{\omega}_{k,i,i}^{\frac{1}{2}}\omega_{k,i,i}^{-\frac{1}{2}}\left(\ve_i-(\mA_k)_i\right)^\top\vt_j = 0 &\quad \text{if } j> i \\
        - \hat{\omega}_{k,i,i}^{\frac{1}{2}}\omega_{k,i,i}^{-\frac{1}{2}}\left(\ve_i-(\mA_k)_i\right)^\top\vt_j &\quad \text{if } j< i.
    \end{aligned}
    \right.
\end{equation}
where we again recall that both $\mA_k$ and $\mT$ are lower-triangular.
From \Cref{eq:single-node-constraints} we can see that
\begin{itemize}
    \item When $i=i_0$ and $j=1$, we have 
    \begin{itemize}
        \item $(\hat{\mA}_k)_{i_0,1} = 0$ if $k\in\{1\}\cup S_{i_0}$, and
        \item $(\hat{\mA}_k)_{i_0,1} = (\hat{\mA}_1)_{i_0,1} = 0$ if $k\notin\{1\}\cup S_{i_0}$, by definition of $S_{i_0}$ and \Cref{asmp:soft-interventions}.
    \end{itemize}
    \item When $i> j$ and $(i,j)\neq (i_0,1)$, we have
    \begin{itemize}
        \item $(\hat{\mA}_k)_{ij} \neq 0$ if $k=1$ or $i=i_k$, which directly follows from \Cref{eq:single-node-constraints}, and
        \item $(\hat{\mA}_k)_{ij} = (\hat{\mA}_1)_{ij} \neq 0$, by \Cref{asmp:soft-interventions}. 
    \end{itemize}
\end{itemize}
To summarize, for each $k$, $(\mA_k)_{ij}\neq 0 \Leftrightarrow j\in\pa_{\G}(i) \text{ and } (i,j)\neq (i_0,1)$.

Finally, let $\hat{\vw}_k(i)$ be the weight vector of node $i$ in environment $k$ in the hypothetical model \emph{i.e.}, the vector of nonzero entries in $(\mA_k)_i$, and $\mT_S$ be the submatrix of $\mT$ by selecting the rows and columns in the index set $S$, then by \Cref{eq:def-hat-Ak} we have that
\begin{equation}
    \label{eq:old-new-weights-relation}
    \hat{\omega}_{k,i,i}=t_{ii}^{-2}\cdot\omega_{k,i,i},\quad \hat{\omega}_{k,i,i}^{\frac{1}{2}}\omega_{k,i,i}^{-\frac{1}{2}}\vw_k(i)\mT_{\pa_{\G}(i)} = \left\{
    \begin{aligned}
        \hat{\vw}_k(i) &\quad \text{if } i\neq i_0 \\
        \left[0, \hat{\vw}_k(i)\right] &\quad \text{if } i= i_0
    \end{aligned}
    \right.
\end{equation}

By our assumption, for $\forall k\geq 2$, $i\neq i_k \Rightarrow \vw_k(i)=\vw_1(i)$ and $\omega_{k,i,i}=\omega_{1,i,i}$. Thus \Cref{eq:old-new-weights-relation} imply that $\forall k\geq 2$, $i\neq i_k \Rightarrow \hat{\vw}_k(i)=\hat{\vw}_1(i)$ and $\hat{\omega}_{k,i,i}=\hat{\omega}_{1,i,i}$.
In other words, a single-node intervention on node $i_k$ in environment $k$ in the ground-truth model corresponds to a single-node intervention on node $i_k$ in environment $k$ in the hypothetical model, thereby completing the proof.

\subsection{Proof of \Cref{thm:alg-guarantee}}
\label{proof:alg-guarantee}

We first prove two lemmas.

\begin{lemma}
    \label{lemma:span-Mk-rows}
    $\forall i\in[d]$, we have $\spanl{(\mM_k)_i: k\in[K]} = \spanl{\vh_j: j\in\hpa_{\G}(i)}$.
\end{lemma}

\begin{proof}
    Since $(\mM_k)_i=(\mB_k)_i\mH$, and $(\mB_k)_{ij}\neq 0 \Leftrightarrow j\in\hpa_{\G}(i)$, we can see that $(\mM_k)_i\in\spanl{\vh_j: j\in\hpa_{\G}(i)}$. On the other hand, since $\mH$ is invertible, by \Cref{asmp:independent-row} we have $\dim \spanl{(\mM_k)_i: k\in[K]} = \dim \spanl{(\mB_k)_i: k\in[K]} = \absx{\hpa_{\G}(i)}$. Thus we must have $\spanl{(\mM_k)_i: k\in[K]}=\spanl{\vh_j: j\in\hpa_{\G}(i)}$.
\end{proof}

\begin{lemma}
\label{lemma:alg-same-spaces}
    Let $\hat{S}$ be an ancestral set of graph $\G$ and $\hat{\mV}_k = \mathrm{span}\left\langle (\mM_k)_s: s\in \hat{S}\right\rangle, k\in[K]$. Then we have $\mV_1=\mV_2=\cdots=\mV_K=\mathrm{span}\left\langle \vh_s: s\in \hat{S}\right\rangle$.
\end{lemma}

\begin{proof}
    Recall that $\mM_k = \mB_k\mH$, so for $\forall s\in \hat{S}$, the $s$-th row of $\mM_k$ can be written as
    \begin{equation}
        \label{eq:Mk-in-span-H}
        (\mM_k)_s = \sum_{t=1}^d (\mB_k)_{st}\vh_t = \sum_{t\in\hpa_{\G}(s)} (\mB_k)_{st}\vh_t \in \mathrm{span}\left\langle\vh_s:s\in \hat{S}\right\rangle
    \end{equation}
    where the last equation is because $\hat{S}$ is ancestral $\Rightarrow \hpa_{\G}(s)\subseteq \hat{S}$. Thus, for $\forall k\in[K]$, $\hat{\mV}_k = \mathrm{span}\left\langle (\mM_k)_s: s\in \hat{S}\right\rangle \subseteq \mathrm{span}\left\langle\vh_s:s\in \hat{S}\right\rangle$. On the other hand, recall that both $\mB_k$ and $\mH$ have full rank, so $\mM_k$ has full row rank as well, which implies that $\dim \mV_k = \absx{S} = \dim \mathrm{span}\left\langle\vh_s:s\in \hat{S}\right\rangle$. Hence, $\mV_k = \mathrm{span}\left\langle\vh_s:s\in \hat{S}\right\rangle, \forall k\in[K]$.
\end{proof}

The following two propositions show that our algorithm always maintain an ancestral set, recursively adds a new node into the set and correctly identifies its parents.

\begin{proposition}[\Cref{prop:alg2} restated]
\label{app-prop:alg2}
    The following two propositions hold for \Cref{alg:learn-graph}:
    \begin{itemize}
        \item $\ans_{\G}(i)\subseteq S \Leftrightarrow$ the \texttt{if} condition in line 8 of \Cref{line:if-rank1} is fulfilled;
        \item the set $S$ maintained in \Cref{alg:learn-graph} is always an ancestral set, in the sense that $j\in S \Rightarrow \ans_{\G}(j)\subseteq S$.
    \end{itemize}
\end{proposition}

\begin{proof}
    At the starting point, we have $S=\emptyset$ which is obviously an ancestral set. Now suppose that after the $\ell$-th iteration, $S=\left\{s_1,s_2,\cdots,s_{\ell}\right\}$ is an ancestral set. In the following, we show that $\ans_{\G}(i)\subseteq S \Leftrightarrow$ the \texttt{if} condition in line 8 is fulfilled. This would immediately imply that there always exists a node $i$ that can be added into $S$ in the $(\ell+1)$-th iteration, and that after adding $i$, $S$ is still an ancestral set.

    Suppose that $\ans_{\G}(i)\subseteq S$ for some $i\notin S$, by \Cref{lemma:span-Mk-rows} we know that $(\mM_k)_i \in \mathrm{span}\left\langle \vh_j: j\in \hpa_{\G}(i)\right\rangle$, so there exists $\alpha_k\in\R$ such that $(\mM_k)_i-\alpha_k\vh_i \in \mathrm{span}\left\langle \vh_j: j\in \pa_{\G}(i)\right\rangle$. Moreover, since $(\mM_k)_i = \sum_{j\in\hpa_{\G}(i)}(\mB_k)_{jj}\vh_j$, $(\mB_k)_{ii} = \omega_{k,i,i}^{-\frac{1}{2}}\neq 0$ and $\mH$ has full row rank by assumption, we must have $(\mM_k)_i\notin \mathrm{span}\left\langle \vh_j: j\in \pa_{\G}(i)\right\rangle$ and so $\alpha_k\neq 0$.  
    Thus, we have by the linearity of the projection operator
    \begin{equation}
        \notag
        \vq_k :=\mathrm{proj}_{\mV_k^{\perp}} \left((\mM_k)_i\right) = \mathrm{proj}_{\mV_k^{\perp}} \left((\mM_k)_i-\alpha_k\vh_i\right)+ \mathrm{proj}_{\mV_k^{\perp}} \left(\alpha_k\vh_i\right) = \alpha_k\mathrm{proj}_{\mV_k^{\perp}} \left(\vh_i\right).
    \end{equation}
    Recall that all the $\mV_k$'s are the same and equal $\mathrm{span}\left\langle\vh_s:s\in S\right\rangle$ by \Cref{lemma:alg-same-spaces}. So $\dim\mathrm{span}\left\langle \vq_k: k\in[K]\right\rangle\leq 1$. Since $\mH$ has full row rank, we have $\vh_i\notin \mathrm{span}\left\langle\vh_s:s\in S\right\rangle = \mV_k$, so that $\dim\mathrm{span}\left\langle \vq_k: k\in[K]\right\rangle= 1$ holds, which is exactly the \texttt{if} condition in line 8.

    Conversely, suppose that there is an $i\notin S$ such that $\ans_{\G}(i)\nsubseteq S$ but $\dim\mathrm{span}\left\langle \vq_k: k\in[K]\right\rangle= 1$ holds. Since $S$ is ancestral, we know that there must be some $j\in\pa_{\G}(i)$ such that $j\notin S$. Since $\ve_i$ and $\ve_j$ both have support on the coordinates in $\hpa_{\G}(i)$,  by \Cref{asmp:independent-row} we know that $\mathrm{span}\langle \ve_i,\ve_j\rangle \subseteq \mathrm{span}\langle (\mB_k)_i: k\in[K]\rangle$,
    so that $\mathrm{span}\langle \vh_i,\vh_j\rangle = \mathrm{span}\langle \ve_i,\ve_j\rangle \mH \subseteq \mathrm{span}\langle (\mB_k)_i: k\in[K]\rangle \mH = \mathrm{span}\langle (\mM_k)_i: k\in[K]\rangle$. Since $\dim\mathrm{span}\left\langle \vq_k: k\in[K]\right\rangle= 1$, there must exist some vector $\vu\in\R^n$ and $\alpha_i,\alpha_j\in\R$ such that $\vh_i-\alpha_i\vu, \vh_j-\alpha_j\vu \in \mV_k = \mathrm{span}\langle \vh_s: s\in S\rangle$. Since $i,j\notin S$ and $\mH$ has full row rank, we can deduce that 
    $\vh_i, \vh_j \notin \mathrm{span}\langle \vh_s: s\in S\rangle$, and so both of $\alpha_i$ and $\alpha_j$ are non-zero. Hence $\alpha_j\vh_i-\alpha_i\vh_j \in \mathrm{span}\langle \vh_s: s\in S\rangle$, which is impossible since we know that $\mH$ has full row-rank.
\end{proof}

\begin{proposition}[\Cref{prop:alg1} restated]
\label{app-prop:alg1}
    Given any ordered ancestral set $S$ that contains $\pa_{\G}(i)$ for some $i\notin S$ 
    , \Cref{alg:identify-parents} returns a set $P_i \subseteq S$ that is exactly $\pa_{\G}(i)$.
\end{proposition}

\begin{proof}
    As we have shown in \Cref{prop:alg2}, for each possible input $(S,i)$ to \Cref{alg:identify-parents}, both $S$ and $S\cup\{i\}$ are ancestral sets, so that $\ans_{\G}(i)\subseteq S$. Similarly one can see that inside the set $S := \left\{ s_1,s_2,\cdots,s_{m}\right\}$, all the ancestors of $s_j$ are contained in $\left\{s_1,s_2,\cdots,s_{j-1}\right\}$. In the following, we show that $\forall m'\in \{0, \ldots, m\}$, $r_{m'}=\absx{\hpa_{\G}(i)-\left\{s_j:j\leq m'\right\}}$ (*).

    By \Cref{lemma:alg-same-spaces} we have $\mW_1=\mW_2=\cdots=\mW_K=\mathrm{span}\left\langle \vh_{s_j}: j\leq m'\right\rangle$. Let $t_1,t_2,\cdots,t_{\ell}$ be elements of $\hpa_{\G}(i)$ that are not in $\left\{s_j:j\leq m'\right\}$, then
    \begin{equation}
        \notag
        \begin{aligned}
            r_{m'} &= \dim\mathrm{span}\left\langle \vp_k: k\in[K]\right\rangle = \dim\left(\mathrm{proj}_{\mathrm{span}\left\langle \vh_{s_j}: j\leq m'\right\rangle^{\perp}} \mathrm{span}\left\langle (\mM_k)_i: k\in[K]\right\rangle \right) \\
            &= \dim\left(\mathrm{proj}_{\mathrm{span}\left\langle \vh_{s_j}: j\leq m'\right\rangle^{\perp}} \mathrm{span}\left\langle \vh_j: j\in\hpa_{\G}(i)\right\rangle \right)\quad \text{(by \Cref{lemma:span-Mk-rows})} \\
            &= \dim\left(\mathrm{proj}_{\mathrm{span}\left\langle \vh_{s_j}: j\leq m'\right\rangle^{\perp}} \mathrm{span}\left\langle \vh_{t_1},\vh_{t_2},\cdots,\vh_{t_{\ell}} \right\rangle \right) \\
            &= \ell \quad \text{(by \Cref{lemma:linalg-proj-dim} and non-degeneracy of $\mH$)}
        \end{aligned}
    \end{equation}
    which proves (*). From (*) it is easy to see that $m'\in\hpa_{\G}(i)$ (and thus in $\pa_{\G}(i)$ since $i\notin S$) if and only if $r_{m'}=r_{m'-1}-1$.
\end{proof}

Now we conclude the proof of \Cref{thm:alg-guarantee}. \Cref{prop:alg2,prop:alg1} directly imply that \Cref{alg:learn-graph} is able to exactly recover the ground-truth causal graph $\G$. It remains to show that Line 20 in \Cref{line:intersect-subspaces} produces the correct $\hat{\vh}_i$'s. By \Cref{lemma:span-Mk-rows} we know that $E_j=\spanl{\vh_{\ell}:\ell\in\hpa_{\G}(j)}$, so
\begin{equation}
    \notag
    \cap_{j\in\hch_{\G}(i)} E_j = \cap_{j\in\hch_{\G}(i)} \spanl{\vh_{\ell}:\ell\in\hpa_{\G}(j)} = \spanl{\vh_{\ell}:\ell\in\overline{\dom}_{\G}(i)}.
\end{equation}
where the last step holds because $\mH$ has full row rank and $\cap_{j\in\hch_{\G}(i)}\hpa_{\G}(j)=\overline{\dom}_{\G}(i)$ by definition. Hence, each $\hat{\vh}_i$ is a linear combination of $\vh_\ell, \ell\in \overline{\dom}_{\G}(i)$, completing the proof.

\section{Omitted Proofs from \Cref{sec:non-param}}

\subsection{Proof of \Cref{lemma:assumptions-relation}}
\label{proof:assumptions-relation}

Let $\vw_k(i)\in\R^{\absx{\pa_{\G}(i)}}$ be the vector obtained by removing all zero entries in the $i$-th row of $\mA_k$ and $\omega_{k,i,i}$ be the $i$-th diagonal entry in $\bm{\Omega}_k$
, then for the $k$-th environment we have $\vz_i = \vw_k(i)^\top\vz_{\pa_{\G}(i)}+\omega_{k,i,i}^{\frac{1}{2}}\epsilon_i$, so that
    \begin{equation}
        \notag
        \hat{p}_k\left(\vz_i\mid\vz_{\pa_{\G}(i)}\right) = \omega_{k,i,i}^{-\frac{1}{2}}p_{\epsilon_i}\left(\omega_{k,i,i}^{-\frac{1}{2}}(\vz_i-\langle\vw_k(i),\vz_{\pa_{\G}(i)}\rangle)\right)
    \end{equation}
    where $p_{\epsilon_i}(\cdot)$ is the density of $\epsilon_i$.
    As a result, we have
    \begin{equation}
        \notag
        \begin{aligned}
            \nabla\frac{\hat{p}_1}{\hat{p}_k}\left(\vz_i\mid\vz_{\pa_{\G}(i)}\right) &= \frac{\hat{p}_1}{\hat{p}_k}\left(\vz_i\mid\vz_{\pa_{\G}(i)}\right)\cdot\nabla\log\frac{\hat{p}_1}{\hat{p}_k}\left(\vz_i\mid\vz_{\pa_{\G}(i)}\right) \\
            &= \frac{\hat{p}_1}{\hat{p}_k}\left(\vz_i\mid\vz_{\pa_{\G}(i)}\right)\cdot \left[ c_{i1} (1,-\vw_1(i)) - c_{ik} (1,-\vw_k(i)) \right]
        \end{aligned}
    \end{equation}
    where for convenience we use $\nabla$ to denote the gradient with respect to all variables $\vz_{\hpa_{\G}(i)}$, and $c_{ik}=\omega_{k,i,i}^{-\frac{1}{2}}\cdot\frac{p_{\epsilon_i}'}{p_{\epsilon_i}}\left(\omega_{k,i,i}^{-\frac{1}{2}}(\vz_i-\langle\vw_k(i),\vz_{\pa_{\G}(i)}\rangle\right)$ (we omit the dependency on $\vz$ for simplicity).

    \Cref{non-degeneracy-distribution} implies that $\mathrm{span}\left\langle c_{i1} (1,-\vw_1(i)) - c_{ik} (1,-\vw_k(i)): 2\leq k\leq K\right\rangle = \R^{\absx{\pa_{\G}(i)}+1}$, thus it holds that $\mathrm{span}\langle (1,-\vw_k(i)): k\in[K]\rangle = \R^{\absx{\pa_{\G}(i)}+1}$ as well. By definition of $\mB_k$, this immediately implies that $\dim\left(\mathrm{span}\left\langle(\mB_k)_i: k\in[K] \right\rangle\right)=\absx{\pa_{\G}(i)}+1$ as desired.

\subsection{Proof of \Cref{thm:single-node-soft-nonparam}}
\label{proof:single-node-soft-nonparam}

Define $\bm{\tau} := \hat{\vh}\circ \vh^{-1}: \R^d\mapsto\R^d$, then we have that $\vv = \bm{\tau}(\vz)$. Since both $\vh$ and $\hat{\vh}$ are diffeomorphisms by assumption, so is $\bm{\tau}$. To avoid confusion, in this section we use $\vz$ (resp. $\vv$) to denote random variables while using $\hat{\vz}$ (resp. $\hat{\vv}$) to denote (deterministic) vectors.

Let $\mathfrak{E}_j = \left\{ E_k^{(j)}: k\in[K_j]\right\}$ be the $j$-th collection of environments according to our assumption. We first prove the following lemma:

\begin{lemma}
    \label{lemma:support-relation}
    $\mO_{\vv} = \bm{\tau}(\mO_{\vz})$.
\end{lemma}

\begin{proof}
    By the change of variable formula \citep{schwartz1954formula}, for $\forall \hat{\vz}\in\R^d$ and $\forall E\in\mathfrak{E}$ we have $p_E(\hat{\vz})=q_E(\hat{\vv})\absx{\det \mJ_{\bm{\tau}}(\hat{\vz})}$, where $\hat{\vv}=\bm{\tau}(\hat{\vz})$. Since $\bm{\tau}$ is a diffeomorphism, we must have $\absx{\det \mJ_{\bm{\tau}}(\hat{\vz})}\neq 0$, so $\hat{\vz}\in\mO_{\vz} \Leftrightarrow \hat{\vv} = \bm{\tau}(\hat{\vz}) \in \mO_{\vv}$, concluding the proof.
\end{proof}

\begin{lemma}
\label{lemma:single-node-intervention-dist-ratio}
    Let $\hat{\vz}\in\mO_{\vz}$. For $\forall j\in[d]$ and $2\leq k\leq K_j$, we have
    \begin{equation}
        \frac{p_j^{E_k^{(j)}}}{p_j^{E_1^{(j)}}}\left(\hat{\vz}_j\mid \hat{\vz}_{\pa_{\G}(j)}\right) = \frac{q_{j}^{E_k^{(j)}}}{q_{j}^{E_1^{(j)}}}\left(\hat{\vv}_{j}\mid \hat{\vv}_{\pa_{\hat{\G}}(j)} \right),
    \end{equation}
    where $\hat{\vv}=\bm{\tau}(\hat{\vz}) \in \mO_{\vv}$.
\end{lemma}

\begin{proof}
    Since $\vv = \bm{\tau}(\vz)$,  
    by the change-of-measure formula \citep{schwartz1954formula} we have that for $\forall \hat{\vz} \in\mO_{\vz}$,
    \begin{equation}
        \label{eq:change-of-measure}
        \begin{aligned}
             \prod_{i=1}^d p_i^E\left(\hat{\vz}_i\mid\hat{\vz}_{\mathrm{pa}_{\G} (i)}\right) &= p_E(\hat{\vz}) = q_E(\hat{\vv})\absx{\det\mJ_{\bm{\tau}}(\hat{\vz})} = \prod_{i=1}^d q_i^E\left(\bm{\tau}_i(\hat{\vz})\mid\bm{\tau}_{\mathrm{pa}_{\hat{\G}}(i)}(\hat{\vz})\right) \absx{\det\mJ_{\bm{\tau}}(\hat{\vz})}
        \end{aligned}
    \end{equation}
    for all $E\in\mathfrak{E}_j$, where $\hat{\vv}=\bm{\tau}(\hat{\vz})$. By Assumption ($\romannumeral 2$) and \Cref{def:intervention}, we know that $p_i^{E_k^{1}}=p_i^{E_1^{(1)}} \Leftrightarrow i \neq 1$ and $q_i^{E_k^{1}}=q_i^{E_1^{(1)}} \Leftrightarrow i \neq 1$ for all $k>1$. Thus, we have that
    \begin{equation}
        \notag
        \prod_{i=1}^d \frac{p_i^{E_k^{(j)}}\left(\hat{\vz}_i\mid\hat{\vz}_{\mathrm{pa}_{\G} (i)}\right)}{p_i^{E_1^{(j)}}\left(\hat{\vz}_i\mid\hat{\vz}_{\mathrm{pa}_{\G} (i)}\right)} = \frac{p_j^{E_k^{(j)}}}{p_j^{E_1^{(j)}}}(\hat{\vz}_j\mid \hat{\vz}_{\hpa_{\G}(j)})
    \end{equation}
    and
    \begin{equation}
        \notag
        \prod_{i=1}^d \frac{q_i^{E_k^{(j)}}\left(\hat{\vv}_i\mid\hat{\vv}_{\mathrm{pa}_{\hat{\G}} (i)}\right)}{q_i^{E_1^{(j)}}\left(\hat{\vv}_i\mid\hat{\vv}_{\mathrm{pa}_{\hat{\G}} (i)}\right)} = \frac{q_j^{E_k^{(j)}}}{q_j^{E_1^{(j)}}}(\hat{\vv}_j\mid \hat{\vv}_{\hpa_{\G}(j)}).
    \end{equation}
    Since the LHS of the above two equations are the same by \Cref{eq:change-of-measure}, the RHS must also be the same, concluding the proof.
\end{proof}

We assume WLOG that the vertices of $\G$ are labelled such that $i\to j \Rightarrow i<j$, and that $\pi(i) =i,\forall i\in[d]$. Also we can assume the nodes are fixed and only consider how they are connected, \emph{i.e.,} $\pi'(i)=i,\forall i\in[d]$. 
\footnote{This is also WLOG because we now have groups of soft interventions where each group corresponds to a single node, so we can just relabel the node in $\hat{\G}$ that corresponds to the $i$-th group as node $i$.}

\begin{lemma}
    \label{lemma:non-degenerate-open}
    We have $\left(\bm{\tau}(\mN_{\vz})\right)^{\mathrm{o}} = \left(\bm{\tau}^{-1}(\mN_{\vv})\right)^{\mathrm{o}} = \emptyset$.
\end{lemma}

\begin{proof}
    The result immediately follows from the assumption that $\mN_{\vz}^{\mathrm{o}}=\mN_{\vv}^{\mathrm{o}}=\emptyset$ and that $\tau$ is a diffeomorphism.
\end{proof}

For any vertex set $V$, we use $\G_V$ to denote its corresponding induced subgraph of $\G$. We first prove the following statements by induction on $j$: 
    \begin{enumerate}[label={(\arabic*)}]
        \item $\forall i\neq j$, $i\in\mathrm{pa}_{\G}(j) \Leftrightarrow i\in\mathrm{pa}_{\G'}\left(j\right)$;
        \item $\forall j\in[d]$, there exists a continuously differentiable function $\phi_i$ such that $\vv_{j}=\phi_j\left(\vz_{\hpa_{\G}(j)}\right)$. Moreover, $\frac{\partial \phi_j}{\partial\vz_j} \not\equiv 0$ (\emph{i.e.}, not always zero). 
        \item $\forall j\in[d]$, there exists a continuously differentiable function $\Upsilon_j$ such that $\vv_{\hpa_{\G}(j)} = \Upsilon_j(\vz_{\hpa_{\G}(j)})$.
    \end{enumerate}

    For $j=1$, by assumption $\mathrm{pa}_{\G}(j)=\emptyset$. 
    \Cref{lemma:single-node-intervention-dist-ratio} implies that for any $\hat{\vz}\in\mO_{\vz}$,
    \begin{equation}
        \label{eq:single-node-intervention-dist-ratio}
        \frac{p_1^{E_k^{(1)}}}{p_1^{E_1^{(1)}}}(\hat{\vz}_1) = \frac{q_{1}^{E_k^{(1)}}}{q_{1}^{E_1^{(1)}}}\left(\hat{\vv}_{1}\mid \hat{\vv}_{\mathrm{pa}_{\hat{\G}}(1)} \right),\forall 2\leq k\leq K_1.
    \end{equation}
    Then for $\forall i \in\hpa_{\hat{\G}}\left(1\right)$, taking the partial derivative w.r.t $\vv_j$ gives
    \begin{equation}
        \notag
        \frac{\partial}{\partial \hat{\vv}_i}\frac{q_{1}^{E_k^{(1)}}}{q_{1}^{E_1^{(1)}}}\left(\hat{\vv}_{1}\mid \hat{\vv}_{\mathrm{pa}_{\hat{\G}}(1)} \right) = \left(\frac{p_1^{E_k^{(1)}}}{p_1^{E_1^{(1)}}}\right)' (\hat{\vz}_1)\cdot\frac{\partial \hat{\vz}_1}{\partial \hat{\vv}_i}  \Rightarrow  \nabla_{\vv_{\hpa_{\hat{\G}}(1)}} \frac{q_{1}^{E_k^{(1)}}}{q_{1}^{E_1^{(1)}}}\left(\hat{\vv}_{1}\mid \hat{\vv}_{\mathrm{pa}_{\hat{\G}}(1)} \right) = \left(\frac{p_1^{E_k^{(1)}}}{p_1^{E_1^{(1)}}}\right)' (\hat{\vz}_1)\cdot\nabla_{\vv_{\hpa_{\hat{\G}}(1)}}\hat{\vz}_1.
    \end{equation}
    Thus, 
    \begin{equation}
        \notag
        \mathrm{rank}\left[\nabla_{\vv_{\hpa_{\hat{\G}}(1)}}\frac{q_{1}^{E_k^{(1)}}}{q_{1}^{E_1^{(1)}}}\left(\hat{\vv}_{1}\mid \hat{\vv}_{\mathrm{pa}_{\hat{\G}}(1)} \right): 2\leq k\leq K_1\right]\leq 1.
    \end{equation}
    Note that the above inequality holds for $\forall \hat{\vv}\in\mO_{\vv}$.
    If $\mathrm{pa}_{\hat{\G}}(1)\neq \emptyset$, then this would contradict the non-degeneracy assumption  \textit{(\romannumeral 3)} which implies that the above matrix should have rank $\geq 2$ at some point $\hat{\vv}\in \mO_{\vv}$. Hence we must have $\mathrm{pa}_{\hat{\G}}(1) = \emptyset$, implying that (1) holds for $j=1$.

    Taking the derivative of both sides of \Cref{eq:single-node-intervention-dist-ratio} w.r.t $\vz_i, i\geq 2$ implies that $\left(\frac{q_{1}^{E_k^{(1)}}}{q_{1}^{E_1^{(1)}}}\right)^{'}(\hat{\vv}_1)\cdot\frac{\partial\hat{\vv}_1}{\partial\hat{\vz}_i}=0$. By our assumption  \textit{(\romannumeral 3)}, for $\forall \hat{\vv}\in\mO_{\vv}\setminus\mN_{\vv}$, there exists $2\leq k\leq K_1$ such that $\left(\frac{q_{1}^{E_k^{(1)}}}{q_{1}^{E_1^{(1)}}}\right)^{'}(\hat{\vv}_1)\neq 0$, and thus we have $\frac{\partial\hat{\vv}_1}{\partial\hat{\vz}_i}=0, \forall \hat{\vz}\in\bm{\tau}^{-1}\left(\mO_{\vv}\setminus\mN_{\vv}\right)$.
    Since $\bm{\tau}$ is a diffeomorphism, we can deduce that $\bm{\tau}^{-1}\left(\mO_{\vv}\setminus\mN_{\vv}\right) = \mO_{\vz}\setminus\bm{\tau}^{-1}\left(\mN_{\vv}\right)$ and $\left(\bm{\tau}^{-1}\left(\mN_{\vv}\right)\right)^{\mathrm{o}}=\emptyset$ by \Cref{lemma:non-degenerate-open}. As a result, we actually have $\frac{\partial\hat{\vv}_1}{\partial\hat{\vz}_i}=0, \forall \hat{\vz}\in\mO_{\vz}$. Hence in $\mO_{\vz}$ there exists a continuous differentiable function $\phi_1$ such that $\vv_1 = \phi_1(\vz_1)$, proving (2). Finally, (3) directly follows from (2) since $\pa_{\G}(1)=\emptyset$, concluding the proof for $j=1$.

    Now suppose that the statement holds up to $j-1$, and we need to prove it for $j$. 
    Again by \Cref{lemma:single-node-intervention-dist-ratio} we have for $\forall \hat{\vz}\in\mO_{\vz}$ that
    \begin{equation}
        \label{eq:single-node-intervention-dist-ratio-2}
        \frac{p_j^{E_k^{(j)}}}{p_j^{E_1^{(j)}}}\left(\hat{\vz}_j\mid \hat{\vz}_{\mathrm{pa}_{\G}(j)}\right) = \frac{q_{j}^{E_k^{(j)}}}{q_{j}^{E_1^{(j)}}}\left(\hat{\vv}_{j}\mid \hat{\vv}_{\mathrm{pa}_{\hat{\G}}(j)} \right),\quad\forall 2\leq k\leq K_j.
    \end{equation}
    For all $i\notin\hpa_{\G}(j)$, taking partial derivative w.r.t. $\vz_i$ gives
    \begin{equation}
        \notag
        0 = \sum_{\ell\in\hpa_{\hat{\G}}\left(j\right)} \frac{\partial}{\partial \hat{\vv}_{\ell}} \frac{q_{j}^{E_k^{(j)}}}{q_{j}^{E_1^{(j)}}}\left(\hat{\vv}_{j}\mid \hat{\vv}_{\mathrm{pa}_{\hat{\G}}(j)} \right)\cdot\frac{\partial \hat{\vv}_{\ell}}{\partial\hat{\vz}_i},\quad\forall 2\leq k\leq K_j,
    \end{equation}
    \emph{i.e.},
    \begin{equation}
        \notag
        \left[ \nabla_{\vv_{\hpa_{\hat{\G}}(j)}} \frac{q_{j}^{E_k^{(j)}}}{q_{j}^{E_1^{(j)}}}\left(\hat{\vv}_{j}\mid \hat{\vv}_{\mathrm{pa}_{\hat{\G}}(j)} \right): 2\leq k\leq K_j\right]^\top \frac{\partial \hat{\vv}_{\hpa_{\hat{\G}}(j)}}{\partial\hat{\vz}_i} = 0.
    \end{equation}
    Similar to the $j=1$ case, by assumption \textit{(\romannumeral 3)}, we know that the above corfficient matrix has full row rank for $\forall \hat{\vv}\in\mO_{\vv}\setminus\mN_{\vv}$, so for $\forall \vz\in\bm{\tau}^{-1}\left(\mO_{\vv}\setminus\mN_{\vv}\right) = \mO_{\vz}\setminus\bm{\tau}^{-1}\left(\mN_{\vv}\right)$, we have $\frac{\partial \hat{\vv}_{\hpa_{\hat{\G}}(j)}}{\partial\hat{\vz}_i} = 0$. Since $\left(\bm{\tau}^{-1}(\mN_{\vv})\right)^{\mathrm{o}}=\emptyset$ by \Cref{lemma:non-degenerate-open}, for all $\hat{\vz}\in\mN_{\vz}$ we can choose a sequence of points $\hat{\vz}^{(i)}, i=1,2,\cdots$ in $\mO_{\vz}$ such that $\hat{\vz}^{(i)}\to\hat{\vz}$. Since $\bm{\tau}$ is a diffeomorphism, its derivatives are continuous and we can deduce that $\frac{\partial \hat{\vv}_{\hpa_{\hat{\G}}(j)}}{\partial\hat{\vz}_i} = \lim_{\ell\to+\infty} \frac{\partial \hat{\vv}_{\hpa_{\hat{\G}}(j)}^{(\ell)}}{\partial\hat{\vz}_i^{(\ell)}} = 0$. As a result, $\frac{\partial \hat{\vv}_{\hpa_{\hat{\G}}(j)}}{\partial\hat{\vz}_i} = 0$ actually holds for all $\vz\in\mO_{\vz}$. Hence, there exists a continuous differentiable function $\Upsilon_j$ such that $\vv_{\hpa_{\hat{\G}}(j)} = \Upsilon_j\left(\vz_{\hpa_{\G}(j)}\right)$.
    
    By our assumption, $\mathrm{pa}_{\G}(j)\subseteq [j-1]$. Suppose that $\mathrm{pa}_{\hat{\G}}(j) \nsubseteq \left\{i:i<j\right\}$, let $\ell\in \mathrm{pa}_{\hat{\G}}(j) \setminus \left\{i:i<j\right\}$, then by induction hypothesis, $\hat{\vv}_{t} = \bm{\tau}_t(\hat{\vz}), \hat{\vz}\in\mO_{\vz}$, $t=1,2,\cdots,j,\ell$ are all functions of $\hat{\vz}_1,\cdots,\hat{\vz}_j$. Since $\bm{\tau}$ is a diffeomorphism and $\mO_{\vz}$ is the support of the distributions $p_{E}, E\in\mathfrak{E}$, we can deduce that the support of the latent variables $\left(\vv_{t}: t=1,2,\cdots,j,\ell\right)$ lie on a submanifold with dimension $\leq j$, which is impossible since $\vv$ is supported on the open set $\mO_{\vv}\subseteq\R^d$ by assumption \textit{(\romannumeral 1)}. 

    Hence, we must have $\mathrm{pa}_{\hat{\G}}(j) \subseteq \left\{i:i<j\right\}$. Furthermore, if there exists $i\in \mathrm{pa}_{\hat{\G}}(j)$ such that $i\notin\pa_{\G}(j)$ 
    , then the induction hypothesis implies that $\frac{\partial \vv_{i}}{\partial \vz_i} \not\equiv 0$, but $\vv_{i}$ is a function of $\vz_{\hpa_{\G}(j)}$ as previously derived, which is also a contradiction. Thus we actually have $\mathrm{pa}_{\hat{\G}}(j) \subseteq \pa_{\G}(j)$. 
    
    In a completely symmetric manner, we can take the derivatives of \Cref{eq:single-node-intervention-dist-ratio-2} w.r.t. $\vv_i, \forall i\in\hpa_{\hat{\G}}(j)$ and obtain that $\pa_{\G}(j) \subseteq \mathrm{pa}_{\hat{\G}}(j)$. 
    Hence, $\mathrm{pa}_{\hat{\G}}(j) = \pa_{\G}(j)$, completing the proof of (1) and (3) for the $j$ case. 

    Finally, if $\frac{\partial\vv_{j}}{\partial\vz_j}\equiv 0$, then by (3) and the induction hypothesis, $\vv_{1},\cdots,\vv_{j}$ 
    are all functions of $\vz_{[j-1]}$, which implies that $\left(\vv_{1},\cdots,\vv_{j}\right)$ 
    lies on a submanifold with dimension $\leq j-1$, again contradicting assumption \textit{(\romannumeral 1)}. Thus $\frac{\partial\vv_{j}}{\partial\vz_j}\not\equiv 0$. This completes the proof of our inductive step.

    To recap, we now know that
    \begin{itemize}
        \item $\G=\hat{\G}$, and
        \item For $\forall i\in[d]$, there exists a function $\Upsilon_i$ such that 
        $\vv_{\hpa_{\G}(i)} = \Upsilon_i\left(\vz_{\hpa_{\G}(i)}\right)$.
    \end{itemize}
    It remains to show that for $\forall k\in\mathrm{pa}_{\G}(i)\setminus\dom_{\G}(i)$, $\Upsilon_i$ doesn't depend on $\vz_k$.

    By definition, if $k\in\mathrm{pa}_{\G}(i)\setminus\dom_{\G}(i)$, we know that there exists $j\in\mathrm{ch}_{\G}(i)$ such that $j\notin\mathrm{ch}_{\G}(k)$. We have shown that $\vv_{i}$, as a component of $\vv_{\hpa_{\G}(j)}$, is a function of $\vz_{\hpa_{\G}(j)}$. By the choice of $k$, we have $k\notin\hpa_{\G}(j)$, so that $\vv_{i}$ does not depend on $\vz_k$. The conclusion follows.

\section{Omitted Proofs for \Cref{thm:informal-linear-ambiguity} and \Cref{thm:informal-non-param-ambiguity}}
\label{appsec:ambiguity}

In this section we provide detailed proofs of main ambiguity results. 

\begin{definition}
\label{def:effect-respecting-matrix}
    We say that a matrix $\mM\in\R^{d\times d}$ is effect-respecting for a causal graph $\G$, or $\mM\in\mathcal{M}_{\dom}(\G)$, if $\mM_{ij}\neq 0 \Leftrightarrow j\in\overline{\dom}_{\G}(i)$. We also write $\mM\in\mathcal{M}_{\dom}^0(\G)$ if $\mM$ is invertible and $\mM_{ij}\neq 0 \Rightarrow j\in\overline{\dom}_{\G}(i)$. Finally, we write $\mM\in\overline{\mathcal{M}}_{\dom}(\G)$ if $\mM_{ij}\neq 0 \Rightarrow j\in\overline{\dom}_{\G}(i)$.
\end{definition}

\begin{remark}
\label{remark:measure}
    By definition $\mathcal{M}_{\dom}^0(\G)$ is the set of all matrices $\mM$ where $\mM_{ij}\neq 0, \forall j\notin\overline{\dom}_{\G}(i)$, so it can be identified as $\R^{d+d_{\G}}$ where $d_{\G}=\sum_{i=1}^d\absx{\dom_{\G}(i)}$. Equipped with the Lebesgue measure, we have $\mathcal{M}_{\dom}(\G)\subset\mathcal{M}_{\dom}^0(\G)\subset\overline{\mathcal{M}}_{\dom}(\G)$ and $\overline{\mathcal{M}}_{\dom}(\G)\setminus\mathcal{M}_{\dom}(\G)$ is a null set. In the remaining part of this section, we will use measure-theoretic statement for $\mM\in \mathcal{M}_{\dom}(\G)$ in the above sense.
\end{remark}

We first present a result that serves as a good starting point to understand why this is the case. It states that latent representations that are equivalent under $\sim_{\dom}$ are essentially generated from the same causal graph.

\begin{proposition}
\label{prop:sp-linear-same-graph}
    Let $\mM$ be an invertible matrix such that $\mM_{ij}\neq 0 \Rightarrow j\in\overline{\dom}_{\G}(i)$. Suppose that the latent variables $\vz\in\R^d$ are generated from any distributions $p_i\left(\vz_i\mid\vz_{\pa_{\G}(i)}\right), i\in[d]$ with joint density $p(\vz)=\prod_{i=1}^d p_i\left(\vz_i\mid\vz_{\pa_{\G}(i)}\right)$ 
    , then the joint density of $\vv = \mM \vz$ can be written as $q(\vv)=\prod_{i=1}^d q_i\left(\vv_i\mid\vv_{\pa_{\G}(i)}\right)$ for some density functions $q_i, i\in[d]$.
\end{proposition}

\subsection{Proof of \Cref{prop:sp-linear-same-graph}}

We first prove the following lemma:

\begin{lemma}
\label{lemma:dom-transform-same-parents}
    Let $\mM\in\mathcal{M}_{\dom}^0(\G)$ and latent variables $\vv=\mM\vz$, then for $\forall i\in[d]$, there exists invertible matrices $\mM_i$ and $\mM_i^-$ such that $\vv_{\pa_{\G}(i)}=\mM_i^-\vz_{\pa_{\G}(i)}$ and $\vv_{\hpa_{\G}(i)}=\mM_i\vz_{\hpa_{\G}(i)}$.
\end{lemma}

\begin{proof}
    $\forall j\in\hpa_{\G}(i)$, we know that $\vv_j$ is a linear function of $\vz_{\ell}, \ell\in\overline{\dom}_{\G}(j)$.
    By \Cref{prop:sp-chain}, we know that $\overline{\dom}_{\G}(j)\subseteq\hpa_{\G}(i)$, so each $\vv_j, j\in\hpa_{\G}(i)$ is a linear function of $\vz_{\hpa_{\G}(i)}$. Thus we can write $\vv_{\hpa_{\G}(i)}=\mM_i\vz_{\hpa_{\G}(i)}$. In the following we argue that $\mM_i$ is invertible. Let $\pi$ be a permutation on $\hpa_{\G}(i)$ such that $k\in\pa_{\G}(\ell) \Rightarrow \pi(k) < \pi(\ell)$ (such $\pi$ can always be chosen since $\G$ is acyclic), then we can write
    \begin{equation}
        \label{eq:impossibility-invertible}
        \left(\hat{\vv}_{\pi(j)}:j\in\hpa_{\G}(i)\right)^\top = \tilde{\mM}_i \left(\hat{\vz}_{\pi(j)}:j\in\hpa_{\G}(i)\right)^\top
    \end{equation}
    where $\tilde{\mM}_i$ is an upper triangular matrix with non-zero diagonal entries 
    by our choice of $\mM$. Since $\mM_i$ can be obtained from $\tilde{\mM}_i$ be exchanging a few rows and columns, $\mM_i$ is invertible as well.

    Similarly, using the fact that $\forall j\in\pa_{\G}(i)$, $\overline{\dom}_{\G}(j)\subseteq\pa_{\G}(i)$, we can prove the existence of an invertible matrix $\mM_i^-$ such that $\vv_{\pa_{\G}(i)}=\mM_i^-\vz_{\pa_{\G}(i)}$.
\end{proof}

Returning to the proof of \Cref{prop:sp-linear-same-graph}. Assume WLOG that the nodes of $\G$ are ordered in a way such that $i\in\pa_{\G}(j) \Rightarrow i<j$, so that $\mM$ is a lower-triangular matrix. 
The joint density of $\vv$ can be written as
\begin{equation}
    \notag
    q(\vv) = \prod_{i=1}^d q\left(\vv_i\mid \vv_1,\cdots,\vv_{i-1}\right).
\end{equation}
Since $\vv=\mM z$ and $\mM$ is lower triangular and invertible (hence, with non-zero diagonals), we know that $(\vv_1,\vv_2,\cdots,\vv_{i-1})$ is an invertible linear function of $(\vz_1,\vz_2,\cdots,\vz_{i-1})$ and $(\vv_1,\vv_2,\cdots,\vv_{i})$ is an invertible linear function of $(\vz_1,\vz_2,\cdots,\vz_{i})$. Let $\hat{\vv}=\mM\hat{\vz}\in\R^d$, then we have
\begin{equation}
    \notag
    \begin{aligned}
        q\left(\hat{\vv}_i\mid \hat{\vv}_1,\cdots,\hat{\vv}_{i-1}\right) &= \frac{q(\hat{\vv}_1,\hat{\vv}_2,\cdots,\hat{\vv}_{i})}{q(\hat{\vv}_1,\hat{\vv}_2,\cdots,\hat{\vv}_{i-1})} = \frac{p(\hat{\vz}_1,\hat{\vz}_2,\cdots,\hat{\vz}_i)\det\hat{\mM}_{1:i,1:i}}{p(\hat{\vz}_1,\hat{\vz}_2,\cdots,\hat{\vz}_{i-1})\det\hat{\mM}_{1:i-1,1:i-1}}\\ 
        &\propto \frac{p(\hat{\vz}_1,\hat{\vz}_2,\cdots,\hat{\vz}_{i})}{p(\hat{\vz}_1,\hat{\vz}_2,\cdots,\hat{\vz}_{i-1})} = p\left(\hat{\vz}_i\mid \hat{\vz}_1,\cdots,\hat{\vz}_{i-1}\right) = p_i\left(\hat{\vz}_i\mid\hat{\vz}_{\pa_{\G}(i)}\right),
    \end{aligned}
\end{equation}
where $\hat{\mM}_{1:i,i:i}$ denotes that top-left submatrix of $\hat{\mM}$ of size $i\times i$, and the last step follows from the causal Markov condition (\Cref{asmp:causal-markov-condition}). On the other hand, let $q_i\left(\hat{\vv}_i\mid\hat{\vv}_{\pa_{\G}(i)}\right)$ be the conditional density of $\vv_i$ on its parents at $\hat{\vv}\in\R^d$. For $\forall j\in\pa_{\G}(i)$, from $\vv=\mM\vz$  
we know that $\vv_j$ is a linear function of $\vz_{\overline{\dom}_{\G}(j)}$. 
By \Cref{lemma:dom-transform-same-parents} we know that $\hat{\vv}_{\pa_{\G}(i)}$ is a linear function of $\hat{\vz}_{\pa_{\G}(i)}$ and $\hat{\vv}_{\hpa_{\G}(i)}$ is a linear function of $\hat{\vz}_{\hpa_{\G}(i)}$, so that
\begin{equation}
    \notag
    q\left(\hat{\vv}_{\pa_{\G}(i)}\right)\propto p\left(\hat{\vz}_{\pa_{\G}(i)}\right)\quad \text{and} \quad q\left(\hat{\vv}_{\hpa_{\G}(i)}\right)\propto p\left(\hat{\vz}_{\hpa_{\G}(i)}\right)
\end{equation}
and
\begin{equation}
    \notag
    q_i\left(\hat{\vv}_i\mid\hat{\vv}_{\pa_{\G}(i)}\right) \propto \frac{p\left(\hat{\vz}_{\hpa_{\G}(i)}\right)}{p\left(\hat{\vz}_{\pa_{\G}(i)}\right)} = p_i\left(\hat{\vz}_i\mid \hat{\vz}_{\pa_{\G}(i)}\right).
\end{equation}
Hence, we have $q_i\left(\hat{\vv}_i\mid\hat{\vv}_{\pa_{\G}(i)}\right) \propto q\left(\hat{\vv}_i\mid \hat{\vv}_1,\cdots,\hat{\vv}_{i-1}\right)$, so that
\begin{equation}
    \notag
    q(\hat{\vv})= \prod_{i=1}^d q_i\left(\hat{\vv}_i\mid\hat{\vv}_{\pa_{\G}(i)}\right) \propto \prod_{i=1}^d q_i\left(\hat{\vv}_i\mid\hat{\vv}_{\pa_{\G}(i)}\right).
\end{equation}
Since both sides integrate to $1$, it turns out that they are equal, as desired.

\subsection{Formal version and proof of \Cref{thm:informal-linear-ambiguity}: the linear case}

\begin{theorem}[Counterpart to \Cref{thm:linear-main-thm}]
\label{thm:linear-ambiguity}
    For any causal model $(\mH,\G)$ and any set of environments $\mathfrak{E}=\left\{E_k: k\in[K]\right\}$, suppose that we have observations $\left\{P_{\mX}^E\right\}_{E\in\mathfrak{E}}$ satisfying \Cref{asmp:data-generating-process}:
    \begin{equation}
        \notag
        \forall k\in[K],\quad \vz = \mA_k \vz + \bm{\Omega}_k^{\frac{1}{2}}\epsilon,\quad \vx = \mH^{\dagger}\vz
    \end{equation}
    such that
    \begin{enumerate}[label={(\roman*)}]
        \item the unmixing matrix $\mH\in\R^{d\times n}$ has full row rank;
        \item $\forall k\in[K]$ and $i,j\in[d]$, $(\mA_k)_{ij}\neq 0 \Leftrightarrow j\in\pa_{\G}(i)$ and $\bm{\Omega}_k$ is a diagonal matrix with positive entries;
        \item $\left\{\mB_k=\bm{\Omega}_k^{-\frac{1}{2}}(\mI-\mA_k)\right\}_{k=1}^K$ are node level non-degenerate in the sense of \Cref{asmp:independent-row},
    \end{enumerate}
    then there must exist a candidate solution $(\hat{\mH},\G)$ and a hypothetical data generating process
    \begin{equation}
        \notag
        \forall k\in[K],\quad \vv = \hat{\mA}_k \vv + \hat{\bm{\Omega}}_k^{\frac{1}{2}}\epsilon,\quad \vx = \hat{\mH}^{\dagger}\vv
    \end{equation}
    such that
    \begin{enumerate}[label={(\roman*$'$)}]
        \item the unmixing matrix $\hat{\mH}\in\R^{d\times n}$ has full row rank;
        \item $\forall k\in[K]$ and $i,j\in[d]$, $(\hat{\mA}_k)_{ij}\neq 0 \Leftrightarrow j\in\pa_{\G}(i)$ and $\hat{\bm{\Omega}}_k$ is a diagonal matrix with positive entries;
        \item $\left\{\hat{\mB}_k=\hat{\bm{\Omega}}_k^{-\frac{1}{2}}(\mI-\hat{\mA}_k)\right\}_{k=1}^K$ are node level non-degenerate in the sense of \Cref{asmp:independent-row},
    \end{enumerate}
    but
    \begin{equation}
        \notag
        \frac{\partial \vv_i}{\partial \vz_j} \neq 0,\quad \forall j\in\overline{\dom}_{\G}(i).
    \end{equation}
    Finally, if we additionally assume that
    \begin{enumerate}[label={(\roman*)}]
        \setcounter{enumi}{2}
        \item the environments are groups of single-node interventions: there exists a partition $\mathfrak{E}=\cup_{i=1}^d \mathfrak{E}_i$ such that $\I_{\vz}^{\mathfrak{E}_i}=\{i\}$ (see \Cref{def:intervention}),
    \end{enumerate}
    then we can guarantee the existence of $(\hat{\mH},\G)$ and weight matrices which, besides the properties listed above, also satisfy
    \begin{enumerate}[label={(\roman*$'$)}]
        \setcounter{enumi}{2}
        \item for the same partition $\mathfrak{E}=\cup_{i=1}^d \mathfrak{E}_i$, we have $\I_{\vv}^{\mathfrak{E}_i}=\{i\}$.
    \end{enumerate}
    In other words, additionally assuming that the environments are from single-node interventions does not resolve the ambiguity.
\end{theorem}

\begin{remark}
    Compared with our identifiability guarantee \Cref{thm:linear-main-thm}, \Cref{thm:linear-ambiguity} actually demonstrates a stronger form of impossibility. Specifically, it states that the SNA cannot be resolved even if both the ground-truth causal graph and the noise variables are known.
\end{remark}

We define
\begin{equation}
    \label{eq:impossibility-v-z-relation}
    \vv = \mM\vz
\end{equation}
where $\mM$ is an effect-respecting matrix. At this point we do not make any other restrictions on $\mM$, but we will specify the appropriate choise of $\mM$ later.

By assumption, the latent variables in the $k$-th environment are generated by
\begin{equation}
    \notag
    \vz=\mA_k\vz+\bm{\Omega}_k^{\frac{1}{2}}\epsilon,
\end{equation}
then $\vv=\mM(\mI-\mA_k)^{-1}\bm{\Omega}_k^{\frac{1}{2}}\epsilon$. Let $\hat{\bm{\Omega}}_k$ be the diagonal matrix with entries $\mM_{ii}^2\cdot(\bm{\Omega}_k)_{ii}, i\in[d]$ and $\hat{\mA}_k=\mI-\hat{\bm{\Omega}}_k^{\frac{1}{2}}\bm{\Omega}_k^{-\frac{1}{2}}(\mI-\mA_k)\mM^{-1}$, then $\vv=\hat{\mA}_k\vv+\hat{\bm{\Omega}}_k^{\frac{1}{2}}\epsilon$. Note that the choice of $\hat{\bm{\Omega}}_k$ here is to that the diagonal entries of $\hat{\mA}_k$ are zero, as we show below. 
It remains to show that: for almost all $\mM\in\mathcal{M}_{\dom}^0(\G)$, it holds for $\forall k\in[K]$ that $(\hat{\mA}_k)_{ij}= 0 \Leftrightarrow j\notin\pa_{\G}(i)$.

For the $\Leftarrow$ direction, since $\mM\in\mathcal{M}_{\dom}^0(\G)$, $\mM^{-1}\in\mathcal{M}_{\dom}^0(\G)$ as well. Thus, $\forall j\notin\pa_{\G}(i)$ we have
\begin{equation}
    \notag
    \begin{aligned}
        \left[(\mI-\mA_k)\mM^{-1}\right]_{ij} &= \sum_{\ell=1}^d (\mI-\mA_k)_{i\ell}\cdot(\mM^{-1})_{\ell j} = \sum_{\ell\in\hpa_{\G}(i)\cap\{\ell':j\in\overline{\dom}_{\G}(\ell')\}} (\mI-\mA_k)_{i\ell}\cdot(\mM^{-1})_{\ell j} \\
        &= \left\{
            \begin{aligned}
                0 &\quad \text{if } j\notin\hpa_{\G}(i) \\
                (\mM^{-1})_{ii} &\quad \text{if } j=i
            \end{aligned}
        \right.
    \end{aligned}
\end{equation}
where the last step holds because $\forall \ell\in [d]$, $\ell\in\hpa_{\G}(i), j\in\overline{\dom}_{\G}(\ell) \Rightarrow j\in\hpa_{i}$, and when $j=i$, the only such $\ell$ is $\ell=i$. Hence, we can see that our choice of $\hat{\mA}_k$ satisfies 
\begin{equation}
    \notag
    \left(\hat{\mA}_k\right)_{ij}= \left\{
    \begin{aligned}
        0-0=0 &\quad \text{if } j\notin\hpa_{\G}(i) \\
        1-\hat{\omega}_{k,i,i}^{\frac{1}{2}}\omega_{k,i,i}^{-\frac{1}{2}}(\mM^{-1})_{ii} = 0 &\quad \text{if } j=i,
    \end{aligned}
    \right.
\end{equation}
so $\left(\hat{\mA}_k\right)_{ij}\neq 0 \Rightarrow j\in\pa_{\G}(i)$.

Conversely, for $\forall j\in\pa_{\G}(i)$,
\begin{equation}
    \label{eq:zero-entry-constraints}
    (\hat{\mA}_k)_{ij}= 0 \Leftrightarrow \sum_{s\in\hpa_{\G}(i)} (\mI-\mA_k)_{is}(\mM^{-1})_{sj} = 0 \Leftrightarrow \sum_{s\in\hpa_{\G}(i)} (-1)^s (\mI-\mA_k)_{is} \det \mM_{sj}^- = 0
\end{equation}
where $\mM_{sj}^-$ is the $(d-1)\times(d-1)$ matrix obtained by removing the $s$-th row and $j$-th column of $\mM$, and the second step in the equation above follows from the fact that $\mM^{-1}=\det(\mM)^{-1}\mathrm{adj}(\mM)$, where $\mathrm{adj}(\mM)$ denotes the adjugate matrix of $\mM$ whose $(i,j)$-th entry is $(-1)^{i+j}\det\mM_{ij}^-$. 

\Cref{eq:zero-entry-constraints} holds if only if $\mM$ takes values on a lower-dimensional algebraic manifold of its embedded space $\R^{d+d_{\G}}$ (see \Cref{remark:measure}).  As a result, for almost every $\mM\in\mathcal{M}_{\dom}^0(\G)$, $\vv$ is generated from a linear causal model with graph $\G$ as defined in \Cref{latent-original}. Moreover, let $\hat{\mB}_k=\mB_k\mM^{-1}, k\in[K]$ , so that $\epsilon=\hat{\mB}_k\vv$ in the $k$-th environment. Then for all nodes $i \in [d]$ and $S\subseteq \mathrm{pa}(i)\cup\{i\}$, we have
\begin{align*}
\dim\spanl{\left(\hat{\mB_k}^\top\ve_i\right)_{S}: k\in[K]} =~& \dim\spanl{\mM^{-\top}\left(\left(\mB_k^\top\ve_i\right)_{S}: k\in[K] \right)}\\
=~& \dim\spanl{\left(\mB_k^\top\ve_i\right)_{S}: k\in[K]}
= \absx{\pa_{\G}(i)}+1,
\end{align*}
implying that $\hat{\mB}_k, k\in[K]$ satisfy \Cref{asmp:independent-row}. 

Now we have shown that for almost every $\mM\in\mathcal{M}_{\dom}^0(\G)$, we can construct a hypothetical data generating process with latent variables $\vv=\mM\vz$ that satisfies all requirements in \Cref{thm:linear-ambiguity}. Choose an arbitrary $\mM$ that is in $\mathcal{M}_{\dom}(\G)$, then we have that
\begin{equation}
    \notag
    \frac{\partial \vv_i}{\partial \vz_j}\neq 0, \quad j\notin\overline{\dom}_{\G}(i).
\end{equation}

Finally, if we additionally assume single-node interventions, $\forall k,\ell\in\mathfrak{E}_i$, we have that $(\mB_k)_j\neq(\mB_{\ell})_j \Leftrightarrow j=i$. For any $\mM\in\mathcal{M}_{\dom}^0(\G)$ (and specifically the $\mM$ that we have already chosen above), we have $(\hat{\mB}_k)_j = (\mB_k)_j\mM^{-1}$ and $(\hat{\mB}_\ell)_j = (\mB_\ell)_j\mM^{-1}, \forall j\in[d]$. Thus, $(\hat{\mB}_k)_j\neq (\hat{\mB}_\ell)_j \Leftrightarrow j=i$ as well, implying that $\mathfrak{E}_i$ is also a group of single-node interventions on $\vv$, concluding the proof.

\subsection{Formal statement and proof of \Cref{thm:non-param-ambiguity}: the non-parametric case}
\label{proof:non-param-ambiguity}

\begin{theorem}[Counterpart to \Cref{thm:single-node-soft-nonparam}]
\label{thm:non-param-ambiguity}
    For any causal model $(\vh,\G)$ and any set of environments $\mathfrak{E}$, suppose that we have observations $\left\{P_{\mX}^E\right\}_{E\in\mathfrak{E}}$ satisfying \Cref{asmp:data-generating-process}:
    \begin{equation}
        \notag
        \forall E\in\mathfrak{E}, \vz\sim p_E(\hat{\vz})=\prod_{i=1}^d p_i^E\left(\hat{\vz}_i\mid \hat{\vz}_{\pa_{\G}(i)}\right), \vx = \vh^{-1}(\vz)
    \end{equation}
    such that
    \begin{enumerate}[label={(\roman*)}]
        \item all densities $p_i^E$ are continuously differentiable and the joint density $p_E$ is positive everywhere;
        \item the environments are groups of single-node interventions: there exists a partition $\mathfrak{E}=\cup_{i=1}^d \mathfrak{E}_i$ such that $\I_{\vz}^{\mathfrak{E}_i}=\{i\}$;
        \item the intervention distributions on each node are non-degenerate: $\forall i\in[d]$, the set of distributions $\left\{p_i^E:E\in\mathfrak{E}_i\right\}$ satisfy \Cref{non-degeneracy-distribution} at any point $\hat{\vz}\in\R^d$,
    \end{enumerate}
    then there must exist a candidate solution $(\hat{\vh},\G)$ and a hypothetical data generating process
    \begin{equation}
        \notag
        \forall E\in\mathfrak{E}, \vv\sim q_E(\hat{\vv})=\prod_{i=1}^d q_i^E\left(\hat{\vv}_i\mid \hat{\vv}_{\pa_{\G}(i)}\right), \vx = \hat{\vh}^{-1}(\vv)
    \end{equation}
    such that
    \begin{enumerate}[label={(\roman*$'$)}]
        \item all densities $q_i^E$ are continuously differentiable and the joint density $q_E$ is positive everywhere;
        \item for the same partition $\mathfrak{E}=\cup_{i=1}^d \mathfrak{E}_i$, we have $\I_{\vv}^{\mathfrak{E}_i}=\{i\}$;
        \item $\forall i\in[d]$, the set of distributions $\left\{q_i^E:E\in\mathfrak{E}_i\right\}$ satisfy \Cref{non-degeneracy-distribution} at any point $\hat{\vv}\in\R^d$,
    \end{enumerate}
    but
    \begin{equation}
        \notag
        \frac{\partial \vv_i}{\partial \vz_j} \neq 0,\quad \forall j\in\overline{\dom}_{\G}(i).
    \end{equation}
\end{theorem}

\begin{remark}
    Similar to the case of \Cref{thm:linear-ambiguity}, \Cref{proof:non-param-ambiguity} also establishes a stronger form of identifiability. First, it is assumed that the causal graph $\G$ is known. Second, we only focus on a special case of the setting of \Cref{thm:single-node-soft-nonparam} by assuming that the support is the whole space, and the non-degeneracy condition \Cref{non-degeneracy-distribution} holds at any point. Even in this case, we show that our identification guarantee up to SNA cannot be improved.
\end{remark}

We state and prove a stronger version of \Cref{thm:non-param-ambiguity}:

\begin{theorem}
\label{thm:restate-non-param-ambiguity}
    For any causal model $(\vh,\G)$ and any set of environments $\mathfrak{E}$, suppose that we have observations $\left\{P_{\mX}^E\right\}_{E\in\mathfrak{E}}$ satisfying \Cref{asmp:data-generating-process}:
    \begin{equation}
        \notag
        \forall E\in\mathfrak{E},\quad \vz\sim p_E(z)=\prod_{i=1}^d p_i^E\left(z_i\mid z_{\pa_{\G}(i)}\right),\quad \vx = \vh^{-1}(\vz)
    \end{equation}
    such that
    \begin{enumerate}[label={(\roman*)}]
        \item all densities $p_i^E$ are continuously differentiable and the joint density $p_E$ is positive everywhere;
        \item the environments are groups of single-node interventions: there exists a partition $\mathfrak{E}=\cup_{i=1}^d \mathfrak{E}_i$ such that $\I_{\vz}^{\mathfrak{E}_i}=\{i\}$;
        \item the intervention distributions on each node are non-degenerate: $\forall i\in[d]$, the set of distributions $\left\{p_i^E:E\in\mathfrak{E}_i\right\}$ satisfy \Cref{non-degeneracy-distribution},
    \end{enumerate}
    then there must exist a candidate solution $(\hat{\vh},\G)$ and a hypothetical data generating process
    \begin{equation}
        \notag
        \forall E\in\mathfrak{E},\quad \vv\sim q_E(v)=\prod_{i=1}^d q_i^E\left(v_i\mid v_{\pa_{\G}(i)}\right),\quad \vx = \hat{\vh}^{-1}(\vv)
    \end{equation}
    such that
    \begin{enumerate}[label={(\roman*$'$)}]
        \item all densities $q_i^E$ are continuously differentiable and the joint density $q_E$ is positive everywhere;
        \item for the same partition $\mathfrak{E}=\cup_{i=1}^d \mathfrak{E}_i$, we have $\I_{\vv}^{\mathfrak{E}_i}=\{i\}$;
        \item $\forall i\in[d]$, the set of distributions $\left\{q_i^E:E\in\mathfrak{E}_i\right\}$ satisfy \Cref{non-degeneracy-distribution},
    \end{enumerate}
    but
    \begin{equation}
        \notag
        \frac{\partial \vv_i}{\partial \vz_j} \neq 0,\quad \forall j\in\overline{\dom}_{\G}(i).
    \end{equation}
    Finally, if we additionally assume minimality (\Cref{asmp:minimality}) and/or faithfulness (\Cref{asmp:faithfulness}) of all $p_E$'s, we can guarantee the existence of $(\hat{\vh},\G)$ and $q_E$'s satisfying minimality and/or faithfulness in addition to the properties listed above. In other words, assuming minimality and/or faithfulness does not resolve the ambiguity.
\end{theorem}
\begin{proof}
We define
\begin{equation}
    \label{eq:impossibility-v-z-relation2}
    \vv = \mM\vz
\end{equation}
where $\mM$ is an effect-respecting matrix. At this point we do not make any other restrictions on $\mM$, and we will choose appropriate $\mM$ later. By \Cref{lemma:dom-transform-same-parents}, there exists invertible matrices $\mM_i$ and $\mM_i^-$ such that $\vv_{\pa_{\G}(i)}=\mM_i^-\vz_{\pa_{\G}(i)}$ and $\vv_{\hpa_{\G}(i)}=\mM_i\vz_{\hpa_{\G}(i)}$, so for all environment $E\in\mathfrak{E}$ we have
\begin{equation}
    \notag
    q_i^E(\vv_{\pa_{\G}(i)}) = p_i^E(\vz_{\pa_{\G}(i)})\cdot\absx{\det(\mM_i^-)^{-1}}, \quad q_i^E(\vv_{\hpa_{\G}(i)}) = p_i^E(\vz_{\hpa_{\G}(i)})\cdot\absx{\det(\mM_i)^{-1}}
\end{equation}
so that
\begin{equation}
    \label{eq:impossibility-change-of-variable}
    q_i^E\left(\vv_i\mid\vv_{\mathrm{pa}_{\G}(i)}\right) = p_i^E\left(\vz_i\mid\vz_{\mathrm{pa}_{\G}(i)}\right)\frac{\absx{\det \mM_i^{-1}}}{\absx{\det (\mM_i^-)^{-1}}}, \quad \forall i\in[d].
\end{equation}
In the following, assuming that $\left(p_i^E: E\in\mathfrak{E}\right)$ satisfies any of the listed assumptions, we show that $\left(q_i^E: E\in\mathfrak{E}\right)$ satisfies the same assumption as well.

Firstly, \Cref{eq:impossibility-change-of-variable} immediately implies that the density of $\vv$ is continuous differentiable and positive everywhere. Secondly, $\forall k,\ell\in\mathfrak{E}_i$, we have that
\begin{equation}
    \notag
    p_j^{E_k}\left(\vz_j\mid\vz_{\mathrm{pa}_{\G}(j)}\right) = p_j^{E_{\ell}}\left(\vz_j\mid\vz_{\mathrm{pa}_{\G}(j)}\right) \Leftrightarrow j=i.
\end{equation}
By \Cref{eq:impossibility-change-of-variable} it is easy to see that
\begin{equation}
    \notag
    q_j^{E_k}\left(\vv_j\mid\vv_{\mathrm{pa}_{\G}(j)}\right) = q_j^{E_\ell}\left(\vv_j\mid\vv_{\mathrm{pa}_{\G}(j)}\right) \Leftrightarrow j=i
\end{equation}
as well, \emph{i.e.,} $q^k, k\in \mathfrak{E}_i$ are single-node interventions on $\vv_i$ according to \Cref{def:intervention}.

Thirdly, we verify the non-degeneracy condition for $q_i^E$'s. Indeed we have for $\forall k\geq 2$ that
\begin{equation}
\notag
\begin{aligned}
    \nabla_{\vv_{\hpa_{\G}(i)}}\frac{q_i^{E_1}}{q_i^{E_k}}\left(\vv_i\mid\vv_{\mathrm{pa}_{\G}(i)}\right) 
    &= \frac{\partial\vz_{\hpa_{\G}(i)}}{\partial\vv_{\hpa_{\G}(i)}}\nabla_{\vz_{\hpa_{\G}(i)}}\frac{q_i^{E_1}}{q_i^{E_k}}\left(\vz_i\mid\vz_{\mathrm{pa}_{\G}(i)}\right) 
    = \mM_i^{-1} \nabla_{\vz_{\hpa_{\G}(i)}}\frac{q_i^{E_1}}{q_i^{E_k}}\left(\vz_i\mid\vz_{\mathrm{pa}_{\G}(i)}\right).
\end{aligned}
\end{equation}
Since $\mM_i$ is invertible, the above equation and the non-degeneracy of $p^{E_k},k\in[K]$ immediately implies that non-degeneracy of $q^{E_k},k\in[K]$.

Thus, for arbitrary $\mM\in\mathcal{M}_{\dom}(\G)$, we have constructed a hypothetical data generating process with latent variable $\vv=\mM\vz$ that satisfies all given conditions. It remains to show that such construction is still possible under additional minimality and faithfulness conditions.
\\

\textbf{Claim 1. There exists a neighbourhood $O$ of the identity matrix $\mI$ in $\overline{\mathcal{M}}_{\dom}(\G)$ (in the sense of \Cref{remark:measure}) such that for $\forall \mM\in O\cap\mathcal{M}_{\dom}^0(\G)$, $p^{E_k}, k\in[K]$ satisfy \Cref{asmp:faithfulness} $\Rightarrow$ $q^{E_k}, k\in[K]$ satisfy \Cref{asmp:faithfulness}.}

For $\forall i,j$ not $d$-separated by $S\subseteq[d]$, for all $k\in[K]$ there exists $\hat{\vz}\in\R^d$ such that $\Delta_k^{(i,j,S)}=p^{E_k}\left(\hat{\vz}_i,\hat{\vz}_j\mid \hat{\vz}_S\right)-p^{E_k}\left(\hat{\vz}_i\mid \hat{\vz}_S\right)p^{E_k}\left(\hat{\vz}_j\mid \hat{\vz}_S\right)\neq 0$. By continuous differentiability of $p^{E_k}$, we know that there exists $\delta_k^{(i,j,S)}>0$ such that for all $\mM\in\overline{\mathcal{M}}_{\dom}(\G)$ such that $\normx{\mM-\mI}_F\leq\delta_k^{(i,j,S)}$, the density of the variable $\vv=\mM\vz$ satisfies $q^{E_k}\left(\hat{\vv}_i,\hat{\vv}_j\mid \hat{\vv}_S\right)\neq q^{E_k}\left(\hat{\vv}_i\mid \hat{\vv}_S\right)q^k\left(\hat{\vv}_j\mid \hat{\vv}_S\right)$ for $\hat{\vv}=\mM\hat{\vz}$, which implies that $\vv_i$ and $\vv_j$ are dependent given $\vv_S$. Now choose $\delta=\min_{k,i,j,S}\delta_k^{(i,j,S)}>0$, then for all $\mM\in\overline{\mathcal{M}}_{\dom}(\G)$ such that $\normx{\mM-\mI}_F\leq\delta$, the resulting distributions $q^{E_k}, k\in[K]$ satisfy assumption \Cref{asmp:minimality}.
\\

\textbf{Claim 2. There exists a neighbourhood $O$ of $\mI$ in $\overline{\mathcal{M}}_{\dom}(\G)$ (in the sense of \Cref{remark:measure}) such that for almost all $\mM\in O\cap\mathcal{M}_{\dom}^0(\G)$, $p^{E_k}, k\in[K]$ satisfies \Cref{asmp:minimality} $\Rightarrow$ $p^{E_k}, k\in[K]$ satisfies \Cref{asmp:minimality}.}

The proof is similar to the previous statement. Since \Cref{asmp:minimality} causal minimality is satisfied for $\vz$, for $\forall k\in[K], i\in[d]$, let $\G_{ij}$ be the resulting graph obtained by removing the edge $j\to i$ from $\G$, then there must exists some $\alpha_{ijk}\in[d]$ such that $\vz_{\alpha_{ijk}}\not\independent\vz_{\nd_{\G_{ij}}(\alpha_{ijk})}\mid \vz_{\pa_{\G_{ij}}(\alpha_{ijk})}$. Hence, there exists $\hat{\vz}^{ijk}\in\R^d$ such that
\begin{equation}
    \notag
    p^{E_k}\left(\hat{\vz}_{\alpha_{ijk}}^{ijk} \mid \hat{\vz}_{\pa_{\G_{ij}}(\alpha_{ijk})}^{ijk}\right) p^{E_k}\left(\hat{\vz}_{\nd_{\G_{ij}}(\alpha_{ijk})}^{ijk}\mid \hat{\vz}_{\pa_{\G_{ij}}(\alpha_{ijk})}^{ijk}\right) \neq p^{E_k}\left(\hat{\vz}_{\overline{\nd}_{\G_{ij}}(\alpha_{ijk})}^{ijk}\mid \hat{\vz}_{\pa_{\G_{ij}}(\alpha_{ijk})}^{ijk}\right).
\end{equation}
By continuous differentiability of $p^{E_k}$, there exists $\delta_k^{(i,j)}>0$ such that for all $\mM\in\bar{\mathcal{M}}_{\dom}(\G)$ such that $\normx{\mM-\mI}_F\leq\delta_k^{(i,j)}$, the density $q_{ij}^{E_k}$ of the variable $\hat{\vv}^{ijk}=\mM\hat{\vz}^{ijk}$ satisfies 
\begin{equation}
    \notag
    q^{E_k}\left(\hat{\vv}_{\alpha_{ijk}}^{ijk} \mid \hat{\vv}_{\pa_{\G_{ij}}(\alpha_{ijk})}^{ijk}\right) q^{E_k}\left(\hat{\vv}_{\nd_{\G_{ij}}(\alpha_{ijk})}^{ijk}\mid \hat{\vv}_{\pa_{\G_{ij}}(\alpha_{ijk})}^{ijk}\right) \neq q^{E_k}\left(\hat{\vv}_{\overline{\nd}_{\G_{ij}}(\alpha_{ijk})}^{ijk}\mid \hat{\vv}_{\pa_{\G_{ij}}(\alpha_{ijk})}^{ijk}\right).
\end{equation}
for $\hat{\vv}^{ijk}=\mM\hat{\vz}^{ijk}$.
This implies that removing the edge $j\to i$ in $\G$ would break the causal Markov condition for $q^{E_k}$. Now let $\delta=\min_{k,i,j}\delta_k^{(i,j)}>0$, then for all $\mM\in\bar{\mathcal{M}}_{\dom}(\G)$ such that $\normx{\mM-\mI}_F\leq\delta$, the resulting distributions $q^{E_k}, k\in[K]$ satisfy assumption \Cref{asmp:data-generating-process}.
\\

Combining the above two statements and what we have proven before, it is straightfoward to see that one can choose some $\mM\in\mathcal{M}_{\dom}(\G)$ in a small neighbourhood of $\mI$ that satisfies all the requirements, completing the proof.
\end{proof}

\end{appendices}

\end{document}